\PassOptionsToPackage{table}{xcolor}
\documentclass{article}
\usepackage{iclr2027_conference,times}

\usepackage{amsmath,amsfonts,bm}

\def\eqref#1{equation~\ref{#1}}
\def\1{\bm{1}}

\DeclareMathAlphabet{\mathsfit}{\encodingdefault}{\sfdefault}{m}{sl}
\SetMathAlphabet{\mathsfit}{bold}{\encodingdefault}{\sfdefault}{bx}{n}

\usepackage{hyperref}
\usepackage{graphicx} % DO NOT CHANGE THIS
\usepackage{natbib}  % DO NOT CHANGE THIS AND DO NOT ADD ANY OPTIONS TO IT
\usepackage{caption} % DO NOT CHANGE THIS AND DO NOT ADD ANY OPTIONS TO IT
\usepackage{wrapfig}
\usepackage{microtype}
\usepackage{graphicx}
\usepackage{subcaption}
\usepackage[table]{xcolor}
\usepackage{booktabs}
\usepackage{multirow}
\usepackage{tabularx}
\usepackage{adjustbox}
\usepackage{makecell}
\usepackage{array}

\usepackage{amsmath}
\usepackage{amssymb}

\usepackage{comment}

\usepackage{mathtools}
\usepackage{amsthm}
\usepackage[textsize=tiny]{todonotes}
\usepackage{algorithm}
\usepackage[ruled,vlined,linesnumbered,algo2e]{algorithm2e}
\usepackage{tcolorbox}
\usepackage{marvosym}
\usepackage{pifont}
\usepackage{tikz}
\usetikzlibrary{calc}
\usepackage{xurl}
\usepackage[capitalize,noabbrev]{cleveref}
\usepackage{algorithm}
\usepackage{algorithmic}
\usepackage{arydshln}
\definecolor{LightRed}{RGB}{255,230,230}
\definecolor{LightBlue}{RGB}{225,240,255}
\theoremstyle{plain}
\newtheorem{assumption}{Assumption}
\newtheorem{proposition}{Proposition}
\newtheorem{lemma}{Lemma}
\newtheorem{theorem}{Theorem}

\newtheorem{remark}{Remark}
\newtheorem{corollary}{Corollary}
\newtheorem*{metatheorem*}{Meta Theorem}
\theoremstyle{remark}

\definecolor{step1bg}{HTML}{FFF8DC}
\definecolor{step2bg}{HTML}{D4EAE0}
\definecolor{step3bg}{HTML}{D9ECFC}
\usetikzlibrary{calc}

\iclrfinalcopy  % 取消匿名，显示作者与单位

\definecolor{LightRed}{RGB}{255,182,193} 
\definecolor{LightBlue}{RGB}{173, 216, 230}
\definecolor{LightGreen}{RGB}{144, 238, 144}

\colorlet{shadecolor}{LightGreen}
\colorlet{shadecolor}{LightBlue}
\colorlet{shadecolor}{LightRed}
\usepackage{framed}
\newtcolorbox{graybox}[1][]{
	colback=gray!2,
	colframe=gray!70,
	boxrule=0.7pt,
	arc=2mm,
	left=0mm, right=0mm, top=0mm, bottom=0mm,
	halign=center,  
	#1
}

\definecolor{LightRed}{RGB}{255,182,193}
\definecolor{LightBlue}{RGB}{173,216,230}

\newtcolorbox{redbox}[1][]{
	colback=LightRed!30!white,
	colframe=LightRed,
	boxrule=0.7pt,
	arc=2mm,
	left=0mm, right=0mm, top=0mm, bottom=0mm,
	halign=center,
	#1
}

\newtcolorbox{bluebox}[1][]{
	colback=LightBlue!30!white,
	colframe=LightBlue,
	boxrule=0.7pt,
	arc=2mm,
	left=0mm, right=0mm, top=0mm, bottom=0mm,
	halign=center,
	#1
}

\usepackage{newfloat}
\usepackage{listings}
\DeclareCaptionStyle{ruled}{labelfont=normalfont,labelsep=colon,strut=off} % DO NOT CHANGE THIS
\floatstyle{ruled}
\newfloat{listing}{tb}{lst}{}
\floatname{listing}{Listing}
\usepackage{array}
\usepackage{booktabs}

\title{
	FedLore: Communication and Memory Efficient  Federated Learning  via Shared Gradient Low-Rank  Projection
}

\author{ Junkang Liu\textsuperscript{1}\\
	Tianjin University\\
	\texttt{\{junkangliukk\}@gmail.com}
}

\begin{document}

\maketitle
	\pagestyle{plain}
\thispagestyle{plain}

\begin{abstract}
Federated training of foundation models is constrained by
client memory and communication costs.
LoRA-based methods reduce these costs through low-rank
adapters, but their fixed rank budget can limit adaptation.
Gradient low-rank optimization offers greater flexibility,
yet independently chosen client subspaces create a problem
we term \emph{subspace fragmentation}: local projections
interact with data heterogeneity to bias aggregated
directions, while aggregation can increase update rank
and communication cost.
Thus, accurate local gradient compression need not
preserve global descent.
We propose \texttt{FedLore}, which shares a low-rank
optimization basis within each round and refreshes it
across rounds.
The shared basis enables exact aggregation in low-rank
coordinates and eliminates the identified projection bias.
Subspace refresh allows the accumulated model update
to exceed the per-round rank budget.
We characterize the aggregation bias and establish an
$O(T^{-1/2})$ stationarity bound for the projected-SGD
variant under a global-gradient coverage condition
and standard smoothness and variance assumptions,
with bounded gradient heterogeneity.
Experiments on vision and language tasks, including
federated pre-training, show that \texttt{FedLore}
outperforms the evaluated low-rank adapter baselines
and matches or exceeds full-parameter training,
while reducing communication and optimizer-state memory.
\end{abstract}


\section{Introduction}
\vspace{-2mm}

Foundation models have become central to modern vision and
language systems, yet much of the data needed for further
training resides on private devices or across institutions
with restricted data access
\cite{bommasani2021opportunities}.
Federated learning (FL) enables collaborative training
without centralizing these data
\cite{mcmahan2017fedavg,kairouz2021advances}.
Scaling this paradigm to foundation models, however,
requires addressing two coupled resource bottlenecks:
maintaining optimizer states on memory-constrained clients
and repeatedly exchanging high-dimensional model updates.
Full-parameter federated optimization (\texttt{FedFull})
incurs substantial costs on both fronts, particularly with
adaptive optimizers such as Adam/AdamW
\cite{kingma2014adam}.

A common approach combines FL with parameter-efficient
fine-tuning, particularly LoRA \citep{hu2021lora}.
Methods such as FedIT \cite{zhang2023towards} reduce
optimizer-state memory and communication by training and
exchanging low-rank adapters
\citep{wu2024fedlora,wang2024flora,yi2023pfedlora}.
For a standard rank-$r$ adapter, however, the model
correction remains rank-constrained, even though its
row and column spaces can evolve during training.
This constraint can limit adaptation when heterogeneous
data or training from scratch requires updates spanning
many directions.
The key question is therefore:
\emph{Can federated training retain the resource efficiency
	of low-rank updates without imposing a fixed rank budget
	on the accumulated model change?}

\begin{comment}

\begin{figure}[tb]
	\centering
	\vspace{-4mm}
	\includegraphics[width=1\textwidth]{figure/Fedlore.png}
	\vspace{-2mm}
	\caption{Fixed-rank adaptation versus accumulated low-rank updates.
		A standard LoRA adapter constrains the rank of the model
		correction, although its row and column spaces can evolve.
		\texttt{FedLore} shares a low-rank update basis within each
		round and refreshes it across rounds, allowing the
		accumulated model change to exceed the per-round rank.}
	\label{fig:fedit_vs_fedlore}
	\vspace{-3mm}
\end{figure}
\end{comment}

\begin{figure*}[tb]
	\centering
	
	\begin{subfigure}[b]{0.24\textwidth}
		\includegraphics[
		width=\linewidth,
		trim=6 5 6 5,
		clip
		]{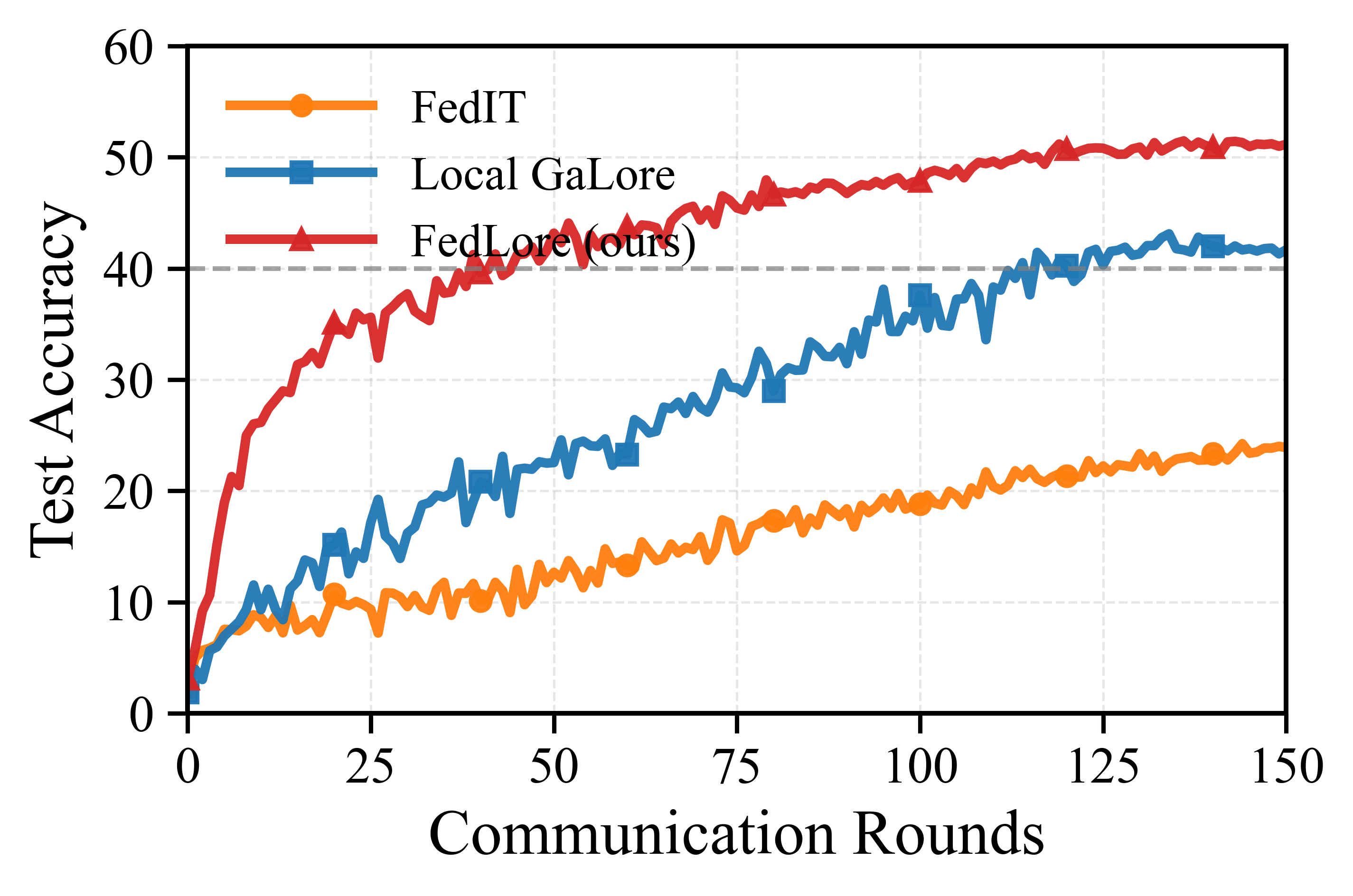}
		\caption{Test accuracy (\%)}
		\label{fig:fedlore_acc}
	\end{subfigure}
	\hspace{-1mm}
	\begin{subfigure}[b]{0.24\textwidth}
		\includegraphics[
		width=\linewidth,
		trim=6 5 6 5,
		clip
		]{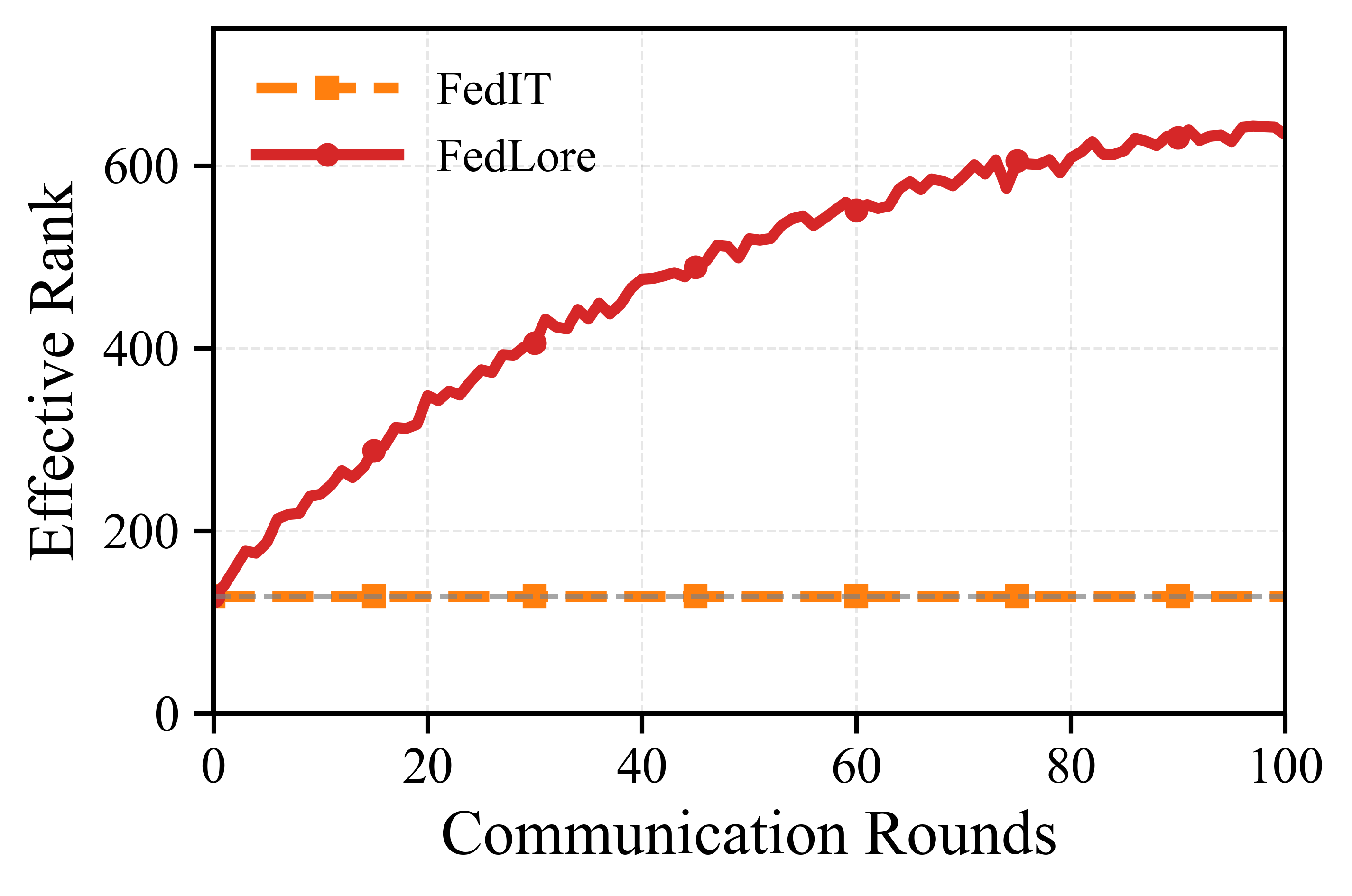}
		\caption{Effective Rank}
		\label{fig:fedlore_rank}
	\end{subfigure}
	\hspace{-1mm}
	\begin{subfigure}[b]{0.24\textwidth}
		\includegraphics[
		width=\linewidth,
		trim=6 5 6 5,
		clip
		]{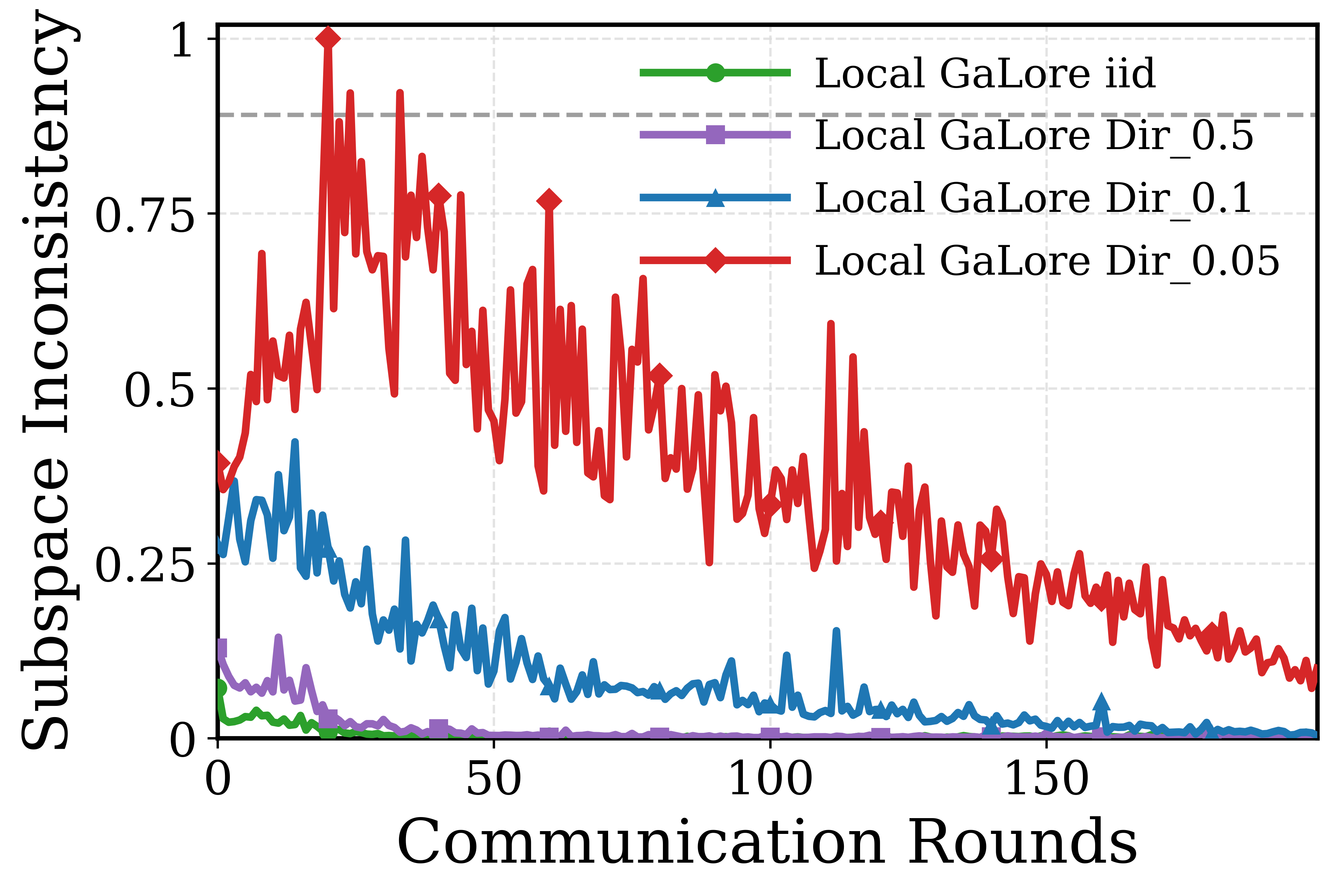}
		\caption{Space fragmentation}
		\label{fig:galore_drift}
	\end{subfigure}
	\hspace{-1mm}
	\begin{subfigure}[b]{0.245\textwidth}
		\includegraphics[
		width=\linewidth,
		%trim=6 5 6 5,
		%clip
		]{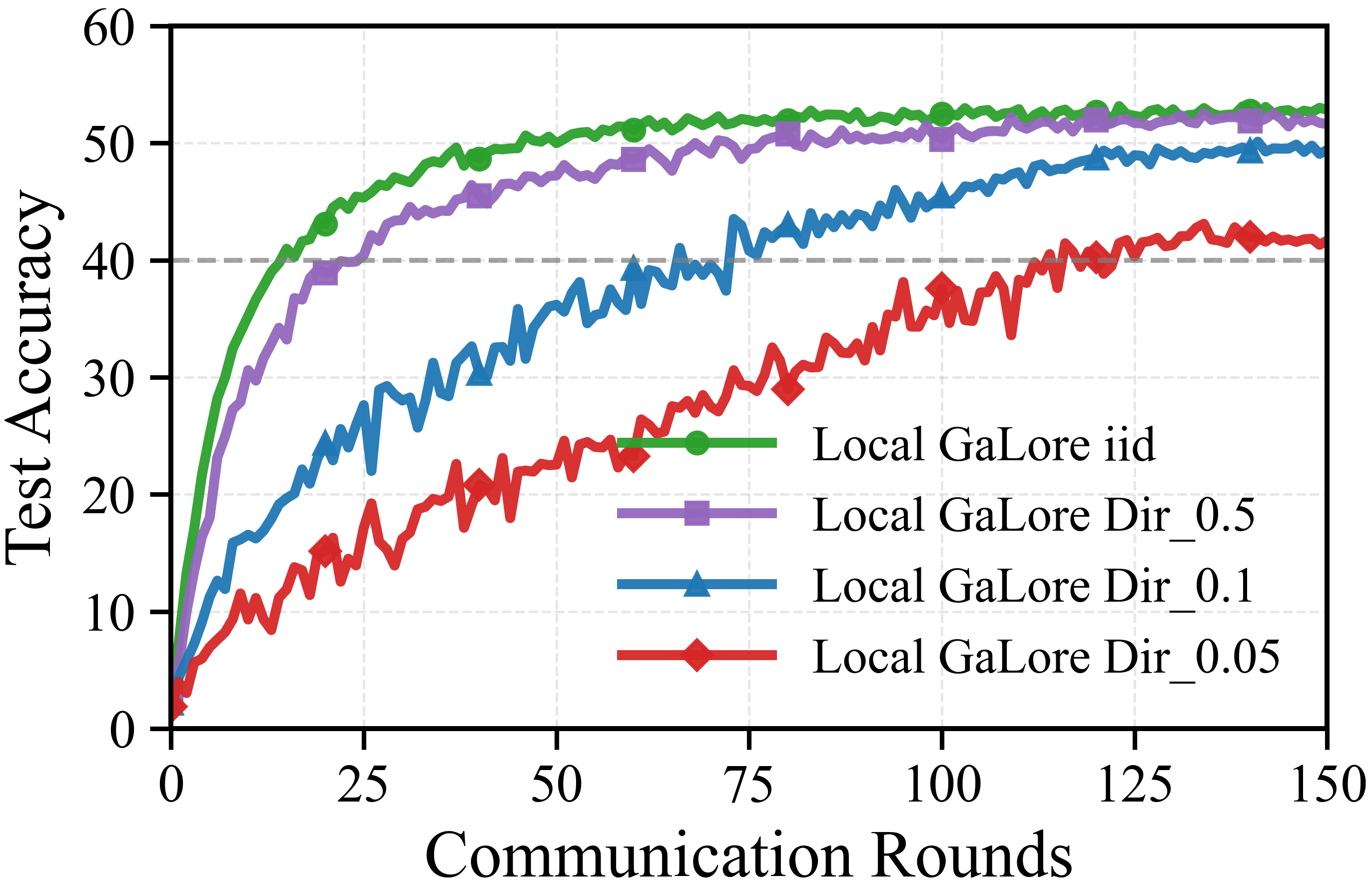}
		\caption{Test accuracy (\%)}
		\label{fig:galore_acc}
	\end{subfigure}
	
	\vspace{-1mm}
	\caption{
		Comparison between \texttt{FedLore} and Local GaLore and FedIT on ViT-Tiny over CIFAR-100 with non-iid data. (a) Test accuracy: \texttt{FedLore} vs FedIT vs Local GaLore. (b) Effective Rank: Dynamic expansion of optimization space of \texttt{FedLore}. (c) Subspace fragmentation of Local GaLore. We measure subspace fragmentation by the  difference
		between projection matrices: $\Delta_{\mathcal D}
		=\frac{1}{Nr}\sum_{i\in\mathcal N}
		\|\Pi_i-\bar{\Pi}\|_F^2,
		\Pi_i=P_iP_i^\top,
		\bar{\Pi}=\frac{1}{N}\sum_{i\in\mathcal N}\Pi_i$. Client-specific projections introduce severe subspace fragmentation under heterogeneous data. (d) Local GaLore degrades under heterogeneity due to misaligned client-wise optimization subspaces.
	}
	\label{fig:insight}
	\vspace{-2mm}
\end{figure*}

Gradient low-rank optimization offers a promising starting
point.
GaLore \citep{zhao2024galore} maintains optimizer states
in projected gradient coordinates while updating model
weights in the original parameter space.
This separates the dimension of the optimization state
from the rank of the accumulated model change:
successive low-rank increments can span an increasingly
rich set of directions as the projection changes.
In FL, however, local compression and global aggregation
must be designed together.
A subspace that captures a client's dominant gradient
directions need not preserve the directions required
for progress on the global objective.

\begin{wrapfigure}{r}{0.49\textwidth}
	\centering
	\vspace{-1mm}
	\includegraphics[width=\linewidth]{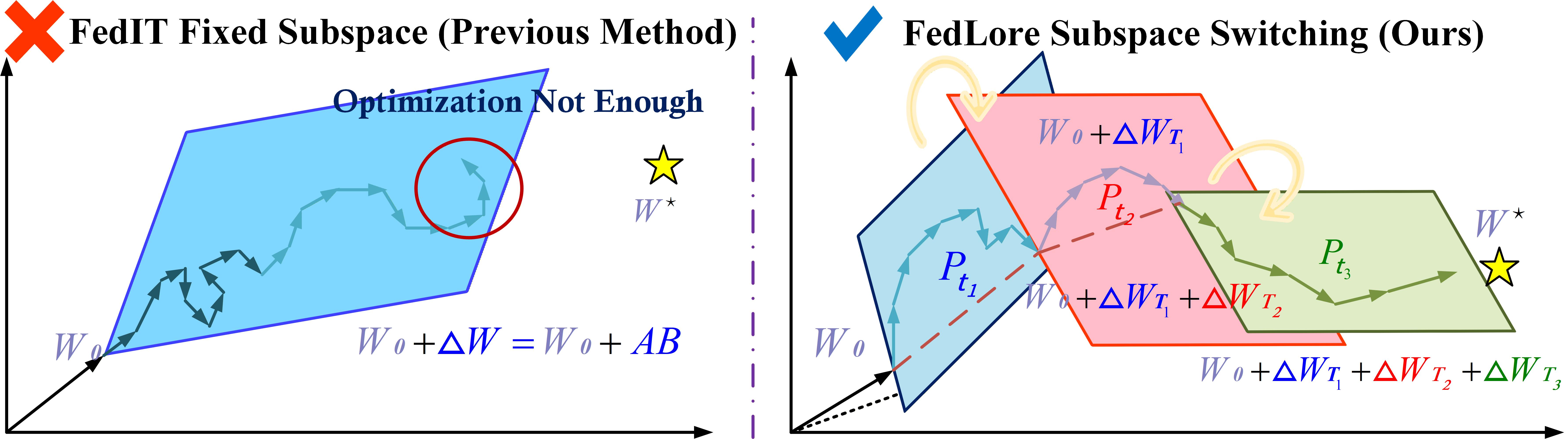}
	\vspace{-2mm}
	\caption{Fixed-rank LoRA vs ours low-rank updates.
		\texttt{FedLore} shares a low-rank update basis within each
		round and refreshes it across rounds, allowing the
		accumulated model change to exceed the per-round rank.}
	\label{fig:fedit_vs_fedlore}
	\vspace{-3mm}
\end{wrapfigure}
We study this mismatch through
\emph{subspace fragmentation}: heterogeneous clients
select different gradient subspaces, and these
client-dependent projections interact with gradient
heterogeneity during aggregation
(Figure~\ref{fig:insight}).
The central issue is not that updates expressed in
different coordinates cannot be averaged, they can
be reconstructed and aggregated in the original
parameter space.
Rather, independently chosen projections alter which
components of each client gradient survive aggregation,
potentially suppressing useful global directions.
This interaction admits an exact decomposition.
At a common model $W$, let
$G_i=\nabla f_i(W)$,
$G=N^{-1}\sum_i G_i$,
$\Pi_i=P_iP_i^\top$, and
$\bar{\Pi}=N^{-1}\sum_i\Pi_i$.
Then
$
\frac{1}{N}\sum_{i=1}^{N}\Pi_iG_i
=
\bar{\Pi}G+b,
\:
b:=
\frac{1}{N}\sum_{i=1}^{N}
(\Pi_i-\bar{\Pi})(G_i-G).
$
Here, $b$ captures the coupling between projection
disagreement and gradient heterogeneity.
The second term captures the interaction between
projection disagreement and data heterogeneity.
It vanishes when all clients share a projector,
but need not vanish when each local projector
captures most of its own gradient energy.
Thus, \emph{accurate local gradient compression does
	not by itself ensure an informative global update}.
Our analysis makes this distinction explicit through
a sharp bias bound and a quadratic example in which
exact local SVD projections stall away from stationarity.

Subspace fragmentation also affects communication.
With a basis fixed within each round and zero weight
decay, each client's accumulated update remains
rank at most $r$ and admits an exact low-rank
factorization.
Nevertheless, aggregating updates from $N$ distinct
client subspaces can produce a matrix of rank up to
$\min\{m,n,Nr\}$.
Low-rank client updates therefore do not automatically
yield a rank-$r$ global update.
This exposes two connected challenges:

\begin{redbox}
	\textit{\textbf{Challenge 1:}
		Client-specific subspaces increase the cost of
		exact  communication.}
\end{redbox}

\begin{bluebox}
	\textit{\textbf{Challenge 2:} Federated subspace fragmentation under
		heterogeneous clients.}
\end{bluebox}

\begin{comment}

\begin{wrapfigure}{r}{0.49\textwidth}
	\centering
	\vspace{-4mm}
	\includegraphics[width=\linewidth]{figure/FedLore.jpg}
	\vspace{-4mm}
	\caption{Comparison of federated training paradigms. Unlike FedIT, which restricts optimization to low-rank trainable adapters, \texttt{FedLore} enables low-rank gradient optimization while preserving full-rank weight updates.}
	\label{fig:fedlore_vs3}
	\vspace{-3mm}
\end{wrapfigure}
	内容...
\end{comment}
We propose \texttt{FedLore}, a federated optimization
framework that shares optimization coordinates
within each round and refreshes them across rounds.
Clients optimize in a common orthonormal basis,
upload their accumulated low-rank coordinates,
and provide candidate subspaces for the next round.
The server constructs the next shared basis from
the leading eigenspace of the averaged candidate
projectors, making subspace aggregation invariant
to the signs and rotations of the candidate bases.

This design separates two roles of low-rank structure.
\emph{Within a round}, sharing the basis makes
coordinate aggregation exactly equivalent to
averaging the local endpoint models, without
additional rank truncation.
Under full participation with synchronized model
caches, the aggregated coordinates and the next
basis suffice for low-rank model synchronization.
\emph{Across rounds}, refreshing the basis allows
the accumulated model change to extend beyond
a single rank-$r$ correction
(Figures~\ref{fig:fedit_vs_fedlore}).
Sharing removes the projection--heterogeneity
coupling term, whereas subspace selection determines
how much of the global gradient is retained.
We distinguish these properties explicitly:
our convergence guarantee requires global-gradient
coverage, which agreement alone does not ensure.

\textbf{Contributions.}
Our main contributions are as follows:\\
\textbf{$\bullet$}
We identify \emph{subspace fragmentation}: client-specific projections couple with data heterogeneity, allowing locally accurate compression to suppress global descent and cost of
exact  communication.\\
\textbf{$\bullet$}
We propose \texttt{FedLore}, which shares a low-rank basis within each round and refreshes it across rounds. This enables exact coordinate aggregation, compact optimizer states, and accumulated updates beyond rank $r$. With synchronized caches and full participation, communication remains low-rank.\\
\textbf{$\bullet$}
We show that shared projections eliminate projection--heterogeneity coupling in projected SGD and establish an $O(T^{-1/2})$ stationarity bound under standard smoothness, variance, and heterogeneity assumptions with global-gradient coverage. Across vision and language tasks, including federated pre-training, FedLore combines competitive model quality with reduced communication and optimizer-state memory, while allowing accumulated updates beyond the per-round rank budget.

%\begin{comment}

\section{Related Work}
\textbf{Parameter-Efficient Fine-Tuning for Large Models.}
As pretrained models scale, parameter-efficient fine-tuning (PEFT) enables adaptation under limited computational and memory budgets \citep{Towards_Federated_PEFT}. LoRA parameterizes updates to frozen pretrained weights using trainable low-rank factors, substantially reducing gradient and optimizer-state memory \citep{hu2021lora}. Subsequent methods extend this approach through adaptive rank allocation and alternative parameterizations \citep{zhang2023adalora,liu2024dora}. However, although standard LoRA's update subspaces evolve during optimization, its accumulated weight updates remain rank-constrained, potentially limiting expressivity under substantial distribution shift or when training from scratch. Low-rank factorization can also introduce optimization challenges \citep{cosson2023lowrankgd}, motivating memory-efficient methods with more flexible weight updates \citep{zhao2024galore}.

\label{sec:related}

%\vspace{-2mm}
\begin{comment}

LoRA is one of the most widely used PEFT in FL due to its low communication and memory cost.
FedIT~\citep{zhang2023towards} averages client LoRA matrices, while later variants improve aggregation, initialization, or factor sharing.
For example, FLoRA~\citep{wang2024flora} and LoRA-FAIR~\citep{bian2025lora} study more robust aggregation, FR-LoRA~\citep{yan2025federated} improves adapter initialization, and FedSA-LoRA~\citep{guo2025selective} and FFA-LoRA~\citep{sun2024improving} modify how the $A$ and $B$ matrices are shared or optimized.
However, these methods still aggregate fixed rank adapters, which restricts global updates to LoRA-induced subspaces.
\end{comment}
\textbf{Extending LoRA to Federated Learning.}
LoRA is widely used for parameter-efficient federated fine-tuning because it reduces communication and optimizer-state memory costs.
FedIT~\citep{zhang2023towards} aggregates client LoRA factors, while subsequent methods refine aggregation, initialization, or factor sharing.
FLoRA~\citep{wang2024flora} and LoRA-FAIR~\citep{bian2025lora} improve aggregation, whereas FedSA-LoRA~\citep{guo2025selective} and FFA-LoRA~\citep{sun2024improving} modify how the $A$ and $B$ factors are shared or optimized.
FR-LoRA~\citep{yan2025federated} combines factor reinitialization with residual low-rank updates, allowing the accumulated model change to exceed the per-round rank budget.
In contrast, FedLore optimizes projected gradients in a shared, evolving basis and enables exact aggregation of accumulated updates in low-rank coordinates.

\section{Why Local Gradient Projection Fails in Federated Learning}

\textbf{FL Problem Setup.}
FL aims to optimize model parameters, i.e., minimizing the following problem:
\vspace{-2mm}
\begin{equation*}
	f_i(\boldsymbol{W}) 
	:= \mathbb{E}_{\xi_i \sim \mathcal{D}_i}
	\bigl[F_i(\boldsymbol{W}; \xi_i)\bigr],  f(\boldsymbol{W}) := \frac{1}{N} \sum_{i=1}^N f_i(\boldsymbol{W}).
	\label{eq:objective}
\end{equation*}
Here, $f_i$ denotes client $i$'s loss, with expectation over $\xi_i \sim \mathcal{D}_i$; $N$ is the client count, and $\boldsymbol{W}$ is the global model.

\subsection{FL with  Gradient Low-Rank Projection in Local}

\label{sec:subspace}
As a motivating example, we describe low-rank optimizers
using a single linear layer $W$. In round $t$, client $i$ receives the global weights $W^t$ and computes the gradient of its local objective $f_i$:
$
G_i^t = -\nabla_W f_i(W_i^{t,0}; \xi_i^{t,k}) \in \mathbb{R}^{m\times n}.
$
To construct a local low-rank subspace tailored to its data distribution, the client performs a truncated singular value decomposition (SVD) of $G_i^t$:
$
	G_i^t \approx U_i^t S_i^t (V_i^t)^\top,
$
where 
$U_i^t \in \mathbb{R}^{m \times r}$,
$S_i^t \in \mathbb{R}^{r \times r}$,
and 
$V_i^t \in \mathbb{R}^{n \times r}$ 
contain the top-$r$ singular components.
We define the left singular vectors as the projection basis for local optimization: $P_i^t = U_i^t.$

Given the subspace basis $P_i^t$, the client projects gradients onto this subspace. 
The projected gradient at step $k$ is
$
	R_{i}^{t,k} = -(P_i^t)^\top \nabla_W f_i(W_{i}^{t,k}; \xi_i^{t,k}) 
	\in \mathbb{R}^{r \times n}.
$
To retain adaptive optimization, we maintain the optimizer states 
in the projected coordinates:
\vspace{-1mm}
\begin{equation}
	\begin{aligned}
		M_i^{t,k+1}
		&= \beta_1 M_i^{t,k} + (1-\beta_1)R_i^{t,k},
		&\quad
		V_i^{t,k+1}
		&= \beta_2 V_i^{t,k} + (1-\beta_2)(R_i^{t,k}\odot R_i^{t,k}),\\
		\widehat{M}_i^{t,k+1}
		&= M_i^{t,k+1}/(1-\beta_1^{k+1}),
		&
		\widehat{V}_i^{t,k+1}
		&= V_i^{t,k+1}/(1-\beta_2^{k+1}),\\
		N_i^{t,k+1}
		&= \widehat{M}_i^{t,k+1}/
		(\sqrt{\widehat{V}_i^{t,k+1}}+\varepsilon),
		&
		W_i^{t,k+1}
		&= W_i^{t,k}+\eta\alpha\,P_i^tN_i^{t,k+1}.
	\end{aligned}
\end{equation}

After $K$ local steps, the client  uploads full-parameter update 
$\Delta W_i^t = W_i^{t,K}-W^t$ to the server.
\subsection{Challenges of Gradient Low-Rank Projection in FL}
Despite the widespread use of GaLore in centralized deep learning, its adaptation to federated settings remains largely unexplored. In this subsection, we analyze two fundamental challenges in FL.
\vspace{-1mm}
\begin{redbox}
	\textit{\textbf{Challenge 1:}
		Client-specific subspaces increase the cost of
		exact  communication.}
\end{redbox}

Low-rank local updates do not necessarily admit a
rank-$r$ representation after aggregation.
Consider zero weight decay and a client-specific basis
$P_i^t\in\mathbb{R}^{m\times r}$ fixed within round $t$.
After $K$ local steps, client $i$ produces
$
\Delta W_i^t
=
\eta\alpha P_i^t\Delta Z_i^t,
\Delta Z_i^t
=
\sum_{k=0}^{K-1}N_i^{t,k+1}.
$
Each client can therefore upload its update exactly
using the factors $(P_i^t,\Delta Z_i^t)$, requiring
$r(m+n)$ scalars.
However, with $N$ clients, the aggregated
update is
\vspace{-1mm}
\begin{equation*}
	\begin{aligned}
		\Delta W^t
		=
		\frac{\eta\alpha}{N}
		\sum_{i\in\mathcal S_t}P_i^t\Delta Z_i^t,\quad
		\operatorname{rank}(\Delta W^t)
		\leq \min\{m,n,Nr\}.
	\end{aligned}
\end{equation*}
Client-specific subspaces can produce an aggregate of rank exceeding $r$, precluding an exact rank-$r$ representation.
Direct transmission requires either $mn$ scalars for the dense aggregate or $Nr(m+n)$ for concatenated client factors; rank-$r$ truncation introduces error when the aggregate rank exceeds $r$.
With a shared basis,
$\Delta W^t=\eta\alpha P^t\bigl(N^{-1}\sum_{i\in\mathcal N}\Delta Z_i^t\bigr)$
has rank at most $r$.
Clients caching $(W^t,P^t)$ reconstruct the next model from $rn$ aggregated coordinates.
Including the next basis costs $r(m+n)$ scalars per client, independently of $N$.
\begin{bluebox}
		\textit{\textbf{Challenge 2:} Federated subspace fragmentation under
			heterogeneous clients.}
\end{bluebox}

In Local GaLore, each client constructs an orthonormal basis
$P_i$ from its local gradient via truncated SVD.
Heterogeneous data can produce different projection subspaces,
a phenomenon we term \emph{subspace fragmentation}.
Although local updates remain exactly aggregable in the ambient
space, projection disagreement can couple with gradient
heterogeneity to impair global descent.
At a common model $W$, let $G_i=\nabla f_i(W)$,
$G=N^{-1}\sum_i G_i$, $\Pi_i=P_iP_i^\top$, and
$\bar{\Pi}=N^{-1}\sum_i\Pi_i$. Then
\[
\frac{1}{N}\sum_{i=1}^{N}\Pi_iG_i
=
\bar{\Pi}G+b,
\quad
b:=
\frac{1}{N}\sum_{i=1}^{N}
(\Pi_i-\bar{\Pi})(G_i-G).
\]
The coupling term $b$ can weaken or cancel global-gradient
alignment despite accurate local compression; disagreement
alone does not imply harmful interference.
Figure~\ref{fig:insight}(c,d) illustrates subspace disagreement
and performance degradation under stronger heterogeneity.
Shared projections eliminate $b$ exactly, while global
descent additionally requires sufficient gradient coverage.

\begin{algorithm}[t]
	\small
	\caption{\texttt{FedLore} Training Procedure}
	\label{alg:fedlore}
	\begin{algorithmic}[1]
		\REQUIRE Global weights $W^0$; learning rate $\eta$;
		scale $\alpha$; rank $r$; Adam coefficients
		$\beta_1,\beta_2$; $\epsilon>0$;
		communication rounds $T$; local steps $K$;
		client set $\mathcal{N}$ with $|\mathcal{N}|=N$;
		weight decay $\lambda=0$.
		\ENSURE Trained global model $W^T$.
		
		\STATE Initialize a shared orthonormal basis
		$P^0\in\mathbb{R}^{m\times r}$
		\STATE Server broadcasts $(W^0,P^0)$ to all clients
		\STATE Each client initializes 
		global-model cache: $\widetilde W_i\gets W^0$
		
		\FOR{$t=0$ to $T-1$}
		\FOR{each client $i\in\mathcal{N}$
			\textbf{in parallel}}
		
		\STATE \texttt{//Low-Rank Model Synchronization}
		\IF{$t>0$}
		\STATE Receive $(P^t,\Delta^{t-1})$
		from the server
		\STATE $\widetilde W_i
		\gets \widetilde W_i
		+\eta\alpha P^{t-1}\Delta^{t-1}$
		\ENDIF
		\STATE $W_i^{t,0}\gets\widetilde W_i$, $M_i^{t,0}\gets 0$,
		$V_i^{t,0}\gets 0$, $\Delta Z_i^t\gets 0$
		
		\STATE \texttt{//Local Projected Adam Updates}
		\FOR{$k=0$ to $K-1$}
		%\STATE Sample a fresh mini-batch
	%	$\xi_i^{t,k}$
		\STATE $G_i^{t,k}\gets
		-\nabla_W f_i(W_i^{t,k};\xi_i^{t,k})
		\in\mathbb{R}^{m\times n}$;\quad $R_i^{t,k}\gets
		(P^t)^\top G_i^{t,k}
		\in\mathbb{R}^{r\times n}$
		
		\STATE $M_i^{t,k+1}\gets
		\beta_1M_i^{t,k}
		+(1-\beta_1)R_i^{t,k}$;\quad $V_i^{t,k+1}\gets
		\beta_2V_i^{t,k}
		+(1-\beta_2)
		(R_i^{t,k}\odot R_i^{t,k})$
		
		\STATE $\widehat M_i^{t,k+1}\gets
		M_i^{t,k+1}/(1-\beta_1^{k+1})$;\quad $\widehat V_i^{t,k+1}\gets
		V_i^{t,k+1}/(1-\beta_2^{k+1})$
		\STATE $N_i^{t,k+1}\gets
		\widehat M_i^{t,k+1}/
		(\sqrt{\widehat V_i^{t,k+1}}+\epsilon)$;\quad $W_i^{t,k+1}\gets
		W_i^{t,k}+\eta\alpha P^tN_i^{t,k+1}$
		\STATE $\Delta Z_i^t\gets
		\Delta Z_i^t+N_i^{t,k+1}$
		\ENDFOR
		
		\STATE \texttt{//Local Candidate Subspace}
		\STATE Reuse the last local gradient
		$G_i^{t,K-1}$;\quad $P_i^t\gets
		\operatorname{TopLeftSVD}_r(G_i^{t,K-1})
		\in\mathbb{R}^{m\times r}$
		\STATE Upload $(\Delta Z_i^t,P_i^t)$
		to the server;\quad Cache the shared basis $P^t$ for reconstructing
		the aggregated update.
		\ENDFOR
		
		\STATE \texttt{//Server-Side Update Aggregation}
		\STATE $\Delta^t\gets
		\frac{1}{N}\sum_{i\in\mathcal{N}}\Delta Z_i^t$;\quad $W^{t+1}\gets
		W^t+\eta\alpha P^t\Delta^t$
		
		\STATE \texttt{//Shared-Subspace Construction}
		\STATE $B^t\gets
		\frac{1}{\sqrt N}
		[P_1^t,\ldots,P_N^t]
		\in\mathbb{R}^{m\times Nr}$;\quad $P^{t+1}\gets
		\operatorname{TopLeftSVD}_r(B^t)
		\in\mathbb{R}^{m\times r}$
		\ENDFOR
	\end{algorithmic}
\end{algorithm}

\section{Federated Training via Shared Gradient Low-Rank Projection}
\label{sec:method}
\label{sec:fedlore}

We propose \texttt{FedLore}, a federated optimization framework using a shared, evolving gradient subspace. It maintains full model weights while storing local optimizer states and communicating updates in low-rank coordinates. The shared basis enables exact aggregation, while refreshing it allows accumulated updates to exceed rank $r$. Algorithm~\ref{alg:fedlore} presents the procedure for a weight matrix $W\in\mathbb{R}^{m\times n}$ under full participation and a constant learning rate.

\textbf{Shared-subspace construction.}
Let $P^t\in\mathbb{R}^{m\times r}$ denote the shared
orthonormal basis used throughout round $t$, with
$(P^t)^\top P^t=I_r$.
The corresponding orthogonal projector is
$\Pi^t=P^t(P^t)^\top$.
At the end of the round, client $i$ extracts a candidate
basis $P_i^t$ from the leading $r$ left singular vectors
of its last local gradient $G_i^{t,K-1}$.
Directly averaging these bases is sensitive to arbitrary
signs and within-subspace rotations.
We instead aggregate their orthogonal projectors:
\begin{equation*}
	A_{\mathrm{proj}}^t
	=
	\frac{1}{N}\sum_{i\in\mathcal N}
	P_i^t(P_i^t)^\top,
	P^{t+1}
	\in
	\underset{P^\top P=I_r}{\arg\max}\;
	\operatorname{tr}\!\left(P^\top A_{\mathrm{proj}}^tP\right).
	\label{eq:fedlore_shared_basis}
\end{equation*}
Thus, $P^{t+1}$ spans a leading rank-$r$ eigenspace of
$A_{\mathrm{proj}}^t$ and maximizes the average overlap
with the client candidate subspaces.
The aggregated operator is invariant to replacing any
$P_i^t$ with $P_i^tQ_i$ for an orthogonal matrix $Q_i$.
To avoid explicitly forming the dense projector average,
we use
$
	B^t
	=
	\frac{1}{\sqrt N}[P_1^t,\ldots,P_N^t],
	A_{\mathrm{proj}}^t=B^t(B^t)^\top,
	\label{eq:fedlore_stacked_basis}
$
and obtain $P^{t+1}$ from the leading left singular vectors
of $B^t$.
Randomized truncated SVD can approximate these subspaces
without forming $A_{\mathrm{proj}}^t$.
The server broadcasts the resulting basis for round $t+1$.

\textbf{Local optimization in shared coordinates.}
Each client starts from the synchronized global model,
$W_i^{t,0}=W^t$, and initializes
$M_i^{t,0}=V_i^{t,0}=0$.
For local step $k$, it computes a stochastic negative
gradient and its low-rank coordinates:
\vspace{-1mm}
\begin{equation*}
	G_i^{t,k}
	=
	-\nabla_W f_i(W_i^{t,k};\xi_i^{t,k}),
	R_i^{t,k}
	=
	(P^t)^\top G_i^{t,k}
	\in\mathbb{R}^{r\times n}.
	\label{eq:fedlore_projected_gradient}
\end{equation*}
The first- and second-moment states are maintained
entirely in this coordinate system:
\vspace{-1mm}
\begin{equation}
	\begin{aligned}
		M_i^{t,k+1}
		&= \beta_1 M_i^{t,k}+(1-\beta_1)R_i^{t,k},
		&\quad
		V_i^{t,k+1}
		&= \beta_2 V_i^{t,k}+(1-\beta_2)
		\bigl(R_i^{t,k}\odot R_i^{t,k}\bigr),\\
		\widehat M_i^{t,k+1}
		&= M_i^{t,k+1}/(1-\beta_1^{k+1}),
		&
		\widehat V_i^{t,k+1}
		&= V_i^{t,k+1}/(1-\beta_2^{k+1}).
	\end{aligned}
	\label{eq:fedlore_moments}
\end{equation}
The normalized direction and weight update are
$
	N_i^{t,k+1}
		=
		\frac{\widehat M_i^{t,k+1}}
		{\sqrt{\widehat V_i^{t,k+1}}+\epsilon},
		W_i^{t,k+1}
		=
		W_i^{t,k}
		+\eta\alpha P^tN_i^{t,k+1}.
	\label{eq:fedlore_local_update}
$
All operations in the normalization are elementwise.
This is AdamW with weight decay set to zero.
Resetting the moments each round avoids transferring
optimizer states between changing coordinate systems.
Each weight increment has rank at most $r$.

% Body: place immediately before the paragraph that should wrap around it
\begin{comment}

\begin{wrapfigure}{r}{0.45\textwidth}
	\centering
	\vspace{-0.7\baselineskip}
	\includegraphics[width=0.49\linewidth]{figure/fedsoap_drift_normalized.png}\hfill
	\includegraphics[width=0.49\linewidth]{figure/local_galore_drift_v2.png}
	\caption{Subspace inconsistency. FedLore has lower subspace inconsistency}
	\label{fig:subspace_drift}
	\vspace{-0.8\baselineskip}
\end{wrapfigure}
\end{comment}
\textbf{Exact low-rank aggregation and synchronization.}
With $P^t$ fixed within each round, local updates factorize exactly as
$\Delta Z_i^t=\sum_{k=0}^{K-1}N_i^{t,k+1}$ and
$W_i^{t,K}-W^t=\eta\alpha P^t\Delta Z_i^t$.
Clients upload $(\Delta Z_i^t,P_i^t)$ for aggregation and basis construction.
The server computes
$\Delta^t=N^{-1}\sum_{i\in\mathcal N}\Delta Z_i^t$ and
$W^{t+1}=W^t+\eta\alpha P^t\Delta^t$,
exactly matching endpoint averaging without additional compression.
Under full participation, clients cache $(W^t,P^t)$ and reconstruct
$W^{t+1}$ upon receiving $(\Delta^t,P^{t+1})$.
Excluding initialization, communication costs $r(m+n)$ scalars per matrix,
per client, per round in each direction.
The two optimizer moments occupy $2rn$ scalars, excluding weights,
gradients, bases, activations.
%\textbf{Dynamic capacity and theoretical scope.}
%Our projected-SGD analysis shows that sharing cancels the aggregation bias arising from projection--heterogeneity coupling.
%Convergence additionally requires global-gradient coverage, which projector averaging alone does not guarantee. The adaptive implementation is evaluated empirically.
\section{Theoretical Analysis}

%\subsection{Assumptions}

Let
$
f(W):=\frac1N\sum_{i=1}^N f_i(W),
f^\star:=\inf_W f(W).
$
For each client $i$ and round $t$, let
$P_i^t\in\mathbb R^{m\times r}$ have orthonormal columns,
and define
$
\Pi_i^t:=P_i^t(P_i^t)^\top,
A_t:=\frac1N\sum_{i=1}^N\Pi_i^t.
$
In the shared-projection setting, $\Pi_i^t=\Pi^t$ for all $i$,
so $A_t=\Pi^t$.
%\textbf{Sampling and conditioning.} Let $\mathcal F_t$ contain the round-start history, $W^t$, and projectors $\{\Pi_i^t\}_{i=1}^N$, which remain fixed for $K$ local steps. Given $\mathcal F_t$, sample $\mathcal S_t$ uniformly without replacement with $|\mathcal S_t|=S$, independently of fresh oracle randomness. Define virtual oracle streams and trajectories for all clients, executing only selected ones. Let $\mathcal H_{t,k}$ contain $\mathcal F_t$, $\mathcal S_t$, and all oracle draws before step $k$; then $W_i^{t,k}$ is $\mathcal H_{t,k}$-measurable.
\begin{assumption}[Smoothness and lower boundedness]
	\label{ass:final_smooth_rev}
	Each local objective
	$f_i:\mathbb R^{m\times n}\to\mathbb R$
	is differentiable and $L$-smooth: for all $X,Y$,
	$
	\|\nabla f_i(X)-\nabla f_i(Y)\|_F
	\le L\|X-Y\|_F.
	$
	The global objective satisfies $f^\star>-\infty$.
\end{assumption}

\begin{assumption}[Stochastic gradient oracle]
	\label{ass:final_oracle_rev_corr}
	Define
	$
	\varepsilon_i^{t,k}
	:=g_i^{t,k}-\nabla f_i(W_i^{t,k}).
	$
	For every client $i$, round $t$, and  step $k$,
	$
	\mathbb E[
	\varepsilon_i^{t,k}
	\mid\mathcal H_{t,k}
	]=0,
	\mathbb E[
	\|\varepsilon_i^{t,k}\|_F^2
	\mid\mathcal H_{t,k}
	]\le\sigma^2.
	$
	For distinct clients $i\ne j$,  projected noises
	at  same local step are conditionally uncorrelated:
	$
	\mathbb E\!\left[
	\left\langle
	\Pi_i^t\varepsilon_i^{t,k},
	\Pi_j^t\varepsilon_j^{t,k}
	\right\rangle_F
	\,\middle|\,
	\mathcal H_{t,k}
	\right]=0.
	$
\end{assumption}

%Conditional independence of the client oracle draws is sufficient for the last condition. The variance bound controls stochastic fluctuations around each client's gradient; differences between client gradients are controlled separately below.

\begin{assumption}[Bounded heterogeneity]
	\label{ass:final_hetero_rev}
	There exists $\zeta\ge0$ such that, for all $W$,
	$
	\frac1N\sum_{i=1}^N
	\|\nabla f_i(W)-\nabla f(W)\|_F^2
	\le\zeta^2.
	$
\end{assumption}

\begin{assumption}[Aggregate global-gradient coverage]
	\label{ass:partial_cover_client_cover_1}
	There exists $\mu_P\in(0,1]$ such that, almost surely,
	for every round $t$,
$
	\left\langle
	\nabla f(W^t),A_t\nabla f(W^t)
	\right\rangle_F
	=
	\frac1N\sum_{i=1}^N
	\|\Pi_i^t\nabla f(W^t)\|_F^2
	\ge
	\mu_P\|\nabla f(W^t)\|_F^2.
	$
\end{assumption}

Assumption~\ref{ass:partial_cover_client_cover_1} requires coverage only along the current global gradient. For shared projections, it becomes
$\|\Pi^t\nabla f(W^t)\|_F^2\ge\mu_P\|\nabla f(W^t)\|_F^2$;
for client-specific projections, coverage is through $A_t$.
Both methods are compared using a common valid $\mu_P$.
Orthogonality directly implies
$\|\Pi_i^tX\|_F\le\|X\|_F$ and $0\preceq A_t\preceq I$,
without additional assumptions.

%\subsection{Main theorem}

We analyze the projected-SGD variants of
\texttt{Local GaLore} and \texttt{FedLore}
under the sampling and conditioning convention stated above.
Let $S$ denote the number of participating clients per round,
and define
$
\Delta:=f(W^0)-f^\star,
\omega_S:=\frac{N-S}{S(N-1)}.
$
Full participation corresponds to $S=N$ and $\omega_S=0$.
For each round $t$, write
$
G_t:=\nabla f(W^t),
d_i^t:=\nabla f_i(W^t)-G_t,
A_t:=\frac1N\sum_{i=1}^N\Pi_i^t,
$
where $\Pi_i^t=P_i^t(P_i^t)^\top$.
Define the projection-induced aggregation bias
$
	b_t:=
	\frac1N\sum_{i=1}^N
	(\Pi_i^t-A_t)d_i^t.
	\label{eq:main_projection_bias}
$
Since $\sum_i d_i^t=0$, the round-start projected direction
admits the exact decomposition
$
	\frac1N\sum_{i=1}^N
	\Pi_i^t\nabla f_i(W^t)
	=
	A_tG_t+b_t.
	\label{eq:main_projection_decomposition}
$
For a horizon $T$, choose
$
\eta=\frac{c}{\alpha L K\sqrt T},
0<c\le\frac{\mu_P}{16},
$
and define the common remainder
$
		\mathcal R_T
		:=\;
		\frac{4L\Delta}{\mu_P c\sqrt T}
		+\frac{40c^2\zeta^2}{3\mu_P^2T}
		+\frac{12c^2\sigma^2}{\mu_P^2KT}
		+
		\frac{4c}{\mu_P\sqrt T}
		\left(
		\omega_S\zeta^2+\frac{\sigma^2}{SK}
		\right).
$

\begin{theorem}[Nonconvex convergence of \texttt{Local GaLore-SGD}]
	\label{thm:main_final_rev}
	Suppose
	Assumptions~\ref{ass:final_smooth_rev}--\ref{ass:partial_cover_client_cover_1}.
	With the stepsize specified above,
	\texttt{Local GaLore-SGD} satisfies
	\vspace{-2mm}
	\begin{equation}
		\begin{aligned}
			\frac1T\sum_{t=0}^{T-1}
			\mathbb E\|\nabla f(W^t)\|_F^2
			\le\;
			\mathcal R_T
			+
			\frac{8}{\mu_P^2T}
			\sum_{t=0}^{T-1}
			\mathbb E\|b_t\|_F^2.
		\end{aligned}
		\label{eq:main_bound_final_rev_1}
	\end{equation}
	Moreover, let
	$
	\bar v_T:=
	\frac1T\sum_{t=0}^{T-1}
	\mathbb E\|A_t-A_t^2\|_2.
	$
	Then
	$
		\frac1T\sum_{t=0}^{T-1}
		\mathbb E\|\nabla f(W^t)\|_F^2
		\le
		\mathcal R_T
		+\frac{8\zeta^2}{\mu_P^2}\bar v_T.
		\label{eq:main_local_dispersion_bound}
		$
\end{theorem}

\begin{theorem}[Nonconvex convergence of \texttt{FedLore-SGD}]
	\label{cor:shared_partial_explicit_clean_1}
	Suppose
	Assumptions~\ref{ass:final_smooth_rev}--
	\ref{ass:partial_cover_client_cover_1}
	and the sampling and conditioning convention hold.
	If all clients share the same round-wise projector,
	$\Pi_i^t=\Pi^t$, then $b_t=0$ for every round.
	With the same stepsize,
	\texttt{FedLore-SGD} satisfies
		\vspace{-2mm}
	\begin{equation}
		\frac1T\sum_{t=0}^{T-1}
		\mathbb E\|\nabla f(W^t)\|_F^2
		\le
		\mathcal R_T.
		\label{eq:main_shared_convergence}
	\end{equation}
	For fixed problem parameters, $K$, $S$,
	and $c$, its averaged squared gradient norm is
	$O(T^{-1/2})$.
\end{theorem}

\noindent\textbf{Theoretical interpretation.}
Theorems~\ref{thm:main_final_rev}
and~\ref{cor:shared_partial_explicit_clean_1}
isolate the effect of projection disagreement under a common
global-gradient coverage condition.
Client-specific projections introduce the coupling term $b_t$
between projection disagreement and gradient heterogeneity.
Its magnitude satisfies
$
\|b_t\|_F^2
\le
\|A_t-A_t^2\|_2
\frac1N\sum_{i=1}^N\|d_i^t\|_F^2.
$
With a shared projector, the zero-mean identity
$\sum_i d_i^t=0$ is preserved after projection, so this bias
vanishes exactly.
Using common valid constants for both methods,
the shared-projection bound therefore removes the additional
bias term in~\eqref{eq:main_bound_final_rev_1}.
The remaining terms account for initial suboptimality,
local drift, client sampling, and stochastic-gradient noise,
and decay under the stated horizon-dependent stepsize.
The local-projection bound also vanishes if its average
bias energy tends to zero; a nonvanishing term in an upper
bound alone does not establish an unavoidable error floor.
The strict quadratic separation in
Theorem~\ref{thm:flsp_separation}
shows that the obstruction can nevertheless occur:
exact local SVD projections may stall at a nonstationary
point, whereas a globally informed shared projection
converges to the optimum.

\textbf{A sharper guarantee for shared projections.}
The shared-projection structure yields a stronger result than simply
setting $b_t=0$ in the generic bound. Let
$h_t=K\gamma_t$, $\Lambda_T=\sum_{t<T}h_t$, and assume
$Lh_t\leq 1/8$. As shown in Theorem~\ref{thm:flsp_shared_residual},
without a global-gradient coverage assumption \ref{ass:partial_cover_client_cover_1},
$
	\frac{1}{\Lambda_T}\sum_{t<T}h_t
	\mathbb{E}\!\left[\|\Pi^t\nabla f(W^t)\|_F^2\right]
	\leq B_T.
$
Consequently,
$
	\frac{1}{\Lambda_T}\sum_{t<T}h_t
	\mathbb{E}\!\left[\|\nabla f(W^t)\|_F^2\right]
	\leq B_T+
	\frac{1}{\Lambda_T}\sum_{t<T}h_t
	\mathbb{E}\!\left[
	\|(I-\Pi^t)\nabla f(W^t)\|_F^2
	\right].
$
Thus, shared projections control projected stationarity without
coverage, while full-gradient stationarity additionally requires
control of the uncovered component. If
$\mathbb{E}\|\Pi^t\nabla f(W^t)\|_F^2
\geq\mu\mathbb{E}\|\nabla f(W^t)\|_F^2$
for every $t$, the full-gradient bound becomes $B_T/\mu$.
The stepsize condition $Lh_t\leq 1/8$ does not depend on $\mu$.
%\definecolor{LightRed}{RGB}{252,235,235}
%\definecolor{LightBlue}{RGB}{235,245,255}

\begin{table*}[tb]
	\centering
	\caption{
		Test accuracy (\%) under Dir-0.1 over $T=100$.
		Experiments are conducted with 50 clients, batch size 16, local steps $K=50$, and LoRA rank $r=8$, projection rank $r=8$.
		Best results are shown in bold and second-best results are underlined.
	}
	\label{tab:merged_results}
	
	\setlength{\tabcolsep}{4pt}
	\renewcommand{\arraystretch}{1.15}
	\resizebox{\textwidth}{!}{
		\begin{tabular}{lcccccccc}
			\toprule
			\multirow{2}{*}{\textbf{Method}}
			& \multirow{2}{*}{\textbf{Venue}}
			& \multicolumn{2}{c}{\textbf{CIFAR-100}}
			& \multicolumn{2}{c}{\textbf{Tiny-ImageNet}}
			& \multicolumn{2}{c}{\textbf{Food-101}}
			& \multirow{2}{*}{\textbf{Avg.}}
			\\
			\cmidrule(lr){3-4}
			\cmidrule(lr){5-6}
			\cmidrule(lr){7-8}
			
			&
			& \textbf{Swin-Base}
			& \textbf{ViT-Base}
			& \textbf{Swin-Base}
			& \textbf{ViT-Base}
			& \textbf{Swin-Base}
			& \textbf{ViT-Base}
			&
			\\
			\midrule
			FedIT
			& ICASSP'24
			& $90.17\pm0.28$ & $89.57\pm0.28$
			& $86.50\pm0.22$ & $84.88\pm0.12$
			& $91.55\pm0.18$ & $91.02\pm0.20$
			& $88.95$
			\\
			
			FFA-LoRA
			& ICLR'24
			& $90.07\pm0.18$ & $86.26\pm0.14$
			& $85.34\pm0.18$ & $77.89\pm0.11$
			& $91.25\pm0.12$ & $90.68\pm0.22$
			& $86.92$
			\\
			
			FlexLoRA
			& ICLR'25
			& \textbf{$90.91\pm0.19$} & $89.59\pm0.19$
			& $85.94\pm0.19$ & $85.55\pm0.19$
			& $91.82\pm0.16$ & $91.34\pm0.10$
			& $89.19$
			\\
			
			LoRA-FAIR
			& ICCV'25
			& $90.26\pm0.19$ & $89.86\pm0.19$
			& $86.59\pm0.15$ & $85.11\pm0.19$
			& $91.41\pm0.12$ & $91.23\pm0.14$
			& $89.08$
			\\
			
			RoLoRA
			& NeurIPS'25
			& $90.55\pm0.26$ & $88.55\pm0.26$
			& $86.03\pm0.28$ & $82.97\pm0.28$
			& $90.54\pm0.15$ & $90.32\pm0.08$
			& $88.16$
			\\
			
			FLoRA
			& NeurIPS'24
			& $86.32\pm0.19$ & $85.48\pm0.19$
			& $87.20\pm0.12$ & $83.69\pm0.15$
			& $87.12\pm0.21$ & $86.56\pm0.11$
			& $86.06$
			\\
			
			FRLoRA
			& ICLR'25
			& $87.17\pm0.16$ & $66.80\pm0.19$
			& $76.45\pm0.22$ & $42.04\pm0.08$
			& $88.12\pm0.26$ & $87.96\pm0.17$
			& $74.76$
			\\
			
			\hdashline
			
			Local GaLore
			& -
			& $90.22\pm0.12$ & $89.62\pm0.21$
			& $86.47\pm0.16$ & $84.96\pm0.12$
			& $91.61\pm0.11$ & $91.27\pm0.15$
			& $89.03$
			\\
			
			FedFull
			& AISTATS'17
			& \underline{$90.81\pm0.12$}
			& \underline{$90.26\pm0.10$}
			& \underline{$87.50\pm0.08$}
			& \underline{$87.02\pm0.20$}
			& \underline{$92.25\pm0.12$}
			& \underline{$92.12\pm0.10$}
			& \underline{$89.99$}
			\\
			\rowcolor{LightBlue!40}
			\texttt{FedLore}
			& Ours
			& $\mathbf{90.85\pm0.08}$
			& $\mathbf{90.55\pm0.08}$
			& $\mathbf{87.52\pm0.10}$
			& $\mathbf{87.11\pm0.10}$
			& $\mathbf{92.32\pm0.16}$
			& $\mathbf{92.21\pm0.17}$
			& $\mathbf{90.09}$
			\\
			
			\bottomrule
		\end{tabular}
	}
\end{table*}
\begin{table*}[tb]
	\centering
	\caption{
		Test accuracy (\%) of RoBERTa-Base with LoRA on six text classification tasks under FL.
		Experiments are conducted for $T=100$ rounds with 20 clients, Dir-0.5 data partition, batch size 16, local steps $K=50$, and LoRA rank $r=8$.
	}
	\label{tab:roberta_base_glue}
	
	\setlength{\tabcolsep}{4pt}
	\renewcommand{\arraystretch}{1.15}
	\resizebox{\textwidth}{!}{
		\begin{tabular}{lcccccccc}
			\toprule
			\textbf{Method}
			& \textbf{Venue}
			& \textbf{SNLI}
			& \textbf{AG News}
			& \textbf{QQP}
			& \textbf{DBPedia 14}
			& \textbf{QNLI}
			& \textbf{MNLI}
			& \textbf{Avg.}
			\\
			\midrule
			
			FedIT
			& ICASSP'24
			& $86.50\pm0.13$ & $92.68\pm0.09$
			& $85.06\pm0.14$ & $98.94\pm0.07$
			& $89.47\pm0.11$ & $81.77\pm0.15$
			& $89.07$
			\\
			
			FFA-LoRA
			& ICLR'24
			& $80.70\pm0.17$ & $90.14\pm0.12$
			& $81.48\pm0.16$ & $98.22\pm0.10$
			& $85.61\pm0.14$ & $72.44\pm0.18$
			& $84.77$
			\\
			
			FlexLoRA
			& ICLR'25
			& $86.64\pm0.15$ & $93.30\pm0.11$
			& $63.18\pm0.19$ & $98.06\pm0.09$
			& $49.46\pm0.20$ & $35.44\pm0.18$
			& $71.01$
			\\
			
			LoRA-FAIR
			& ICCV'25
			& $86.76\pm0.10$ & $92.76\pm0.08$
			& $85.33\pm0.12$ & $98.88\pm0.07$
			& $89.47\pm0.09$ & $82.04\pm0.11$
			& $89.21$
			\\
			
			RoLoRA
			& NeurIPS'25
			& $85.95\pm0.12$ & $92.06\pm0.10$
			& $84.51\pm0.13$ & $98.78\pm0.08$
			& $88.30\pm0.12$ & $80.58\pm0.14$
			& $88.36$
			\\
			
			FLoRA
			& NeurIPS'24
			& $69.36\pm0.18$ & $89.77\pm0.13$
			& $79.78\pm0.15$ & $98.22\pm0.11$
			& $81.93\pm0.16$ & $78.23\pm0.17$
			& $82.88$
			\\
			
			FRLoRA
			& ICLR'25
			& $86.67\pm0.11$ & $93.00\pm0.09$
			& $85.04\pm0.12$ & $98.90\pm0.08$
			& $89.16\pm0.10$ & $81.81\pm0.13$
			& $89.10$
			\\
			
			\hdashline
			
			Local GaLore
			& -
			& $86.65\pm0.11$ & $92.71\pm0.12$
			& $85.21\pm0.15$ & $98.96\pm0.08$
			& $89.52\pm0.12$ & $82.07\pm0.16$
			& $89.19$
			\\
			
			FedFull
			& AISTATS'17
			& \underline{$88.02\pm0.09$}
			& \underline{$93.72\pm0.08$}
			& \underline{$86.31\pm0.14$}
			& \underline{$99.02\pm0.17$}
			& \underline{$90.16\pm0.10$}
			& \underline{$83.85\pm0.16$}
			& \underline{$90.18$}
			\\
			\rowcolor{LightBlue!40}
			\texttt{FedLore}
			& Ours
			& $\mathbf{88.19\pm0.08}$
			& $\mathbf{93.97\pm0.07}$
			& $\mathbf{86.79\pm0.14}$
			& $\mathbf{99.10\pm0.17}$
			& $\mathbf{90.48\pm0.10}$
			& $\mathbf{84.11\pm0.16}$
			& $\mathbf{90.44}$
			\\
			
			\bottomrule
		\end{tabular}
	}
\end{table*}

\begin{table*}[t]
	\centering
	\caption{
		Comparison with low-rank federated algorithms for pre-training LLaMA models on the C4 dataset.
		We report validation loss (Val Loss), validation perplexity (Val PPL), communication cost (Comm.), and  local training time.
		All methods are trained for 100 communication rounds (20 clients, batch size 16, $K=50$).
		The projection rank of \texttt{FedLore} is identical to  LoRA rank.
	}
	\label{tab:galore_compare_full}
	\setlength{\tabcolsep}{5pt}
	\renewcommand{\arraystretch}{1.15}
		\resizebox{\textwidth}{!}{
	\begin{tabular}{lccccccccccc}
		\toprule
		\multirow{2}{*}{\textbf{Method}}
		& \multicolumn{2}{c}{\textbf{60M}}
		& \multicolumn{2}{c}{\textbf{130M}}
		& \multicolumn{2}{c}{\textbf{350M}}
		& \multicolumn{2}{c}{\textbf{1B}}
		& \multirow{2}{*}{\textbf{Comm. $\downarrow$}}
		& \multirow{2}{*}{\textbf{Time/Round $\downarrow$}}\\
		
		\cmidrule(lr){2-3}
		\cmidrule(lr){4-5}
		\cmidrule(lr){6-7}
		\cmidrule(lr){8-9}
		
		& Loss & PPL
		& Loss & PPL
		& Loss & PPL
		& Loss & PPL
		& &on LLaMA 350M (s)\\
		
		\midrule
		
		FedIT
		& 8.811 & 6707.62
		& 8.544 & 5135.85
		& 8.072 & 3203.50
		& 7.561 & 1921.77
		& $O(rd)$
		& 118.6
		\\
		
		Local GaLore
		& 4.326 & 75.64
		& 4.215 & 67.69
		& 3.962 & 52.56
		& 3.781 & 43.86
		& $O(d^2)$
		& 136.6
		\\
		
		FedFull
		& 4.134 & 62.43
		& 4.056 & 57.74
		& 3.798 & 44.61
		& 3.561 & 35.20
		& $O(d^2)$
		& 132.1
		\\
		
		\midrule
		
		\texttt{FedLore-Random}
		& 4.112 & 61.07
		& 4.036 & 56.60
		& 3.781 & 43.86
		& 3.523 & 33.89
		& $O(rd)$
		& 119.8
		\\
		\rowcolor{LightBlue!40}
		\texttt{FedLore}
		& \textbf{3.833} & \textbf{46.20}
		& \textbf{3.611} & \textbf{37.00}
		& \textbf{3.554} & \textbf{34.95}
		& \textbf{3.345} & \textbf{28.36}
		& $O(rd)$
		& 122.4
		\\
		
		\midrule
		
		$r / d_{\text{model}}$
		& \multicolumn{2}{c}{128 / 512}
		& \multicolumn{2}{c}{256 / 768}
		& \multicolumn{2}{c}{256 / 1024}
		& \multicolumn{2}{c}{512 / 2048}
		& &
		\\
		
		\bottomrule
	\end{tabular}
}
\end{table*}

\begin{figure*}[tb]
	\centering

	\begin{subfigure}[b]{0.245\textwidth}
		\centering
		\includegraphics[width=\linewidth,trim=6 4 6 4,clip]{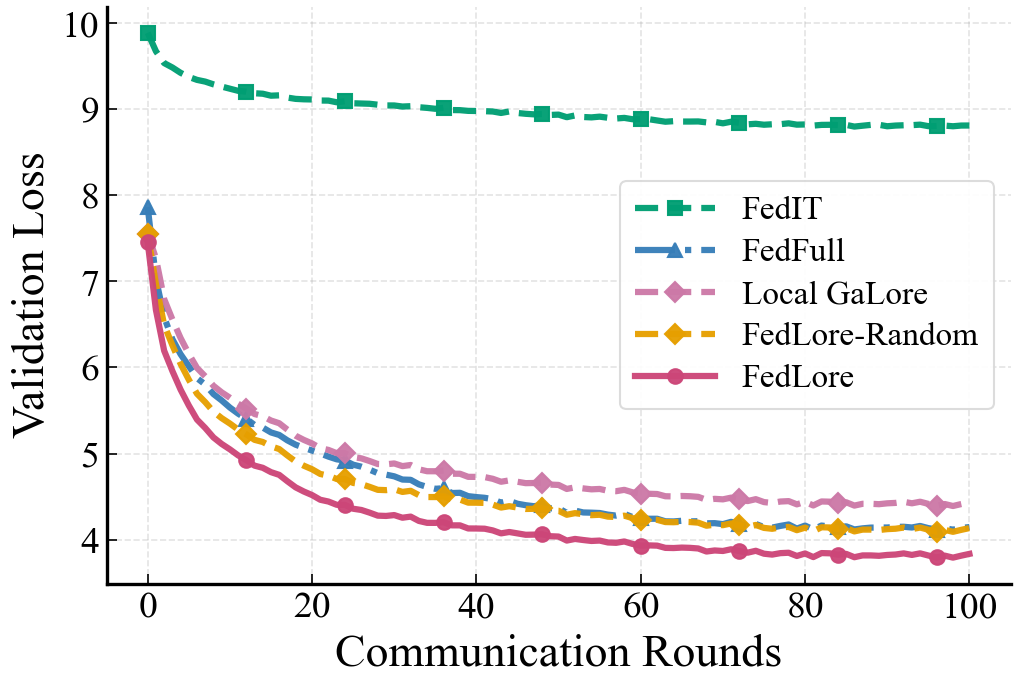}
		\caption{LLaMA 60M, C4}
		\label{fig:vit_cifar100}
	\end{subfigure}
	\hspace{-1mm}
	\begin{subfigure}[b]{0.245\textwidth}
		\centering
		\includegraphics[width=\linewidth,trim=6 4 6 4,clip]{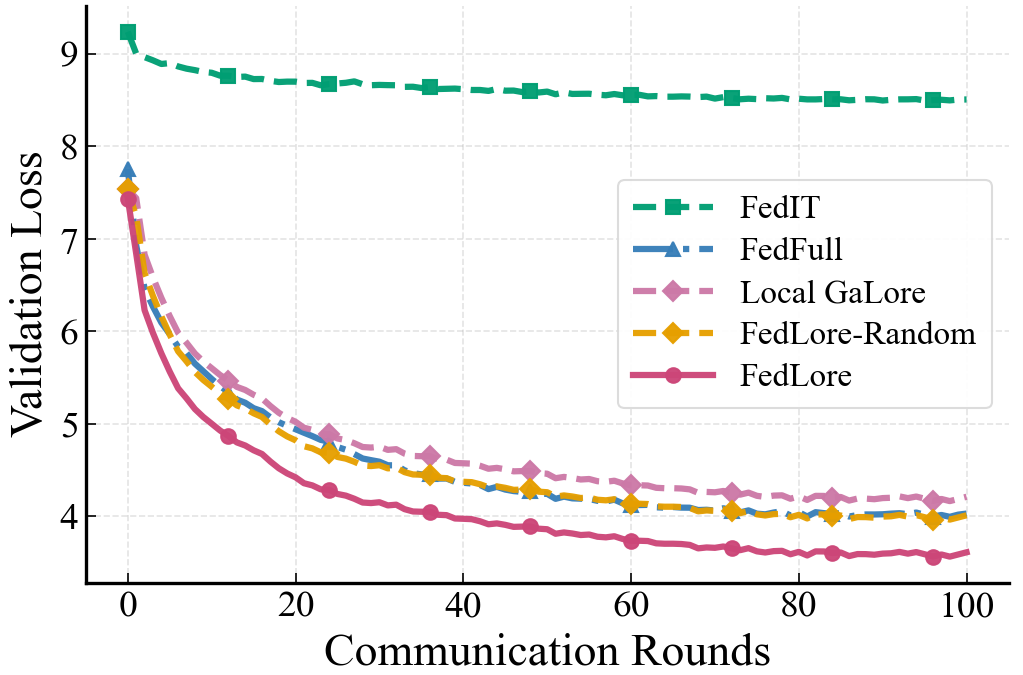}
		\caption{LLaMA 130M, C4}
		\label{fig:llama_130m}
	\end{subfigure}
	\hspace{-1mm}
	\begin{subfigure}[b]{0.245\textwidth}
		\centering
		\includegraphics[width=\linewidth,trim=6 4 6 4,clip]{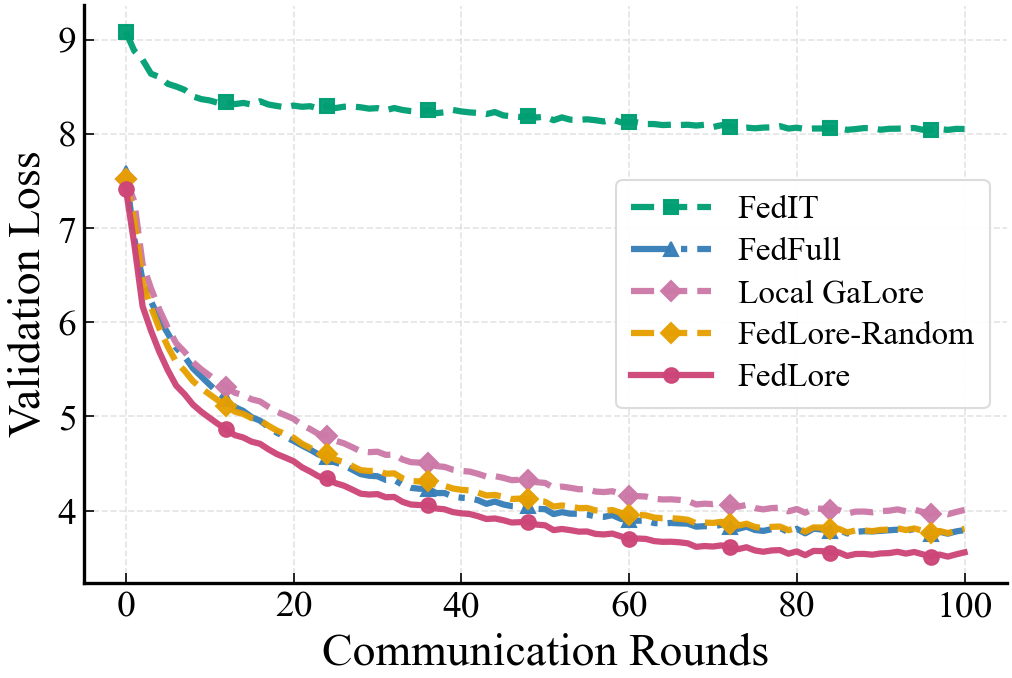}
		\caption{LLaMA 350M, C4}
		\label{fig:llama_350m}
	\end{subfigure}
		\hspace{-1mm}
	\begin{subfigure}[b]{0.245\textwidth}
		\centering
		\includegraphics[width=\linewidth,trim=6 4 6 4,clip]{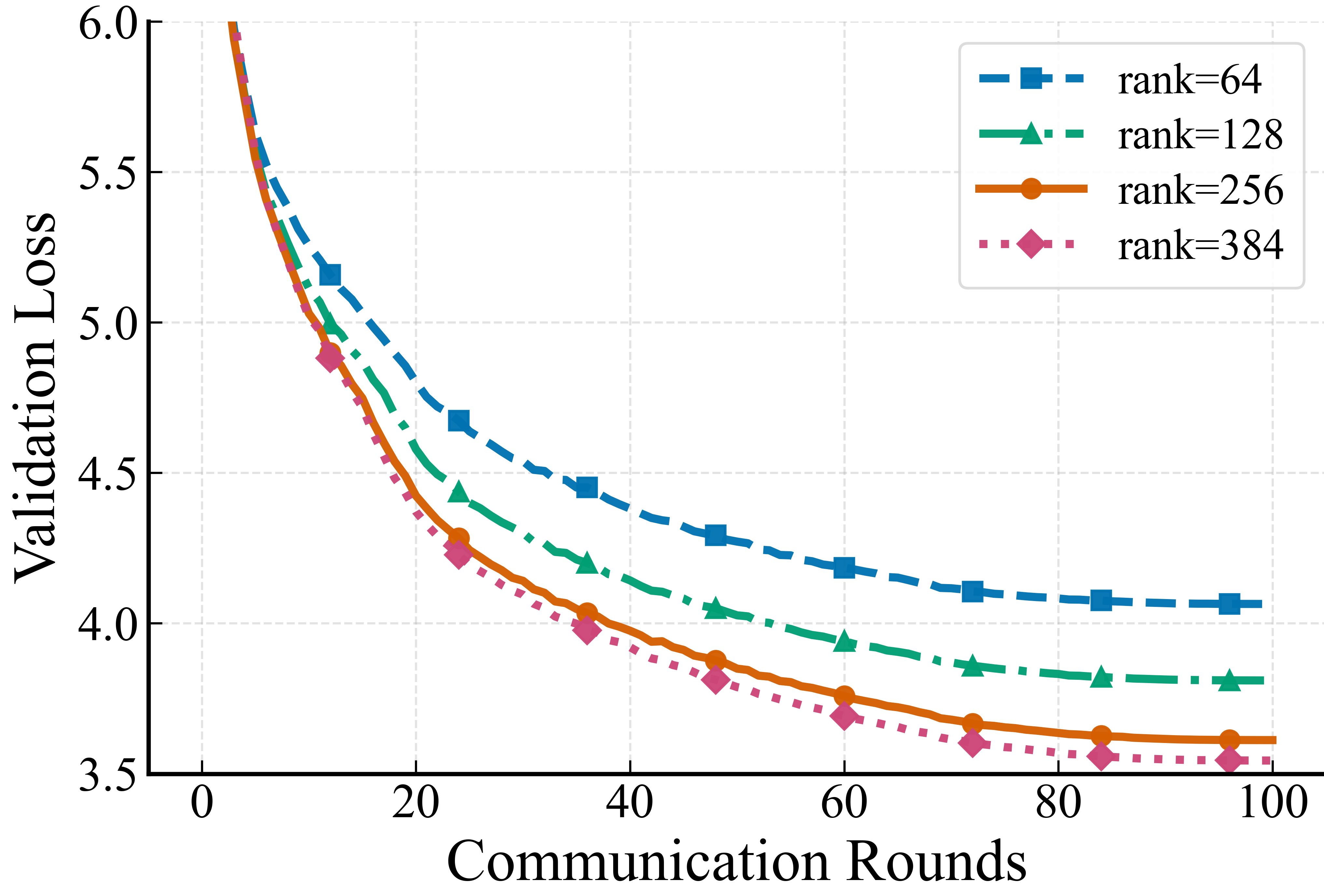}
		\caption{Effect of rank $r$}
		\label{fig:llama_350m}
	\end{subfigure}
	
	\vspace{-1mm}
	\caption{ (a,b,c): Validation loss convergence curves on C4 with LLaMA models of different scales. (d): Effect of the projection rank $r$ on validation loss for LLaMA 350M.}
	\label{fig:llama}
	\vspace{-2mm}
\end{figure*}

\section{Experiments}
\vspace{-2mm}
\subsection{Experiments Settings}
\vspace{-2mm}
\textbf{Datasets.} We evaluate \texttt{FedLore} on both vision and language tasks. (\textit{i}) For image classification, we use CIFAR-100~\cite{krizhevsky2009learning}, Tiny-ImageNet~\cite{le2015tiny}, Food-101 \cite{bossard2014food}. (\textit{ii}) For NLP tasks pretraining, we adopt C4 \cite{raffel2020exploring} dataset. For language understanding, we evaluate SNLI, AG News, QQP, DBPedia-14, QNLI, and MNLI. To simulate data heterogeneity across clients, we follow the Dirichlet partitioning scheme~\cite{hsu2019measuring}.\\% For Dir-$\alpha$ partitioning, smaller $\alpha$ indicates more severe data heterogeneity.\\
\textbf{Model Architectures.} We explore a variety of model types:  (\textit{i}) ViT-Base~\cite{dosovitskiy2020image}, Swin-Base \cite{liu2021swin} and ViT-Tiny~\cite{dosovitskiy2020image}, and (\textit{ii}) LLaMA \cite{touvron2023llama} for large-scale language model and RoBERTa-Base \cite{liu2019roberta}.\\
\textbf{Baselines.} We compare our method against  sota FL algorithms:  FedIT \cite{zhang2023towards}, FFA-LoRA \cite{sun2024improving}, FlexLoRA \cite{bai2024federated}, RoLoRA \cite{chen2024robust}, FLoRA \cite{wang2024flora},  FRLoRA \cite{yan2025federated}, LoRA-FAIR \cite{bian2025lora}, FedFull, and Local Galore.
\\
\textbf{Hyperparameter Settings.} 
The $lr$ is selected from 
$\{10^{-4},\ 3 \times 10^{-4},\ 5 \times 10^{-4},\ 8 \times 10^{-4},\ 10^{-3},\  3 \times10^{-3},\ 5\times10^{-3}\}$. 
Their  weight decay is  selected from 
$\{ 0 ,\ 10^{-3},\  3 \times10^{-3},\ 1\times10^{-2},\ 3\times10^{-2}\}$. For \texttt{FedLore}, $\alpha=1$. We apply cosine learning rate decay. Additional hyperparameter configurations are detailed in the \textbf{Appendix}. We release all code to ensure full reproducibility upon acceptance. 
All results are averaged over 5 runs  with seeds {42, 43, 44, 45, 46} on  NVIDIA RTX 4090 GPU.

%\textbf{Evaluation Metrics.}
%For vision tasks, we report test accuracy and training loss after 100 communication rounds. For language tasks, we report task-specific accuracy metrics used in the GLUE benchmark.
\subsection{Results on Fine-tuning.}
\vspace{-2mm}
\textbf{Vision Benchmark Results.}
As shown in Table~\ref{tab:merged_results}, \texttt{FedLore} consistently matches or exceeds the performance of existing federated fine-tuning approaches under the challenging Dir-0.1 heterogeneous setting. It achieves the highest average accuracy of 90.09\%, compared with 89.99\% for the full-parameter baseline \texttt{FedFull}. Across all six dataset combinations, \texttt{FedLore} either achieves the best result or remains statistically competitive with the strongest baseline. The advantages are particularly noticeable on ViT-Base, where \texttt{FedLore} reaches 90.55\% on CIFAR-100 and 87.11\% on Tiny-ImageNet. Overall, these results indicate that dynamic low-rank subspace optimization provides a more flexible and effective alternative to LoRA training in heterogeneous federated environments.

\textbf{Language Benchmark Results.}
As shown in Table~\ref{tab:roberta_base_glue}, \texttt{FedLore} delivers the strongest overall performance on federated language understanding. It ranks first on all six tasks and achieves the best average accuracy of 90.44\%, surpassing \texttt{FedFull} by 0.26 points and the best LoRA-based baseline by more than 1 point. The improvements are particularly pronounced on QQP, QNLI, and MNLI, suggesting that \texttt{FedLore} is more effective than LoRA training in handling heterogeneous NLP tasks. These results confirm that \texttt{FedLore} is not limited to vision models, but also provides a strong and communication-efficient alternative to  FL for large language transformers.

\subsection{Results on LLaMA pre-training.}

As shown in Table~\ref{tab:galore_compare_full} and Figure~\ref{fig:llama},
\texttt{FedLore} consistently outperforms existing federated baselines across different LLaMA scales.
Compared with \texttt{FedIT}, \texttt{FedLore} significantly reduces validation perplexity, from 6707.62 to 46.20 on LLaMA-60M and from 1921.77 to 28.36 on LLaMA-1B, demonstrating that fixed low-rank adapters severely limit the optimization capacity for federated pre-training.
Meanwhile, \texttt{FedLore} achieves better performance than full-parameter training while retaining low-rank communication and memory efficiency, e.g., obtaining a PPL of 28.36 on LLaMA-1B compared with 35.20 of \texttt{FedFull}.
The inferior performance of \texttt{FedLore-Random}, which uses a shared random projection, further verifies that the gains come from adaptive and data-driven subspace evolution rather than shared projection alone.

\begin{table}[tb]
	\centering
	\caption{Communication and optimizer memory on LLaMA 350M (left), and component ablations on C4 pre-training (right). All costs are in GB.}
	\label{tab:fedlore_system_ablation}
	\vspace{-1mm}
	\setlength{\tabcolsep}{4pt}
	\setlength{\aboverulesep}{0pt}
	\setlength{\belowrulesep}{0pt}
	\renewcommand{\arraystretch}{1.2}
	\footnotesize
	\resizebox{\linewidth}{!}{
		\begin{tabular}{lcc!{\vrule width 1pt}lcc}
			\toprule
			\textbf{Method}
			& \textbf{Comm.}
			& \textbf{Opt. Memory}
			& \textbf{Ablation Setting}
			& \textbf{Val Loss / PPL}
			& \textbf{Comm.} $\downarrow$ \\
			\midrule
			FedIT
			& $0.36$ & $0.72$
			& Local GaLore
			& 3.962 / 52.56 & $0.68$ \\
			
			FedFull
			& $0.68$ & $1.37$
			& \texttt{FedLore}
			& \textbf{3.554 / 34.95} & $0.36$ \\
			
			Local GaLore
			& $0.68$ & $0.54$
			& w/o Shared Projector
			& 3.890 / 48.91 & $0.68$ \\
			
			\texttt{FedLore}
			& $0.36$ & $0.54$
			& w/o Dynamic Subspace
			& 5.952 / 384.7 & $0.18$ \\
			\bottomrule
		\end{tabular}
	}
	\vspace{-1mm}
\end{table}

\subsection{Ablation Study}

\textbf{Effect of Subspace Rank.} Figure~\ref{fig:llama}(d) studies the effect of the projection rank $r$ on LLaMA-350M pre-training.
Larger ranks accelerate convergence and reduce the final validation loss, since the projected subspace captures more informative update directions.
The improvement is significant from $r=64$ to $r=256$, while increasing $r$ to $384$ brings only marginal gains.
This indicates that \texttt{FedLore} achieves a favorable accuracy and communication trade-off with a moderate rank.

\textbf{Ablation study.}
Table~\ref{tab:fedlore_system_ablation} examines the roles of
shared projections and dynamic subspace updates in
LLaMA-350M pre-training on C4.
Removing the shared projector increases validation loss
from 3.554 to 3.890 and communication from 0.36 to 0.68 GB
under the evaluated protocol, supporting the benefits of
shared optimization coordinates for both optimization
and communication efficiency.
Fixing the shared subspace reduces communication to
0.18 GB but increases validation loss to 5.952,
indicating that subspace refresh is important for
maintaining optimization capacity throughout training.
Together, these results support the complementary roles
of subspace sharing and refresh in balancing
model quality and communication cost.

\textbf{Communication and Memory Efficiency Analysis.}
Table~\ref{tab:fedlore_system_ablation} summarizes the communication and memory efficiency on LLaMA 350M.
FedFull incurs the largest cost, requiring $0.68$ GB communication cost and $1.37$ GB optimizer memory.
Local GaLore reduces optimizer memory to $0.54$ GB, but still communicates full updates.
FedIT lowers communication to $0.36$ GB, while requiring $0.72$ GB optimizer memory.
By contrast, \texttt{FedLore} achieves the best trade-off, with $0.36$ GB total communication and $0.54$ GB optimizer memory.
This represents approximately $47\%$ lower communication than FedFull/Local GaLore.

\section{Conclusion}

We introduced \texttt{FedLore}, a communication- and memory-efficient framework for FL foundation-model training. A shared, evolving basis enables exact low-rank aggregation without imposing a fixed rank constraint on accumulated updates. Our analysis characterizes projection--heterogeneity coupling and establishes projected-SGD convergence under global-gradient coverage. Experiments across vision, language, and federated pre-training support the approach. Future work will develop coverage-aware subspace selection and efficient synchronization for intermittent participation.
%\newpage
%\bibliographystyle{plain}
%\bibliography{reference}

\newpage

\bibliography{aaai2027}
\bibliographystyle{iclr2027_conference}
\appendix
%\onecolumn
%\input{sections/proof}
% FedLore appendix: replace the existing appendix block with this file.
% Preamble: amsmath, amssymb, amsthm, booktabs, graphicx, multirow,
% float, algorithm, algorithmic, hyperref. Keep existing theorem definitions.
% Compile the main manuscript at least twice to resolve directory page numbers.
% Original equations, data, statements and existing labels are retained.
\clearpage
\onecolumn
\section*{Ethics Statement}

This work studies communication- and memory-efficient federated
learning using existing vision and language benchmarks:
CIFAR-100, Tiny-ImageNet, Food-101, C4, SNLI, AG News, QQP,
DBPedia-14, QNLI, and MNLI.
Our experiments simulate federated clients using benchmark
partitions rather than collecting data from real users.
Although FedLore keeps training data local within the federated
protocol, communicated updates and shared subspaces may still
expose information about local data. Neither federated training
nor low-rank projection alone provides a formal privacy guarantee;
FedLore does not establish differential privacy or resistance
to inference attacks.
Models trained with this approach may also inherit biases or
harmful content from the underlying datasets and pretrained
models. Deployment with sensitive data therefore requires
additional privacy safeguards and application-specific evaluation.

\section*{Reproducibility Statement}

The paper provides the FedLore algorithm and its theoretical
assumptions, with detailed derivations and proofs in the appendix.
The experimental settings in the main text and the appendices
on federated pre-training and fine-tuning describe the datasets,
model architectures, client partitions, participation settings,
communication rounds, local optimization steps, projection ranks,
optimizer configurations, and hyperparameter selection.
The paper also reports the hardware and the five random seeds
used for repeated experiments.
We will publicly release the implementation and experiment
scripts upon acceptance.

\section*{AI Use Statement}

We used large language models to assist with language editing,
including sentence rephrasing, grammar correction, and readability
improvements. We did not use these tools for research ideation,
methodology or experiment design, theoretical derivations,
method implementation, data generation or processing, or
interpretation of results.
The authors reviewed and revised all AI-assisted text and take
full responsibility for the final manuscript, including its
claims, analyses, and conclusions.
\clearpage
\appendix
\section*{Appendix Contents}
\phantomsection
\label{app:fedlore_contents}
\pdfbookmark[0]{Appendix Contents}{fedlore.contents}
\begingroup
\small
\setlength{\parindent}{0pt}
\setlength{\parskip}{0pt}
\newcommand{\FedLoreTOCEntry}[3]{%
	\par\noindent\hspace*{#1}%
	\hyperref[#2]{\ref*{#2}\quad #3}\nobreak\dotfill
	\hyperref[#2]{\pageref*{#2}}\par
}
\vspace{0.45em}
\FedLoreTOCEntry{0pt}{app:fedlore_outline_1}{\textbf{Details of Pre-Training Experiments}}
\FedLoreTOCEntry{1.5em}{app:fedlore_outline_2}{Architecture and Hyperparameters}
\FedLoreTOCEntry{1.5em}{app:fedlore_outline_3}{Federated Pre-Training Setup}
\FedLoreTOCEntry{1.5em}{app:fedlore_outline_4}{Optimization Details}
\FedLoreTOCEntry{1.5em}{app:fedlore_outline_5}{Communication and Memory Analysis}
\FedLoreTOCEntry{1.5em}{app:fedlore_outline_6}{Hardware and Reproducibility}
\FedLoreTOCEntry{1.5em}{app:fedlore_outline_7}{Learning Rate Selection}
\vspace{0.45em}
\FedLoreTOCEntry{0pt}{app:fedlore_finetuning}{\textbf{Details of Fine-Tuning Experiments}}
\FedLoreTOCEntry{1.5em}{app:fedlore_outline_9}{Vision Benchmark}
\FedLoreTOCEntry{1.5em}{app:fedlore_outline_10}{Language Benchmark}
\FedLoreTOCEntry{1.5em}{app:fedlore_outline_11}{Hyperparameter Selection}
\FedLoreTOCEntry{1.5em}{app:fedlore_outline_12}{Learning Rate Selection}
\vspace{0.45em}
\FedLoreTOCEntry{0pt}{app:fedlore_resources}{\textbf{Communication and Memory Overhead}}
\FedLoreTOCEntry{1.5em}{app:fedlore_outline_14}{Communication Cost}
\FedLoreTOCEntry{1.5em}{app:fedlore_outline_15}{Memory Estimates}
\vspace{0.45em}
\FedLoreTOCEntry{0pt}{app:flsp_analysis}{\textbf{Shared Projections in Federated Optimization}}
\FedLoreTOCEntry{1.5em}{sec:flsp_setup}{Setting and assumptions}
\FedLoreTOCEntry{1.5em}{sec:flsp_bias}{Projection disagreement and aggregation bias}
\FedLoreTOCEntry{1.5em}{sec:flsp_convergence}{Convergence under partial participation}
\FedLoreTOCEntry{1.5em}{sec:flsp_separation}{A strict separation with exact local SVD projections}
\FedLoreTOCEntry{1.5em}{sec:flsp_actual_selection}{Coverage residual of the implemented projector average}
\FedLoreTOCEntry{1.5em}{sec:flsp_selection}{Coverage of an approximate shared subspace}
\FedLoreTOCEntry{1.5em}{sec:flsp_sync}{Exact synchronization and optimizer scope}
\vspace{0.45em}
\FedLoreTOCEntry{0pt}{app:fedlore_baseline}{\textbf{Baseline Algorithm}}
\endgroup
\clearpage
\section{Details of Pre-Training Experiments}
\label{app:fedlore_outline_1}

\subsection{Architecture and Hyperparameters}
\label{app:fedlore_outline_2}

We provide detailed configurations of the LLaMA models and hyperparameters used in
federated pre-training experiments. We adopt the LLaMA architecture and evaluate
four model scales, including LLaMA-60M, LLaMA-130M, LLaMA-350M, and LLaMA-1B.
All models are pretrained on the C4 dataset under the federated setting.
For all experiments, we use a maximum sequence length of 256 and a batch size of
16 for each client. The training process contains 100 communication rounds, and
each round performs 50 local optimization steps.

\begin{table}[!htbp]
	\centering
	\caption{LLaMA architecture configurations used in
		federated pre-training. Model names follow the
		GaLore implementation.}
	\label{tab:llama_architecture}
	\begin{tabular}{lrrrr}
		\toprule
		Model & Hidden & Intermediate & Heads & Layers \\
		\midrule
		LLaMA-60M  & 512  & 1376 & 8  & 8  \\
		LLaMA-130M & 768  & 2048 & 12 & 12 \\
		LLaMA-350M & 1024 & 2736 & 16 & 24 \\
		LLaMA-1B   & 2048 & 5461 & 32 & 24 \\
		\bottomrule
	\end{tabular}
\end{table}

For optimization, we adopt AdamW in the shared low-rank subspace.
The learning rate is selected from
$\{10^{-4},3\times10^{-4},5\times10^{-4},
8\times10^{-4},10^{-3},3\times10^{-3},5\times10^{-3}\}$
according to the validation perplexity.
For FedLore, we set the projection rank according to the model scale and use
$\alpha=1$. The learning rate follows cosine decay during training.

\subsection{Federated Pre-Training Setup}
\label{app:fedlore_outline_3}

We follow the standard federated pre-training protocol. The C4 dataset is
partitioned among clients using a Dirichlet distribution to simulate
heterogeneous data distributions.
The detailed federated training configurations are summarized below:
Different from LoRA-based federated fine-tuning methods that restrict model
updates to fixed adapter spaces, FedLore performs optimization in dynamically
evolving shared low-rank subspaces while reconstructing updates in the original
parameter space. This enables full-parameter adaptation with low communication
and memory overhead.
\begin{table}[!htbp]
	\centering
	\caption{Federated pre-training configurations.}
	\label{tab:pretrain_config}
	\begin{tabular}{lc}
		\toprule
		Configuration & Value \\
		\midrule
		Dataset & C4 \\
		Number of clients & 20 \\
		Client participation & 100\% \\
		Communication rounds & 100 \\
		Local steps per round & 50 \\
		Batch size & 16 \\
		Sequence length & 256 \\
		Optimizer & AdamW \\
		Weight decay & 0 \\
		LR scheduler & Cosine decay \\
		\bottomrule
	\end{tabular}
\end{table}

\subsection{Optimization Details}
\label{app:fedlore_outline_4}

For all LLaMA pre-training experiments, we use AdamW optimizer with
$\beta_1=0.9$ and $\beta_2=0.95$. The weight decay is set to zero since weight
decay introduces a full-rank component into accumulated local updates, which
breaks the low-rank communication property.
Specifically, after $K$ local optimization steps, the local update can be written
as

\begin{equation}
	\Delta W_i^t =
	\sum_{k=0}^{K-1}\eta \alpha P^t
	\Delta Z_i^{t,k},
\end{equation}

where $\Delta Z_i^{t,k}$ is the low-rank update in the shared projection space.
Therefore, only the low-rank representation is communicated between clients and
the server.

\subsection{Communication and Memory Analysis}
\label{app:fedlore_outline_5}

For a weight matrix
$W\in\mathbb{R}^{m\times n}$,
FedLore communicates the projected update
$\Delta Z\in\mathbb{R}^{r\times n}$,
while the shared projection matrix is maintained by the server.
The communication complexity is therefore
\begin{equation}
	O(r(m+n)),
\end{equation}
which is comparable to LoRA-based methods.
The optimizer states are also maintained in the projected space with memory
complexity
\begin{equation}
	O(rn).
\end{equation}
The comparison with existing methods is shown in
Table~\ref{tab:communication_memory}.
\begin{table}[!htbp]
	\centering
	\caption{Communication and optimizer memory comparison.
		$r\ll\min(m,n)$.}
	\label{tab:communication_memory}
	\begin{tabular}{lcc}
		\toprule
		Method & Communication Cost & Optimizer Memory \\
		\midrule
		FedIT & $O(2r(m+n))$ & $O(r(m+n))$ \\
		FedFull & $O(2mn)$ & $O(mn)$ \\
		Local GaLore & $O(2mn)$ & $O(rn)$ \\
		FedLore & $O(2r(m+n))$ & $O(rn)$ \\
		\bottomrule
	\end{tabular}
\end{table}

\subsection{Hardware and Reproducibility}
\label{app:fedlore_outline_6}

All experiments are conducted on NVIDIA RTX 4090 GPUs.
For each experiment, we run five independent trials with random seeds
$\{42,43,44,45,46\}$ and report the average performance.
The implementation will be released to facilitate reproducibility.

\subsection{Learning Rate Selection}
\label{app:fedlore_outline_7}

To ensure a fair comparison, we perform learning rate tuning for all methods
under each model scale. The learning rate is selected from
$\{10^{-4},\ 3 \times 10^{-4},\ 5 \times 10^{-4},\ 8 \times 10^{-4},\ 10^{-3},\  3 \times10^{-3},\ 5\times10^{-3}\}$ according to the validation
perplexity on the C4 validation set.
Table~\ref{tab:lr_selection} reports the optimal learning rates used by
different methods for different LLaMA model sizes. We observe that
full-parameter methods generally require smaller learning rates due to the
large optimization space, while low-rank optimization methods can adopt larger
learning rates owing to the reduced optimization dimension. FedLore remains
stable across different model scales and consistently achieves competitive
performance with a relatively large learning rate.

\begin{table*}[!htbp]
	\centering
	\caption{
		Optimal learning rates for different methods and LLaMA model scales.
		The learning rate is selected according to validation perplexity on C4.
	}
	\label{tab:lr_selection}
	\resizebox{\textwidth}{!}{
		\begin{tabular}{lcccc}
			\toprule
			Method
			& LLaMA-60M
			& LLaMA-130M
			& LLaMA-350M
			& LLaMA-1B \\
			\midrule
			
			FedIT
			& $1\times10^{-3}$
			& $8\times10^{-4}$
			& $5\times10^{-4}$
			& $5\times10^{-4}$ \\
			
			Local GaLore
			& $3\times10^{-3}$
			& $3\times10^{-3}$
			& $1\times10^{-3}$
			& $1\times10^{-3}$ \\
			
			FedFull
			& $1\times10^{-3}$
			& $8\times10^{-4}$
			& $5\times10^{-4}$
			& $5\times10^{-4}$ \\
			
			FedLore-Random
			& $3\times10^{-3}$
			& $3\times10^{-3}$
			& $3\times10^{-3}$
			& $1\times10^{-3}$ \\
			
			FedLore
			& $3\times10^{-3}$
			& $3\times10^{-3}$
			& $3\times10^{-3}$
			& $1\times10^{-3}$ \\
			
			\bottomrule
		\end{tabular}
	}
\end{table*}

\clearpage
\section{Details of Fine-Tuning Experiments}
\label{app:fedlore_finetuning}

We provide detailed experimental configurations for the vision and language
federated fine-tuning experiments. All experiments are conducted under
heterogeneous federated learning settings, where client data are partitioned
using the Dirichlet distribution.

\subsection{Vision Benchmark}
\label{app:fedlore_outline_9}

For vision federated fine-tuning, we evaluate FedLore on three benchmark
datasets, including CIFAR-100, Tiny-ImageNet, and Food-101. We adopt two
representative vision transformer architectures, i.e., Swin-Base and ViT-Base,
as backbone models.
Following previous federated fine-tuning works, we simulate data
heterogeneity using a Dirichlet distribution with concentration parameter
$\alpha=0.1$. The federated training process contains 100 communication rounds
with 50 clients. All clients participate in each communication round.
The detailed experimental configurations are summarized in
Table~\ref{tab:vision_setting}.

\begin{table}[!htbp]
	\centering
	\caption{Experimental configurations for vision federated fine-tuning.}
	\label{tab:vision_setting}
	\begin{tabular}{lc}
		\toprule
		Configuration & Value \\
		\midrule
		Datasets & CIFAR-100, Tiny-ImageNet, Food-101 \\
		Backbone & Swin-Base, ViT-Base \\
		Number of clients & 50 \\
		Data partition & Dirichlet ($\alpha=0.1$) \\
		Communication rounds & 100 \\
		Local steps $K$ & 50 \\
		Batch size & 16 \\
		LoRA rank & 8 \\
		Projection rank & 8 \\
		Optimizer & AdamW \\
		Learning rate & Grid search \\
		LR scheduler & Cosine decay \\
		Weight decay & 0 \\
		\bottomrule
	\end{tabular}
\end{table}

\subsection{Language Benchmark}
\label{app:fedlore_outline_10}

For language federated fine-tuning, we evaluate FedLore with
RoBERTa-Base on six text classification benchmarks:
SNLI, AG News, QQP, DBPedia-14, QNLI, and MNLI.
The data among clients are partitioned using a Dirichlet distribution with
$\alpha=0.5$. We use 20 clients and train the model for 100 communication
rounds. Each client performs 50 local optimization steps in each round.
The detailed configurations are summarized in
Table~\ref{tab:nlp_setting}.

\begin{table}[!htbp]
	\centering
	\caption{Experimental configurations for language federated fine-tuning.}
	\label{tab:nlp_setting}
	\begin{tabular}{lc}
		\toprule
		Configuration & Value \\
		\midrule
		Datasets & SNLI, AG News, QQP, DBPedia-14, QNLI, MNLI \\
		Backbone & RoBERTa-Base \\
		Number of clients & 20 \\
		Data partition & Dirichlet ($\alpha=0.5$) \\
		Communication rounds & 100 \\
		Local steps $K$ & 50 \\
		Batch size & 16 \\
		LoRA rank & 8 \\
		Projection rank & 8 \\
		Optimizer & AdamW \\
		Learning rate & Grid search \\
		LR scheduler & Cosine decay \\
		Weight decay & 0 \\
		\bottomrule
	\end{tabular}
\end{table}

\subsection{Hyperparameter Selection}
\label{app:fedlore_outline_11}

For all federated fine-tuning experiments, the learning rate is selected from
\[
\{10^{-4},\ 3 \times 10^{-4},\ 5 \times 10^{-4},\ 8 \times 10^{-4},\ 10^{-3},\  3 \times10^{-3},\ 5\times10^{-3}\}
\]
according to the validation performance.
For all methods, we use the same search space to ensure a fair comparison.
The LoRA rank and projection rank are both set to 8 following previous
parameter-efficient federated fine-tuning approaches.
FedLore adopts AdamW optimizer with zero weight decay. The learning rate
follows cosine decay during training. All reported results are averaged over
five independent runs with random seeds
$\{42,43,44,45,46\}$.

\subsection{Learning Rate Selection}
\label{app:fedlore_outline_12}

To ensure a fair comparison, we perform learning rate tuning for all methods.
The learning rate is selected from
$\{10^{-4},\ 3 \times 10^{-4},\ 5 \times 10^{-4},\ 8 \times 10^{-4},\ 10^{-3},\  3 \times10^{-3},\ 5\times10^{-3}\}$
according to the validation performance.
Tables~\ref{tab:vision_lr} reports the selected learning
rates for vision and language federated fine-tuning experiments, respectively.
Different methods use their individually optimized learning rates.

\begin{table*}[!htbp]
	\centering
	\caption{
		Selected learning rates for vision federated fine-tuning experiments.
	}
	\label{tab:vision_lr}
	\resizebox{\textwidth}{!}{
		\begin{tabular}{lcccccc}
			\toprule
			Method &
			\multicolumn{2}{c}{CIFAR-100} &
			\multicolumn{2}{c}{Tiny-ImageNet} &
			\multicolumn{2}{c}{Food-101}\\
			\cmidrule(lr){2-3}
			\cmidrule(lr){4-5}
			\cmidrule(lr){6-7}
			&
			Swin-Base & ViT-Base &
			Swin-Base & ViT-Base &
			Swin-Base & ViT-Base\\
			\midrule
			
			FedIT
			& $1e^{-3}$ & $8e^{-4}$
			& $1e^{-3}$ & $8e^{-4}$
			& $1e^{-3}$ & $8e^{-4}$ \\
			
			FFA-LoRA
			& $3e^{-3}$ & $3e^{-3}$
			& $3e^{-3}$ & $3e^{-3}$
			& $3e^{-3}$ & $3e^{-3}$ \\
			
			FlexLoRA
			& $1e^{-3}$ & $1e^{-3}$
			& $1e^{-3}$ & $1e^{-3}$
			& $1e^{-3}$ & $1e^{-3}$ \\
			
			LoRA-FAIR
			& $3e^{-3}$ & $3e^{-3}$
			& $3e^{-3}$ & $3e^{-3}$
			& $3e^{-3}$ & $3e^{-3}$ \\
			
			RoLoRA
			& $1e^{-3}$ & $1e^{-3}$
			& $1e^{-3}$ & $1e^{-3}$
			& $1e^{-3}$ & $1e^{-3}$ \\
			
			FLoRA
			& $3e^{-3}$ & $3e^{-3}$
			& $3e^{-3}$ & $3e^{-3}$
			& $3e^{-3}$ & $3e^{-3}$ \\
			
			FRLoRA
			& $5e^{-4}$ & $5e^{-4}$
			& $5e^{-4}$ & $5e^{-4}$
			& $5e^{-4}$ & $5e^{-4}$ \\
			
			Local GaLore
			& $3e^{-3}$ & $3e^{-3}$
			& $3e^{-3}$ & $3e^{-3}$
			& $3e^{-3}$ & $3e^{-3}$ \\
			
			FedFull
			& $5e^{-4}$ & $5e^{-4}$
			& $5e^{-4}$ & $5e^{-4}$
			& $5e^{-4}$ & $5e^{-4}$ \\
			
			FedLore
			& $3e^{-3}$ & $3e^{-3}$
			& $3e^{-3}$ & $3e^{-3}$
			& $3e^{-3}$ & $3e^{-3}$ \\
			
			\bottomrule
	\end{tabular}}
\end{table*}

\begin{table}[!htbp]
	\centering
	\caption{
		Selected learning rates for RoBERTa-Base federated fine-tuning on different
		NLP datasets.
	}
	\label{tab:nlp_lr_selection}
	\resizebox{\columnwidth}{!}{
		\begin{tabular}{lcccccc}
			\toprule
			Method
			& SNLI
			& AG News
			& QQP
			& DBPedia-14
			& QNLI
			& MNLI \\
			\midrule
			
			FedIT
			& $5e^{-4}$
			& $1e^{-3}$
			& $5e^{-4}$
			& $1e^{-3}$
			& $5e^{-4}$
			& $5e^{-4}$ \\
			
			FFA-LoRA
			& $1e^{-3}$
			& $3e^{-3}$
			& $1e^{-3}$
			& $3e^{-3}$
			& $1e^{-3}$
			& $1e^{-3}$ \\
			
			FlexLoRA
			& $5e^{-4}$
			& $1e^{-3}$
			& $5e^{-4}$
			& $1e^{-3}$
			& $5e^{-4}$
			& $5e^{-4}$ \\
			
			LoRA-FAIR
			& $1e^{-3}$
			& $3e^{-3}$
			& $1e^{-3}$
			& $3e^{-3}$
			& $1e^{-3}$
			& $1e^{-3}$ \\
			
			RoLoRA
			& $1e^{-3}$
			& $3e^{-3}$
			& $1e^{-3}$
			& $3e^{-3}$
			& $1e^{-3}$
			& $1e^{-3}$ \\
			
			FLoRA
			& $1e^{-3}$
			& $3e^{-3}$
			& $1e^{-3}$
			& $3e^{-3}$
			& $1e^{-3}$
			& $1e^{-3}$ \\
			
			FRLoRA
			& $5e^{-4}$
			& $1e^{-3}$
			& $5e^{-4}$
			& $1e^{-3}$
			& $5e^{-4}$
			& $5e^{-4}$ \\
			
			Local GaLore
			& $3e^{-3}$
			& $3e^{-3}$
			& $3e^{-3}$
			& $5e^{-3}$
			& $3e^{-3}$
			& $3e^{-3}$ \\
			
			FedFull
			& $5e^{-4}$
			& $5e^{-4}$
			& $5e^{-4}$
			& $1e^{-3}$
			& $5e^{-4}$
			& $5e^{-4}$ \\
			
			FedLore
			& $3e^{-3}$
			& $3e^{-3}$
			& $3e^{-3}$
			& $5e^{-3}$
			& $3e^{-3}$
			& $3e^{-3}$ \\
			
			\bottomrule
		\end{tabular}
	}
\end{table}

\clearpage
\section{Communication and Memory Overhead}
\label{app:fedlore_resources}

\subsection{Communication Cost}
\label{app:fedlore_outline_14}

We analyze the communication and memory overhead of FedLore compared with
existing federated optimization methods. Consider a weight matrix
$W\in\mathbb{R}^{m\times n}$ and a low-rank projection dimension
$r\ll\min(m,n)$.
For full-parameter federated training (FedFull), each client needs to upload and
download the complete model update, resulting in a communication cost of
$
\mathcal{O}(2mn).
$
For LoRA-based federated fine-tuning methods, only the low-rank adapter
parameters are communicated. The communication cost is
$
\mathcal{O}(2r(m+n)).
$

The communication comparison is summarized in
Table~\ref{tab:comm_complexity}.

\begin{table}[!htbp]
	\centering
	\caption{
		Communication and optimizer memory complexity comparison.
	}
	\label{tab:comm_complexity}
	\begin{tabular}{lcc}
		\toprule
		Method & Communication Cost & Optimizer Memory\\
		\midrule
		FedIT & $\mathcal{O}(2r(m+n))$
		& $\mathcal{O}(r(m+n))$\\
		
		FedFull & $\mathcal{O}(2mn)$
		& $\mathcal{O}(mn)$\\
		
		Local GaLore
		& $\mathcal{O}(2mn)$
		& $\mathcal{O}(rn)$\\
		
		FedLore
		& $\mathcal{O}(2r(m+n))$
		& $\mathcal{O}(rn)$\\
		
		\bottomrule
	\end{tabular}
\end{table}

\subsection{Memory Estimates}
\label{app:fedlore_outline_15}

Since the GPU memory consumption of individual components is difficult to
measure precisely, we estimate the memory usage of model weights and optimizer
states based on the number of trainable parameters.
All memory estimations are calculated using BF16 precision, where each
parameter requires 2 bytes. For AdamW optimization, two additional first- and
second-order moment states are maintained for trainable parameters.
For FedLore, the model weights are stored in the original full-parameter space,
while optimizer states are maintained only in the shared low-rank projection
space. Therefore, FedLore preserves the expressive capability of full-parameter
training while requiring significantly less optimizer memory.
 Table~\ref{tab:optimizer_memory} report the
estimated memory usage for different LLaMA model sizes.

\begin{comment}

\begin{table}[!htbp]
	\centering
	\caption{
		Memory estimation of weight parameters under BF16 precision.
	}
	\label{tab:weight_memory}
	\begin{tabular}{lcccc}
		\toprule
		Method & 60M & 130M & 350M & 1B\\
		\midrule
		
		Full-Rank
		&0.12G&0.25G&0.68G&2.60G\\
		
		GaLore
		&0.12G&0.25G&0.68G&2.60G\\
		
		Low-Rank
		&0.08G&0.18G&0.36G&1.19G\\
		
		LoRA
		&0.20G&0.44G&1.04G&3.79G\\
		
		FedLore
		&0.12G&0.25G&0.68G&2.60G\\
		
		\bottomrule
	\end{tabular}
\end{table}
	内容...
\end{comment}
\begin{table}[!htbp]
	\centering
	\caption{
		Memory estimation of optimizer states under BF16 precision.
	}
	\label{tab:optimizer_memory}
	\begin{tabular}{lcccc}
		\toprule
		Method & 60M & 130M & 350M & 1B\\
		\midrule
		
		Full-Rank
		&0.23G&0.51G&1.37G&5.20G\\
		
		GaLore
		&0.13G&0.28G&0.54G&1.78G\\
		
		Low-Rank
		&0.17G&0.37G&0.72G&2.38G\\
		
		LoRA
		&0.17G&0.37G&0.72G&2.38G\\
		
		FedLore
		&0.13G&0.28G&0.54G&1.78G\\
		
		\bottomrule
	\end{tabular}
\end{table}

Overall, FedLore achieves the best trade-off between memory efficiency and
optimization capability. Unlike LoRA-based methods, FedLore updates the full
model parameters through dynamically evolving low-rank gradient subspaces.
Meanwhile, compared with full-parameter training, FedLore significantly reduces
optimizer memory by maintaining adaptive states only in the projected space.

\providecommand{\flspR}{\mathbb{R}}
\providecommand{\flspE}{\mathbb{E}}
\providecommand{\flspavg}[1]{\frac{1}{N}\sum_{i=1}^{N}#1}
\providecommand{\flspnf}[1]{\left\lVert #1\right\rVert_F}
\providecommand{\flspno}[1]{\left\lVert #1\right\rVert_2}
\providecommand{\flspip}[2]{\left\langle #1,#2\right\rangle_F}
\providecommand{\flspdiag}{\operatorname{diag}}
\providecommand{\flsptr}{\operatorname{tr}}

\clearpage
\section{Shared Projections in Federated Optimization}
\label{app:flsp_analysis}

We analyze projected SGD to isolate the effect of shared optimization
subspaces. We first characterize the exact coupling between client-specific
projectors and gradient heterogeneity. A unified convergence bound separates
this coupling from stochastic noise, client sampling, and local drift.
For shared projectors, a sharper analysis controls projected stationarity
without a coverage assumption and exposes the uncovered gradient as an
explicit residual. We then bound that residual for FedLore's candidate-projector
average and identify sufficient conditions for full-gradient convergence.
Exact coordinate aggregation is also established for adaptive updates;
the stationarity results themselves do not analyze coordinatewise Adam.

\subsection{Setting and assumptions}
\label{sec:flsp_setup}
Let
\[
f(W)=\flspavg{f_i(W)},\qquad W\in\flspR^{m\times n},\qquad N\ge2.
\]
We use the Frobenius inner product $\flspip XY=\operatorname{tr}(X^\top Y)$,
the Frobenius norm $\flspnf X$, and the operator norm $\flspno X$.
For each round $t$, let $\mathcal F_t$ contain the complete history and
the round-start model $W^t$ and projectors
\[
\Pi_i^t=P_i^t(P_i^t)^\top,\qquad
(P_i^t)^\top P_i^t=I_r,\qquad P_i^t\in\flspR^{m\times r}.
\]
The rank satisfies $1\le r\le m$. Here $P_i^t$ is a basis and $\Pi_i^t$
is an orthogonal projection operator.
Projectors may depend on past iterates, past data, or a separate construction
batch already included in $\mathcal F_t$. They remain fixed during the round.
The shared case is $\Pi_i^t=\Pi^t$ for all $i$.

\paragraph{Sampling and the stochastic oracle.}
After $\mathcal F_t$ is fixed, a subset $\mathcal S_t$ of size $S$ is sampled
uniformly without replacement. Conditional on $\mathcal F_t$, sampling is
independent of the fresh within-round oracle streams. Each virtual local
trajectory is a function only of $\mathcal F_t$ and its client oracle stream,
not of $\mathcal S_t$; only selected trajectories are executed. This
conditional independence justifies averaging virtual trajectories over the
uniformly sampled subset. Basis construction uses earlier draws or a
separate construction batch; the local optimization draws below are fresh.
We use positive stochastic gradients throughout. Unless a result explicitly
allows round-dependent stepsizes, let $\gamma=\eta\alpha>0$ be constant.
There is no weight decay in the stationarity analysis. The updates are
\begin{align}
	W_i^{t,0}&=W^t,\nonumber\\
	W_i^{t,k+1}&=W_i^{t,k}-\gamma\Pi_i^t g_i^{t,k},
	\quad k=0,\ldots,K-1,\label{eq:flsp_local}\\
	W^{t+1}&=W^t-\frac{\gamma}{S}
	\sum_{i\in\mathcal S_t}\sum_{k=0}^{K-1}\Pi_i^t g_i^{t,k}.
	\label{eq:flsp_server}
\end{align}
Let $\mathcal H_{t,k}$ include $\mathcal F_t$, the sampled subset, and all
virtual oracle draws at steps strictly less than $k$. In particular,
$W_i^{t,k}$ is $\mathcal H_{t,k}$-measurable. The current noises
$\varepsilon_i^{t,k}=g_i^{t,k}-\nabla f_i(W_i^{t,k})$ satisfy
\begin{equation}
	\flspE[\varepsilon_i^{t,k}\mid\mathcal H_{t,k}]=0,
	\qquad
	\flspE[\flspnf{\varepsilon_i^{t,k}}^2\mid\mathcal H_{t,k}]\le\sigma^2.
	\label{eq:flsp_oracle}
\end{equation}
For distinct clients $i\ne j$, assume
\[
\flspE[\flspip{\Pi_i^t\varepsilon_i^{t,k}}{\Pi_j^t\varepsilon_j^{t,k}}
\mid\mathcal H_{t,k}]=0.
\]
Independent client oracle draws conditional on $\mathcal H_{t,k}$ suffice.
Across steps, the martingale-difference property in
\eqref{eq:flsp_oracle} supplies the required orthogonality.
Write $\flspE_t[\cdot]=\flspE[\cdot\mid\mathcal F_t]$.
All objectives are real-valued. The finite second moments needed below
follow by induction from deterministic $W^0$, smoothness, bounded oracle
second moments, finite deterministic stepsizes, and contraction of the
projectors. Smoothness then also ensures integrability of the objectives.

\begin{assumption}[Smoothness and a finite lower bound]
	\label{ass:flsp_smooth}
	Each $f_i$ is differentiable and $L$-smooth, with $L>0$:
	\[
	\flspnf{\nabla f_i(X)-\nabla f_i(Y)}\le L\flspnf{X-Y}.
	\]
	The global objective satisfies $f^\star=\inf_W f(W)>-\infty$.
	Let $W^0$ be fixed and $\Delta_0=f(W^0)-f^\star$.
\end{assumption}

\begin{assumption}[Bounded gradient heterogeneity]
	\label{ass:flsp_hetero}
	For all $W$,
	\[
	\flspavg{\flspnf{\nabla f_i(W)-\nabla f(W)}^2}\le\zeta^2.
	\]
	It suffices that this bound hold almost surely at the round-start iterates.
\end{assumption}

At a round-start point define
\begin{align}
	G_t&=\nabla f(W^t),&
	d_i^t&=\nabla f_i(W^t)-G_t,&
	A_t&=\flspavg{\Pi_i^t},\nonumber\\
	b_t&=\flspavg{(\Pi_i^t-A_t)d_i^t},&
	q_t&=\flspavg{\Pi_i^t\nabla f_i(W^t)},&
	v_t&=\flspno{A_t-A_t^2}.
	\label{eq:flsp_definitions}
\end{align}
The scalar $v_t$ measures projection dispersion. It is independent of the
choice of orthonormal bases for the subspaces. The Frobenius dispersion
used in the motivating experiment satisfies
\[
\frac1N\sum_i\flspnf{\Pi_i^t-A_t}^2
=\operatorname{tr}(A_t-A_t^2).
\]
This trace and $v_t=\|A_t-A_t^2\|_2$ are different quantities;
$v_t\le\operatorname{tr}(A_t-A_t^2)\le m v_t$.

\begin{assumption}[Aggregate global-gradient coverage]
	\label{ass:flsp_coverage}
	There exists $\mu\in(0,1]$ such that almost surely, for every round,
	\begin{equation}
		\flspip{G_t}{A_tG_t}\ge\mu\flspnf{G_t}^2.
		\label{eq:flsp_coverage}
	\end{equation}
\end{assumption}

Condition~\eqref{eq:flsp_coverage} concerns only the current global gradient.
It takes the respective forms
\[
\flspnf{\Pi^tG_t}^2\ge\mu\flspnf{G_t}^2,
\qquad
\frac1N\sum_i\flspnf{\Pi_i^tG_t}^2\ge\mu\flspnf{G_t}^2
\]
for shared and client-specific projections.
Neither statement follows from coverage of each client's own gradient.
The same form of coverage is used throughout, and comparisons below use a
common valid value of $\mu$.

The order of conditioning matters. A projector constructed from current
oracle noise need not preserve its zero mean, and a projector constructed
from the current participating subset may be correlated with that subset.
The results below apply when projectors are chosen before sampling and the
fresh oracle draws, as specified above. Reusing earlier rounds is permitted.
Approximate coverage is treated in Proposition~\ref{prop:flsp_additive}.

\subsection{Projection disagreement and aggregation bias}
\label{sec:flsp_bias}

\begin{theorem}[Bias identity, sharp worst-case bound, and cancellation]
	\label{thm:flsp_bias}
	Fix any orthogonal projectors $\Pi_1,\ldots,\Pi_N$ and set
	$A=N^{-1}\sum_i\Pi_i$. For matrices $d_i$ with $\sum_i d_i=0$, define
	$b=N^{-1}\sum_i(\Pi_i-A)d_i$. Then:
	\begin{enumerate}
		\item For any $G$ and $G_i=G+d_i$,
		\begin{equation}
			\flspavg{\Pi_iG_i}=AG+b.
			\label{eq:flsp_bias_identity}
		\end{equation}
		\item The following operator identity and bounds hold:
		\begin{equation}
			\flspavg{(\Pi_i-A)^2}=A-A^2,\qquad
			0\preceq A-A^2\preceq\tfrac14 I.
			\label{eq:flsp_variance_identity}
		\end{equation}
		Moreover, for every $z\ge0$,
		\begin{equation}
			\sup_{\substack{\sum_i d_i=0\\N^{-1}\sum_i\flspnf{d_i}^2\le z^2}}
			\flspnf{\flspavg{(\Pi_i-A)d_i}}^2
			=z^2\flspno{A-A^2}.
			\label{eq:flsp_sharp_bias}
		\end{equation}
		\item Bias vanishes for every zero-mean collection $(d_i)$ if and only if
		all projectors coincide. Equivalently,
		\begin{equation}
			[b=0\text{ for every }(d_i)]
			\Longleftrightarrow A-A^2=0
			\Longleftrightarrow \Pi_1=\cdots=\Pi_N.
			\label{eq:flsp_iff}
		\end{equation}
	\end{enumerate}
	Consequently, under Assumption~\ref{ass:flsp_hetero},
	$\flspnf{b_t}^2\le v_t\zeta^2$, and shared projections give $b_t=0$ exactly.
\end{theorem}

\begin{proof}
	Since $\flspavg{d_i}=0$,
	\[
	\flspavg{\Pi_i(G+d_i)}=AG+\flspavg{\Pi_i d_i}
	=AG+\flspavg{(\Pi_i-A)d_i}.
	\]
	
	All matrices $\Pi_i-A$ are symmetric. Expanding their squares gives
	\begin{align*}
		\flspavg{(\Pi_i-A)^2}
		&=\flspavg{\Pi_i^2}-\left(\flspavg{\Pi_i}\right)A
		-A\left(\flspavg{\Pi_i}\right)+A^2\\
		&=A-A^2.
	\end{align*}
	An equivalent pairwise form is
	\begin{equation}
		A-A^2=\frac1{2N^2}\sum_{i,j=1}^N(\Pi_i-\Pi_j)^2.
		\label{eq:flsp_pairwise}
	\end{equation}
	Indeed, the sum of the two squared terms is $2N^2A$, while the sum of
	the two ordered cross terms is $2N^2A^2$. Taking traces gives
	\[
	\operatorname{tr}(A-A^2)
	=\frac1{2N^2}\sum_{i,j=1}^N\flspnf{\Pi_i-\Pi_j}^2.
	\]
	Thus the dispersion has an intrinsic subspace interpretation, independent
	of basis coordinates.
	Since $0\preceq A\preceq I$, its eigenvalues belong to $[0,1]$.
	The eigenvalues of $A-A^2$ are $a(1-a)\in[0,1/4]$.
	For any test matrix $Z$ with $\flspnf{Z}=1$, symmetry and Cauchy--Schwarz yield
	\begin{align*}
		|\flspip{Z}{b}|
		&=\left|\flspavg{\flspip{(\Pi_i-A)Z}{d_i}}\right|\\
		&\le\left(\flspavg{\flspnf{(\Pi_i-A)Z}^2}\right)^{1/2}
		\left(\flspavg{\flspnf{d_i}^2}\right)^{1/2}\\
		&=\flspip{Z}{(A-A^2)Z}^{1/2}
		\left(\flspavg{\flspnf{d_i}^2}\right)^{1/2}\\
		&\le \sqrt{\flspno{A-A^2}}\,z.
	\end{align*}
	Taking the supremum over $Z$ proves the upper bound in
	\eqref{eq:flsp_sharp_bias}.
	Let $v=\flspno{A-A^2}>0$, let $u\in\flspR^m$ be a unit eigenvector with eigenvalue
	$v$, and let $a\in\flspR^n$ be a unit vector. Set $Z=ua^\top$ and
	\[
	d_i=\frac{z}{\sqrt v}(\Pi_i-A)Z.
	\]
	Then $\flspavg{d_i}=0$ and
	\[
	\flspavg{\flspnf{d_i}^2}
	=\frac{z^2}{v}\flspip{Z}{(A-A^2)Z}=z^2.
	\]
	Furthermore,
	\[
	b=\frac{z}{\sqrt v}(A-A^2)Z=z\sqrt v\,Z,
	\]
	which attains $\flspnf{b}^2=z^2v$. The cases $v=0$ or $z=0$ are immediate.
	If $A-A^2=0$, then for every vector $u$,
	\[
	0=u^\top(A-A^2)u=\flspavg{\| (\Pi_i-A)u\|_2^2}.
	\]
	Every summand is nonnegative, so each $\Pi_i-A$ is zero.
	Conversely, identical projectors give $A^2=A$ and $b=0$.
	If some projectors differ, then $v>0$, and the maximizing construction
	above gives a nonzero bias for every $z>0$.
\end{proof}

\begin{remark}[Worst-case interpretation]
	The supremum in \eqref{eq:flsp_sharp_bias} fixes the projectors and varies
	the zero-mean gradient perturbations. Every admissible perturbation is
	realizable by smooth quadratic objectives, for example
	$f_i(W)=\frac12\flspnf W^2+\flspip{G+d_i}{W}$ at $W=0$.
	The identity does not assert attainability under a prescribed rule that
	recomputes the projectors as the gradients vary. An example with exact
	local SVD selection is given in Theorem~\ref{thm:flsp_separation}.
\end{remark}

The norm bound does not retain the orientation of the bias. The next result
characterizes its effect on the descent direction without this relaxation.

\begin{proposition}[Sharp directional alignment]
	\label{prop:flsp_directional}
	Fix the projectors and $G\in\flspR^{m\times n}$. Let $V=A-A^2$ and
	$s_G=\sqrt{\flspip G{VG}}$. For every $z\ge0$,
	\begin{equation}
		\inf_{\substack{\sum_i d_i=0\\N^{-1}\sum_i\flspnf{d_i}^2\le z^2}}
		\flspip G{\frac1N\sum_i\Pi_i(G+d_i)}
		=\flspip G{AG}-zs_G.
		\label{eq:flsp_directional}
	\end{equation}
	Consequently, for $G\ne0$, the negative aggregate is a strict descent
	direction for every admissible perturbation if and only if
	\begin{equation}
		\flspip G{AG}>z\sqrt{\flspip G{(A-A^2)G}}.
		\label{eq:flsp_robust_descent}
	\end{equation}
	For a shared projector the right-hand side is zero.
\end{proposition}
\begin{proof}
	The bias identity and Cauchy--Schwarz imply
	\begin{align*}
		\flspip G{AG+b}
		&=\flspip G{AG}+\frac1N\sum_i\flspip{(\Pi_i-A)G}{d_i}\\
		&\ge\flspip G{AG}
		-\left(\frac1N\sum_i\flspnf{(\Pi_i-A)G}^2\right)^{1/2}
		\left(\frac1N\sum_i\flspnf{d_i}^2\right)^{1/2}\\
		&\ge\flspip G{AG}-zs_G.
	\end{align*}
	If $s_G>0$, equality holds for
	$d_i=-z(\Pi_i-A)G/s_G$: their mean is zero and their mean squared norm
	is $z^2$. If $s_G=0$, every $(\Pi_i-A)G$ is zero, so the directional
	bias vanishes for all admissible $(d_i)$. This proves the equality.
	The directional derivative of $f$ along $-q$ is $-\flspip Gq$.
	Since the infimum is attained, uniform strict descent is equivalent to
	\eqref{eq:flsp_robust_descent}.
\end{proof}

\subsection{Convergence under partial participation}
\label{sec:flsp_convergence}
Write
\begin{equation}
	h=\gamma K,\qquad x=Lh,\qquad
	\omega_S=\frac{N-S}{S(N-1)},
	\label{eq:flsp_constants}
\end{equation}
and define
\begin{equation}
	a_K=\frac56L^2\gamma^2(K-1)(2K-1),\qquad
	e_K=\frac32L^2\gamma^2(K-1).
	\label{eq:flsp_drift_constants}
\end{equation}
Here $a_K\le5x^2/3$, $e_K\le3x^2/(2K)$, and both vanish when $K=1$.
The two coefficients distinguish drift from deterministic gradients and
from accumulated martingale noise.

\begin{theorem}[Convergence with projection bias]
	\label{thm:flsp_convergence}
	Suppose the model in Section~\ref{sec:flsp_setup} and
	Assumptions~\ref{ass:flsp_smooth}--\ref{ass:flsp_coverage} hold, and
	\begin{equation}
		0<x=L\gamma K\le\frac\mu{16}.
		\label{eq:flsp_stepsize}
	\end{equation}
	For every $T\ge1$,
	\begin{align}
		\frac1T\sum_{t=0}^{T-1}\flspE\flspnf{G_t}^2
		&\le\frac{4\Delta_0}{\mu hT}
		+\frac8{\mu^2T}\sum_{t=0}^{T-1}\flspE\flspnf{b_t}^2\nonumber\\
		&\quad+\frac8{\mu^2}(a_K\zeta^2+e_K\sigma^2)
		+\frac{4x}{\mu}\left(\omega_S\zeta^2+\frac{\sigma^2}{SK}\right).
		\label{eq:flsp_main_bound}
	\end{align}
	The projection-bias term is at most
	\begin{equation}
		\frac{8\zeta^2}{\mu^2}\bar v_T,
		\qquad \bar v_T=\frac1T\sum_{t=0}^{T-1}\flspE\flspno{A_t-A_t^2},
		\label{eq:flsp_dispersion_penalty}
	\end{equation}
	and is zero for shared projections.
\end{theorem}

The term $\omega_S\zeta^2$ accounts for client sampling and vanishes under
full participation. The variance $\sigma^2/(SK)$ is the accumulated oracle
noise in the server direction. Its indirect contribution through local drift
is controlled by $e_K\sigma^2=O(x^2\sigma^2/K)$, whereas deterministic
heterogeneity contributes $a_K\zeta^2=O(x^2\zeta^2)$. No independence
between the local drift and the accumulated noise is used.

\subsubsection{Auxiliary estimates}

\begin{lemma}[Finite-population sampling]
	\label{lem:flsp_sampling}
	For any $\mathcal F_t$-measurable matrices $Q_1,\ldots,Q_N$, let
	$q=N^{-1}\sum_i Q_i$ and $q_{\mathcal S}=S^{-1}\sum_{i\in\mathcal S_t}Q_i$.
	Then
	\begin{equation}
		\flspE_t q_{\mathcal S}=q,\qquad
		\flspE_t\flspnf{q_{\mathcal S}-q}^2
		=\omega_S\flspavg{\flspnf{Q_i-q}^2}.
		\label{eq:flsp_sampling_variance}
	\end{equation}
	For $Q_i=\Pi_i^t\nabla f_i(W^t)$, this implies
	\begin{equation}
		\flspE_t\flspnf{q_{\mathcal S}}^2
		\le3\flspnf{G_t}^2+2\flspnf{b_t}^2+\omega_S\zeta^2.
		\label{eq:flsp_qsample}
	\end{equation}
\end{lemma}
\begin{proof}
	Let $I_i=\mathbf1\{i\in\mathcal S_t\}$ and $X_i=Q_i-q$.
	Uniform sampling gives $\flspE_t I_i=S/N$ and, for $i\ne j$,
	$\flspE_t I_iI_j=S(S-1)/(N(N-1))$. Since $\sum_i X_i=0$,
	$\sum_{i\ne j}\flspip{X_i}{X_j}=-\sum_i\flspnf{X_i}^2$.
	Expanding the square of $S^{-1}\sum_i I_iX_i$ gives
	\[
	\frac1{S^2}\left(\frac SN-
	\frac{S(S-1)}{N(N-1)}\right)\sum_i\flspnf{X_i}^2
	=\omega_S\flspavg{\flspnf{X_i}^2}.
	\]
	For the stated application, $q=A_tG_t+b_t$, and hence
	$\flspnf q^2\le2\flspnf{G_t}^2+2\flspnf{b_t}^2$ because $\flspno{A_t}\le1$.
	Also,
	\begin{align*}
		\flspavg{\flspnf{Q_i-q}^2}
		&\le\flspavg{\flspnf{Q_i}^2}
		\le\flspavg{\flspnf{G_t+d_i^t}^2}\\
		&=\flspnf{G_t}^2+\flspavg{\flspnf{d_i^t}^2}
		\le\flspnf{G_t}^2+\zeta^2.
	\end{align*}
	Combining these bounds and $\omega_S\le1$ gives
	\eqref{eq:flsp_qsample}.
\end{proof}

Define the average direction, accumulated noise, and drift error by
\begin{align}
	U_t&=\frac1{SK}\sum_{i\in\mathcal S_t}\sum_{k=0}^{K-1}
	\Pi_i^t g_i^{t,k},\nonumber\\
	B_t&=\frac1{SK}\sum_{i\in\mathcal S_t}\sum_{k=0}^{K-1}
	\Pi_i^t\varepsilon_i^{t,k},\nonumber\\
	C_t&=\frac1{SK}\sum_{i\in\mathcal S_t}\sum_{k=0}^{K-1}
	\Pi_i^t[\nabla f_i(W_i^{t,k})-\nabla f_i(W^t)].
	\label{eq:flsp_errors}
\end{align}
Then the exact identities are
\begin{equation}
	U_t=q_{\mathcal S_t}+B_t+C_t,
	\qquad W^{t+1}=W^t-hU_t.
	\label{eq:flsp_direction_decomposition}
\end{equation}
Neither $C_t$ nor $B_t+C_t$ is assumed to be independent of the local
trajectory. In particular, $C_t$ need not have zero mean.

\begin{lemma}[Accumulated oracle noise]
	\label{lem:flsp_noise}
	Under \eqref{eq:flsp_oracle} and the sampling convention,
	\begin{equation}
		\flspE[B_t\mid\mathcal F_t,\mathcal S_t]=0,\qquad\flspE_t\flspnf{B_t}^2\le\frac{\sigma^2}{SK}.
		\label{eq:flsp_noise}
	\end{equation}
\end{lemma}
\begin{proof}
	Condition first on $\mathcal F_t$ and $\mathcal S_t$.
	At equal local steps, cross-client noise products have zero conditional
	expectation by the oracle assumption. At unequal steps $k<\ell$,
	$\Pi_i^t\varepsilon_i^{t,k}$ is measurable before step $\ell$, whereas
	$\Pi_j^t\varepsilon_j^{t,\ell}$ has zero conditional expectation.
	The tower property therefore eliminates every cross term in the squared
	norm of the sum. Since projectors contract norms, the $SK$ diagonal
	terms are each at most $\sigma^2$ in expectation. Division by $S^2K^2$
	proves the variance bound. The mean-zero assertion follows in the same way.
\end{proof}

\begin{lemma}[Local drift with separate noise dependence]
	\label{lem:flsp_drift}
	Suppose $x\le1/8$ and define
	\[
	D_{t,k}=\frac1N\sum_i\flspE_t\flspnf{W_i^{t,k}-W^t}^2,\qquad
	H_t=\flspnf{G_t}^2+\zeta^2.
	\]
	Then, for $0\le k\le K$,
	\begin{align}
		D_{t,k}&\le5\gamma^2k^2H_t+3\gamma^2k\sigma^2,
		\label{eq:flsp_drift}\\
		\flspE_t\flspnf{C_t}^2&\le a_KH_t+e_K\sigma^2.
		\label{eq:flsp_Cbound}
	\end{align}
\end{lemma}
\begin{proof}
	Separate the gradient and noise sums in the unrolled update:
	\[
	W_i^{t,k}-W^t
	=-\gamma\sum_{j<k}\Pi_i^t\nabla f_i(W_i^{t,j})
	-\gamma\sum_{j<k}\Pi_i^t\varepsilon_i^{t,j}.
	\]
	Although these two sums are generally correlated, the inequality
	$\flspnf{X+Y}^2\le2\flspnf X^2+2\flspnf Y^2$ applies without independence.
	Within the second sum, the martingale-difference property eliminates all
	cross terms, so
	\[
	\flspE_t\flspnf{\sum_{j<k}\Pi_i^t\varepsilon_i^{t,j}}^2\le k\sigma^2.
	\]
	Cauchy--Schwarz on the first sum therefore gives
	\begin{equation}
		D_{t,k}\le2\gamma^2k\sum_{j<k}
		\frac1N\sum_i\flspE_t\flspnf{\nabla f_i(W_i^{t,j})}^2
		+2\gamma^2k\sigma^2.
		\label{eq:flsp_drift_recursion1}
	\end{equation}
	By smoothness and the zero-mean identity for $d_i^t$,
	\begin{align*}
		\frac1N\sum_i\flspE_t\flspnf{\nabla f_i(W_i^{t,j})}^2
		&\le2L^2D_{t,j}+2\frac1N\sum_i\flspnf{G_t+d_i^t}^2\\
		&=2L^2D_{t,j}+2\flspnf{G_t}^2+2\frac1N\sum_i\flspnf{d_i^t}^2\\
		&\le2L^2D_{t,j}+2H_t.
	\end{align*}
	Substitution into \eqref{eq:flsp_drift_recursion1} yields
	\begin{equation}
		D_{t,k}\le4\gamma^2L^2k\sum_{j<k}D_{t,j}
		+4\gamma^2k^2H_t+2\gamma^2k\sigma^2.
		\label{eq:flsp_drift_recursion2}
	\end{equation}
	Fix $k$ and put $M_{t,k}=\max_{0\le j\le k}D_{t,j}$.
	For each $\ell\le k$, apply \eqref{eq:flsp_drift_recursion2} at $\ell$
	and bound $\sum_{j<\ell}D_{t,j}\le\ell M_{t,k}$. Taking the maximum gives
	\[
	M_{t,k}\le4\gamma^2L^2k^2M_{t,k}
	+4\gamma^2k^2H_t+2\gamma^2k\sigma^2.
	\]
	Since $4\gamma^2L^2k^2\le4x^2\le1/16$,
	\[
	D_{t,k}\le M_{t,k}
	\le\frac{4\gamma^2k^2H_t+2\gamma^2k\sigma^2}
	{1-4\gamma^2L^2k^2}
	\le5\gamma^2k^2H_t+3\gamma^2k\sigma^2.
	\]
	For $C_t$, Jensen's inequality, contraction, and smoothness imply
	\begin{align*}
		\flspE_t\flspnf{C_t}^2
		&\le\frac{L^2}{K}\sum_{k=0}^{K-1}
		\flspE_t\left[\frac1S\sum_{i\in\mathcal S_t}
		\flspnf{W_i^{t,k}-W^t}^2\right]\\
		&=\frac{L^2}{K}\sum_{k=0}^{K-1}D_{t,k}.
	\end{align*}
	For the equality, conditional independence of the subset and virtual
	trajectories gives
	\[
	\flspE_t\left[\mathbf1\{i\in\mathcal S_t\}
	\flspnf{W_i^{t,k}-W^t}^2\right]
	=\frac SN\flspE_t\flspnf{W_i^{t,k}-W^t}^2.
	\]
	Using $\sum_{k<K}k^2=K(K-1)(2K-1)/6$ and
	$\sum_{k<K}k=K(K-1)/2$ proves \eqref{eq:flsp_Cbound}.
\end{proof}

\subsubsection{Proof of the convergence theorem}
\begin{proof}[Proof of Theorem~\ref{thm:flsp_convergence}]
	All expectations in the first part are conditional on $\mathcal F_t$.
	The update $W^{t+1}=W^t-hU_t$ and smoothness give
	\begin{equation}
		\flspE_t f(W^{t+1})\le f(W^t)-h\flspip{G_t}{\flspE_tU_t}
		+\frac{Lh^2}{2}\flspE_t\flspnf{U_t}^2.
		\label{eq:flsp_smoothness}
	\end{equation}
	By the sampling and noise lemmas,
	$\flspE_tU_t=A_tG_t+b_t+\flspE_tC_t$.
	Coverage and two applications of
	$ab\le\mu a^2/4+b^2/\mu$ yield
	\begin{align}
		-\flspip{G_t}{\flspE_tU_t}
		&\le-\mu\flspnf{G_t}^2+\flspnf{G_t}\flspnf{b_t}
		+\flspnf{G_t}\flspnf{\flspE_tC_t}\nonumber\\
		&\le-\frac\mu2\flspnf{G_t}^2
		+\frac1\mu\flspnf{b_t}^2+\frac1\mu\flspE_t\flspnf{C_t}^2.
		\label{eq:flsp_first_order}
	\end{align}
	The noise makes no contribution to this first-order term.
	
	For the second moment, $q_{\mathcal S_t}$ is measurable with respect to
	$\sigma(\mathcal F_t,\mathcal S_t)$ and
	$\flspE[B_t\mid\mathcal F_t,\mathcal S_t]=0$. Hence
	\[
	\flspE_t\flspip{q_{\mathcal S_t}}{B_t}=0,\qquad
	\flspE_t\flspnf{q_{\mathcal S_t}+B_t}^2
	=\flspE_t\flspnf{q_{\mathcal S_t}}^2+\flspE_t\flspnf{B_t}^2.
	\]
	Consequently, the two-term norm inequality gives
	\begin{align}
		\flspE_t\flspnf{U_t}^2
		&\le2\flspE_t\flspnf{q_{\mathcal S_t}+B_t}^2+2\flspE_t\flspnf{C_t}^2\nonumber\\
		&\le6\flspnf{G_t}^2+4\flspnf{b_t}^2
		+2\omega_S\zeta^2+\frac{2\sigma^2}{SK}
		+2\flspE_t\flspnf{C_t}^2.
		\label{eq:flsp_second_order}
	\end{align}
	Substituting \eqref{eq:flsp_first_order} and
	\eqref{eq:flsp_second_order} into \eqref{eq:flsp_smoothness}, with $x=Lh$,
	gives
	\begin{align}
		\flspE_t f(W^{t+1})
		&\le f(W^t)-h\left(\frac\mu2-3x\right)\flspnf{G_t}^2\nonumber\\
		&\quad+h\left(\frac1\mu+2x\right)\flspnf{b_t}^2
		+h\left(\frac1\mu+x\right)\flspE_t\flspnf{C_t}^2\nonumber\\
		&\quad+hx\left(\omega_S\zeta^2+\frac{\sigma^2}{SK}\right).
		\label{eq:flsp_preabsorption}
	\end{align}
	Since $x\le\mu/16\le1/16$, Lemma~\ref{lem:flsp_drift} applies. Also,
	\[
	\frac1\mu+2x\le\frac2\mu,\qquad
	\frac1\mu+x\le\frac2\mu,
	\]
	and the coefficient to be absorbed satisfies
	\[
	3x+\left(\frac1\mu+x\right)a_K
	\le3x+\frac{10x^2}{3\mu}
	\le\left(\frac3{16}+\frac5{384}\right)\mu
	=\frac{77}{384}\mu<\frac\mu4.
	\]
	Thus \eqref{eq:flsp_preabsorption} implies
	\begin{align}
		\flspE_t f(W^{t+1})
		&\le f(W^t)-\frac{\mu h}{4}\flspnf{G_t}^2
		+\frac{2h}{\mu}\flspnf{b_t}^2\nonumber\\
		&\quad+\frac{2h}{\mu}(a_K\zeta^2+e_K\sigma^2)
		+hx\left(\omega_S\zeta^2+\frac{\sigma^2}{SK}\right).
		\label{eq:flsp_one_round}
	\end{align}
	Taking total expectations and summing over $t<T$ telescopes the objective:
	\begin{align*}
		\frac{\mu h}{4}\sum_{t<T}\flspE\flspnf{G_t}^2
		&\le f(W^0)-\flspE f(W^T)
		+\frac{2h}{\mu}\sum_{t<T}\flspE\flspnf{b_t}^2\\
		&\quad+\frac{2hT}{\mu}(a_K\zeta^2+e_K\sigma^2)
		+hxT\left(\omega_S\zeta^2+\frac{\sigma^2}{SK}\right).
	\end{align*}
	Use $f(W^T)\ge f^\star$ and divide by $\mu hT/4$.
	This proves \eqref{eq:flsp_main_bound}; Theorem~\ref{thm:flsp_bias} supplies
	\eqref{eq:flsp_dispersion_penalty}.
\end{proof}

\subsubsection{A sharper bound for shared projections}
\label{sec:flsp_shared_sharp}

For shared projections, the entire server direction belongs to the common
subspace. Retaining this fact gives a stronger result than setting $b_t=0$
in Theorem~\ref{thm:flsp_convergence}. In particular, projected stationarity
does not require a lower bound on global-gradient coverage.

\begin{theorem}[Shared-projection stationarity with an explicit residual]
	\label{thm:flsp_shared_residual}
	Assume the sampling and oracle conditions in Section~\ref{sec:flsp_setup}
	and Assumptions~\ref{ass:flsp_smooth} and~\ref{ass:flsp_hetero}.
	All clients use the same $\mathcal F_t$-measurable orthogonal projector
	$\Pi^t$ throughout round $t$. Allow deterministic round-dependent stepsizes
	$\gamma_t>0$, constant within each round, and define
	\begin{align}
		h_t&=K\gamma_t,\quad x_t=Lh_t\le\tfrac18,\quad
		\Lambda_T=\sum_{t<T}h_t,\nonumber\\
		a_{K,t}&=\tfrac56L^2\gamma_t^2(K-1)(2K-1),\quad
		e_{K,t}=\tfrac32L^2\gamma_t^2(K-1),\nonumber\\
		\rho_t^2&=\flspnf{(I-\Pi^t)G_t}^2,\quad
		V_S=\omega_S\zeta^2+\frac{\sigma^2}{SK}.
		\label{eq:flsp_shared_constants}
	\end{align}
	Set
	\begin{equation}
		\mathcal B_T=
		\frac{2\Delta_0}{\Lambda_T}
		+\frac4{\Lambda_T}\sum_{t<T}h_t(a_{K,t}\zeta^2+e_{K,t}\sigma^2)
		+\frac{2V_S}{\Lambda_T}\sum_{t<T}h_tx_t.
		\label{eq:flsp_shared_remainder}
	\end{equation}
	Without Assumption~\ref{ass:flsp_coverage},
	\begin{align}
		\frac1{\Lambda_T}\sum_{t<T}h_t\flspE\flspnf{\Pi^tG_t}^2
		&\le\mathcal B_T,\label{eq:flsp_projected_stationarity}\\
		\frac1{\Lambda_T}\sum_{t<T}h_t\flspE\flspnf{G_t}^2
		&\le\mathcal B_T+
		\frac1{\Lambda_T}\sum_{t<T}h_t\flspE\rho_t^2.
		\label{eq:flsp_residual_stationarity}
	\end{align}
	If, additionally, for a constant $\mu\in(0,1]$ and every $t$,
	\begin{equation}
		\flspE\flspnf{\Pi^tG_t}^2\ge\mu\flspE\flspnf{G_t}^2,
		\label{eq:flsp_expected_coverage}
	\end{equation}
	then the left-hand side of~\eqref{eq:flsp_residual_stationarity} is at most
	$\mathcal B_T/\mu$. Thus expected coverage suffices; pathwise coverage
	in Assumption~\ref{ass:flsp_coverage} is a sufficient, stronger condition.
\end{theorem}

\begin{proof}
	Fix $t$, condition on $\mathcal F_t$, and write $\Pi=\Pi^t$,
	$Q_t=\Pi G_t$, $\gamma=\gamma_t$, $h=h_t$, and $x=x_t$.
	The round-start sampled direction has conditional mean $Q_t$.
	Since $\sum_i d_i^t=0$ and $\Pi$ is shared, the finite-population identity gives
	\begin{align}
		\flspE_t\flspnf{q_{\mathcal S_t}}^2
		&=\flspnf{Q_t}^2+
		\omega_S\frac1N\sum_i\flspnf{\Pi d_i^t}^2\nonumber\\
		&\le\flspnf{Q_t}^2+\omega_S\zeta^2.
		\label{eq:flsp_shared_sampling}
	\end{align}
	The projected version of the local-drift estimate is
	\begin{equation}
		\flspE_t\flspnf{C_t}^2
		\le a_{K,t}(\flspnf{Q_t}^2+\zeta^2)+e_{K,t}\sigma^2.
		\label{eq:flsp_shared_drift}
	\end{equation}
	For completeness, let $D_{t,k}=N^{-1}\sum_i
	\flspE_t\flspnf{W_i^{t,k}-W^t}^2$.
	The unrolled recursion and martingale orthogonality yield
	\[
	D_{t,k}\le2\gamma^2k\sum_{j<k}\frac1N\sum_i
	\flspE_t\flspnf{\Pi\nabla f_i(W_i^{t,j})}^2+2\gamma^2k\sigma^2.
	\]
	Smoothness and the identity
	\[
	\frac1N\sum_i\flspnf{\Pi\nabla f_i(W^t)}^2
	=\flspnf{Q_t}^2+\frac1N\sum_i\flspnf{\Pi d_i^t}^2
	\le\flspnf{Q_t}^2+\zeta^2
	\]
	give
	\[
	D_{t,k}\le4\gamma^2L^2k\sum_{j<k}D_{t,j}
	+4\gamma^2k^2(\flspnf{Q_t}^2+\zeta^2)+2\gamma^2k\sigma^2.
	\]
	Because $L\gamma K\le1/8$, the maximum argument in
	Lemma~\ref{lem:flsp_drift} proves
	$D_{t,k}\le5\gamma^2k^2(\flspnf{Q_t}^2+\zeta^2)
	+3\gamma^2k\sigma^2$.
	Averaging over the uniformly sampled clients and local steps, as in that
	lemma, gives~\eqref{eq:flsp_shared_drift} with the stated $a_{K,t},e_{K,t}$.
	
	Now $\Pi U_t=U_t$ and $\Pi C_t=C_t$ pathwise, and
	$\flspE_tU_t=Q_t+\flspE_tC_t$. Hence
	\begin{align}
		-\flspip{G_t}{\flspE_tU_t}
		&=-\flspnf{Q_t}^2-\flspip{Q_t}{\flspE_tC_t}\nonumber\\
		&\le-\tfrac34\flspnf{Q_t}^2+\flspE_t\flspnf{C_t}^2.
		\label{eq:flsp_shared_first_order}
	\end{align}
	Here we used $ab\le a^2/4+b^2$ and conditional Jensen.
	Lemma~\ref{lem:flsp_noise}, the orthogonality of $q_{\mathcal S_t}$ and
	$B_t$ in conditional expectation, and~\eqref{eq:flsp_shared_sampling} imply
	\[
	\flspE_t\flspnf{U_t}^2
	\le2\flspnf{Q_t}^2+2V_S+2\flspE_t\flspnf{C_t}^2.
	\]
	Applying smoothness to $W^{t+1}=W^t-hU_t$ therefore gives
	\begin{align*}
		\flspE_tf(W^{t+1})
		&\le f(W^t)-h(3/4-x)\flspnf{Q_t}^2
		+h(1+x)\flspE_t\flspnf{C_t}^2+hxV_S\\
		&\le f(W^t)-h[3/4-x-(1+x)a_{K,t}]\flspnf{Q_t}^2\\
		&\quad+h(1+x)(a_{K,t}\zeta^2+e_{K,t}\sigma^2)+hxV_S.
	\end{align*}
	Since $a_{K,t}\le5x^2/3$ and $x\le1/8$,
	\[
	x+(1+x)a_{K,t}\le\tfrac18+\tfrac98\cdot\tfrac5{192}
	=\tfrac{79}{512}<\tfrac14,
	\qquad 1+x\le2.
	\]
	Consequently,
	\begin{equation}
		\flspE_tf(W^{t+1})\le f(W^t)
		-\frac{h_t}{2}\flspnf{\Pi^tG_t}^2
		+2h_t(a_{K,t}\zeta^2+e_{K,t}\sigma^2)+h_tx_tV_S.
		\label{eq:flsp_shared_one_round}
	\end{equation}
	Take total expectations, sum over $t<T$, use $f(W^T)\ge f^\star$, and
	divide by $\Lambda_T/2$ to prove~\eqref{eq:flsp_projected_stationarity}.
	The orthogonal decomposition
	$\flspnf{G_t}^2=\flspnf{\Pi^tG_t}^2+\rho_t^2$
	proves~\eqref{eq:flsp_residual_stationarity}.
	Finally, summing~\eqref{eq:flsp_expected_coverage} with the deterministic
	weights $h_t$ gives the last assertion.
\end{proof}

\paragraph{Meaning of the residual.}
The first bound controls stationarity within the currently used subspace.
The second exposes the missing component of the full gradient exactly.
Neither an arbitrary shared projector nor subspace refresh alone forces
$\rho_t$ to vanish. The theorem applies to the actual predictable shared
selection rule even when a uniform positive coverage constant is unavailable;
in that case it is a residual bound, not an unconditional full-stationarity
guarantee. For nonconstant schedules, the conclusion is weighted by $h_t$;
it does not assert the same unweighted rate for every cosine schedule.

\begin{corollary}[Shared projections under additive expected coverage]
	\label{cor:flsp_shared_additive}
	Under Theorem~\ref{thm:flsp_shared_residual}, suppose
	\[
	\flspE\flspnf{\Pi^tG_t}^2\ge\mu\flspE\flspnf{G_t}^2-\chi_t
	\]
	for deterministic $\chi_t\ge0$ and $\mu\in(0,1]$. Then
	\[
	\frac1{\Lambda_T}\sum_{t<T}h_t\flspE\flspnf{G_t}^2
	\le\frac{\mathcal B_T}{\mu}
	+\frac1{\mu\Lambda_T}\sum_{t<T}h_t\chi_t.
	\]
\end{corollary}
\begin{proof}
	Sum the assumed inequality with weights $h_t$ and apply
	\eqref{eq:flsp_projected_stationarity}.
\end{proof}

\subsubsection{Consequences for shared and client-specific projections}
Define the common remainder
\begin{equation}
	\mathcal R_T=\frac{4\Delta_0}{\mu hT}
	+\frac8{\mu^2}(a_K\zeta^2+e_K\sigma^2)
	+\frac{4x}{\mu}\left(\omega_S\zeta^2+\frac{\sigma^2}{SK}\right).
	\label{eq:flsp_remainder}
\end{equation}

\begin{corollary}[Sharper shared-projection convergence]
	\label{cor:flsp_shared}
	Under Theorem~\ref{thm:flsp_shared_residual} and expected coverage
	\eqref{eq:flsp_expected_coverage}, fix a horizon $T$ and set
	$\gamma_t=c/(LK\sqrt T)$ for all $t<T$, where $0<c\le1/8$.
	Then
	\begin{align}
		\frac1T\sum_{t<T}\flspE\flspnf{G_t}^2
		&\le\frac{2L\Delta_0}{\mu c\sqrt T}
		+\frac{20c^2\zeta^2}{3\mu T}
		+\frac{6c^2\sigma^2}{\mu KT}\nonumber\\
		&\quad+\frac{2c}{\mu\sqrt T}
		\left(\omega_S\zeta^2+\frac{\sigma^2}{SK}\right).
		\label{eq:flsp_shared_rate}
	\end{align}
	For fixed problem parameters, $K,S,c,\mu$, this is $O(T^{-1/2})$.
	For $K=1$, both drift terms vanish in the unrelaxed bound
	$\mathcal B_T/\mu$; for $S=N$, the client-sampling term vanishes.
	Without coverage, the same right-hand side with $\mu=1$ bounds the
	average projected-gradient norm, and adding
	$T^{-1}\sum_{t<T}\flspE\rho_t^2$ bounds the average full-gradient norm.
\end{corollary}
\begin{proof}
	Substitute $h_t=c/(L\sqrt T)$, $x_t=c/\sqrt T$,
	$a_{K,t}\le5c^2/(3T)$, and $e_{K,t}\le3c^2/(2KT)$ into
	\eqref{eq:flsp_shared_remainder}, and use
	Theorem~\ref{thm:flsp_shared_residual}.
\end{proof}

The admissible shared stepsize is independent of $\mu$; the resulting
coverage factor enters as $1/\mu$. This improvement uses the common
subspace of the entire server direction. It does not extend automatically
to client-specific projections.

\begin{corollary}[Comparison under common constants]
	\label{cor:flsp_comparison}
	Consider shared and client-specific runs on the same objective and
	initialization, with the same $\gamma,K,S$ and common valid constants
	$L,\sigma,\zeta,\mu$. Their bounds in Theorem~\ref{thm:flsp_convergence}
	are respectively
	\[
	\mathcal R_T
	\quad\text{and}\quad
	\mathcal R_T+\frac8{\mu^2T}\sum_{t<T}\flspE\flspnf{b_t^{\rm local}}^2.
	\]
	The second bound is strictly larger if the total bias energy is positive.
\end{corollary}
\begin{proof}
	Apply \eqref{eq:flsp_main_bound} to each run and use $b_t^{\rm shared}=0$.
\end{proof}

The comparison concerns upper bounds evaluated along different trajectories.
It is not an ordering of the actual optimization errors. In particular,
client-specific projections can converge when their bias is sufficiently
small, even if they are not identical.

\begin{corollary}[Client-specific projections with controlled bias]
	\label{cor:flsp_relative_bias}
	Under Theorem~\ref{thm:flsp_convergence}, suppose for some $\theta\in[0,1)$,
	\begin{equation}
		\flspE\flspnf{b_t}^2\le\frac{\theta\mu^2}{8}\flspE\flspnf{G_t}^2
		\quad\text{for every }t.
		\label{eq:flsp_relative_bias}
	\end{equation}
	Then $T^{-1}\sum_{t<T}\flspE\flspnf{G_t}^2\le\mathcal R_T/(1-\theta)$.
	More generally, with $\gamma=c/(LK\sqrt T)$ and $0<c\le\mu/16$,
	$T^{-1}\sum_{t<T}\flspE\flspnf{b_t}^2=O(T^{-1/2})$ is sufficient for the same
	$O(T^{-1/2})$ stationarity rate.
\end{corollary}
\begin{proof}
	Substitute \eqref{eq:flsp_relative_bias} into \eqref{eq:flsp_main_bound}:
	its bias term is at most $\theta T^{-1}\sum_{t<T}\flspE\flspnf{G_t}^2$.
	Move it to the left. The second statement follows directly from the same bound.
\end{proof}

For example, a pathwise condition
$\|A_t-A_t^2\|_2 N^{-1}\sum_i\flspnf{d_i^t}^2
\le\theta\mu^2\flspnf{G_t}^2/8$ implies
\eqref{eq:flsp_relative_bias} by Theorem~\ref{thm:flsp_bias}.
Thus disagreement alone does not force a residual error; its coupling to
heterogeneity determines the obstruction.

\subsection{A strict separation with exact local SVD projections}
\label{sec:flsp_separation}

\begin{theorem}[Local optimal subspaces can stall away from stationarity]
	\label{thm:flsp_separation}
	There exist four $1$-smooth, strongly convex quadratic client objectives
	on $\flspR^{2\times2}$ with bounded gradient heterogeneity for which the following
	holds under full participation, exact gradients, rank $r=1$, and any
	fixed number of local steps $K\ge1$ (with each projector fixed within a round).
	Starting from the same $W^0=0$:
	\begin{enumerate}
		\item Local projected SGD, using each local gradient's exact leading left
		singular subspace at the start of every round, stays at $W^t=0$ for all $t$. Its squared
		global-gradient norm is $1/2$, and its global objective gap is $1/4$.
		Every local projection captures $4/5$ of its own round-start gradient energy.
		The average local projection also satisfies aggregate global-gradient
		coverage with $\mu=1/2$.
		\item Shared projected SGD, using a leading left singular subspace of the
		round-start global gradient, reaches the global minimizer in two rounds when
		$\gamma=1$. For any constant $0<\gamma\le1$, define
		$\nu=1-(1-\gamma)^K\in(0,1]$. It satisfies
		\begin{equation}
			f(W^t)-f^\star
			\le\left(1-\frac{\nu(2-\nu)}2\right)^t
			[f(W^0)-f^\star].
			\label{eq:flsp_example_rate}
		\end{equation}
		In particular, strict separation also holds for any sufficiently small
		stepsize satisfying Theorem~\ref{thm:flsp_convergence}.
	\end{enumerate}
\end{theorem}

\begin{proof}
	Let
	\[
	C_1=\flspdiag(2,1),\quad C_2=\flspdiag(-2,1),\quad
	C_3=\flspdiag(1,2),\quad C_4=\flspdiag(1,-2),
	\]
	and set
	\[
	f_i(W)=\tfrac12\flspnf{W}^2+\flspip{C_i}{W}.
	\]
	Then $\nabla f_i(W)=W+C_i$ and every Hessian is the identity.
	Their average is $\bar C=\frac12 I_2$, so
	\[
	f(W)=\tfrac12\flspnf{W}^2+\flspip{\tfrac12I_2}{W},\qquad
	W^\star=-\tfrac12 I_2,
	\]
	and
	\begin{equation}
		f(W)-f^\star=\tfrac12\flspnf{W-W^\star}^2,
		\qquad \nabla f(W)=W+\tfrac12 I_2.
		\label{eq:flsp_example_objective}
	\end{equation}
	The heterogeneity is independent of $W$:
	\[
	\frac14\sum_{i=1}^4\flspnf{\nabla f_i(W)-\nabla f(W)}^2
	=\frac14\sum_i\flspnf{C_i-\bar C}^2
	=5-\tfrac12=\tfrac92.
	\]
	Thus $\zeta^2=9/2$ is a global valid bound, and $\sigma=0$.
	At $W=0$, the leading singular values are all $2$, strictly larger than
	the other singular value $1$. Hence the rank-one orthogonal projectors
	are uniquely determined as
	\[
	\Pi_1=\Pi_2=\flspdiag(1,0),\qquad
	\Pi_3=\Pi_4=\flspdiag(0,1).
	\]
	Their projected gradients are
	\[
	\Pi_1 C_1=\flspdiag(2,0),\quad
	\Pi_2 C_2=\flspdiag(-2,0),\quad
	\Pi_3 C_3=\flspdiag(0,2),\quad
	\Pi_4 C_4=\flspdiag(0,-2).
	\]
	The average is zero. To see that this persists for any $K$, unroll the
	quadratic local dynamics with a fixed projector:
	\[
	W_i^{t,k}=W^t-[1-(1-\gamma)^k]\Pi_i^t(W^t+C_i).
	\]
	Indeed, if $Y_i^k=\Pi_i^t(W_i^{t,k}+C_i)$, idempotence gives
	$Y_i^{k+1}=(1-\gamma)Y_i^k$ and
	$Y_i^0=\Pi_i^t(W^t+C_i)$.
	Summing $W_i^{t,k+1}-W_i^{t,k}=-\gamma Y_i^k$ and using
	$\gamma\sum_{j=0}^{k-1}(1-\gamma)^j=1-(1-\gamma)^k$
	gives the displayed expression.
	At $W^t=0$, the averaged endpoint is therefore
	$-\nu\,\frac14\sum_i\Pi_iC_i=0$.
	Thus every server round leaves $W=0$ unchanged,
	regardless of the positive local stepsize. By induction, all subsequent
	local SVDs return the same projectors. Each retained-energy ratio is
	$4/(4+1)=4/5$, yet
	\[
	\flspnf{\nabla f(0)}^2=\flspnf{\tfrac12 I_2}^2=\tfrac12,
	\qquad f(0)-f^\star=\tfrac14.
	\]
	Moreover,
	\[
	A=\tfrac12I_2,\quad A-A^2=\tfrac14I_2,\quad
	AG=\tfrac14I_2,\quad b=-\tfrac14I_2.
	\]
	Thus $\flspip G{AG}=\frac12\flspnf G^2$: global-gradient coverage holds even
	though the bias cancels the covered descent direction exactly.
	Choose any leading left singular direction of the current global gradient
	$G=W+\frac12I_2$, with projector $\Pi$. At the round-start point,
	linearity and full participation give
	\[
	\frac14\sum_i\Pi\nabla f_i(W)=\Pi\nabla f(W)=\Pi G.
	\]
	Since $G$ is a $2\times2$ matrix, a leading left singular direction retains
	at least half its squared Frobenius norm:
	$\flspnf{\Pi G}^2\ge\frac12\flspnf G^2$.
	The same unrolled local dynamics show that after $K$ steps,
	$W^+=W-\nu\Pi G$, where $\nu=1-(1-\gamma)^K$.
	For this quadratic, $G^+=G-\nu\Pi G$. Orthogonality gives the exact
	identity
	\[
	\flspnf{G^+}^2=\flspnf G^2-\nu(2-\nu)\flspnf{\Pi G}^2
	\le\left(1-\frac{\nu(2-\nu)}2\right)\flspnf G^2.
	\]
	Using \eqref{eq:flsp_example_objective} proves
	\eqref{eq:flsp_example_rate}.
	For $\gamma=1$, $\nu=1$ for every $K\ge1$. Choose
	$\Pi^0=\flspdiag(1,0)$ to resolve the initial tie:
	\[
	W^1=\flspdiag(-1/2,0),\quad G_1=\flspdiag(0,1/2),\quad
	\Pi^1=\flspdiag(0,1),\quad W^2=-\tfrac12I_2=W^\star.
	\]
	After reaching the minimizer, the shared aggregate remains zero.
\end{proof}

Both runs use the same rank, participation, local stepsize, and number of
local steps. Their difference is the choice of the round-wise subspace.
The example establishes a failure of exact local SVD selection despite
positive local and aggregate coverage. The shared comparator uses the global
gradient; constructing an adequate approximation with limited communication
is a separate requirement, addressed by the coverage bounds below.

\subsection{Coverage residual of the implemented projector average}
\label{sec:flsp_actual_selection}

The shared basis in FedLore is obtained from candidate projectors, not from
the SVD of an averaged gradient. We analyze this construction directly.
Candidate projectors below are denoted by $Q_i$ to distinguish them from
the optimization projector $\Pi^t$.

\begin{proposition}[Residual certificate for projector averaging]
	\label{prop:flsp_actual_certificate}
	Let $H_i\in\flspR^{m\times n}$ be candidate gradient matrices and let $Q_i$
	be rank-$r$ orthogonal projectors, with $1\le r\le m$.
	Define
	\begin{align}
		\bar H&=\frac1N\sum_i H_i,\qquad
		A^{\rm cand}=\frac1N\sum_i Q_i,\nonumber\\
		\epsilon_{\rm loc}^2&=\frac1N\sum_i\flspnf{(I-Q_i)H_i}^2,\qquad
		M=\max_i\flspno{H_i},\nonumber\\
		d_r&=r-\sum_{j=1}^r\lambda_j(A^{\rm cand}),
		\label{eq:flsp_certificate_terms}
	\end{align}
	where eigenvalues are ordered nonincreasingly.
	Let $\widehat\Pi$ be a rank-$r$ orthogonal projector satisfying
	\begin{equation}
		\operatorname{tr}(\widehat\Pi A^{\rm cand})
		\ge\sum_{j=1}^r\lambda_j(A^{\rm cand})-\varepsilon_{\rm eig},
		\qquad \varepsilon_{\rm eig}\ge0.
		\label{eq:flsp_eigensolver_accuracy}
	\end{equation}
	For any target gradient $G$ and $\delta=\flspnf{\bar H-G}$,
	\begin{align}
		\flspnf{(I-\widehat\Pi)G}
		&\le\epsilon_{\rm loc}+M\sqrt{d_r+\varepsilon_{\rm eig}}+\delta,
		\label{eq:flsp_actual_residual}\\
		\flspnf{\widehat\Pi G}^2
		&\ge\flspnf G^2-
		(\epsilon_{\rm loc}+M\sqrt{d_r+\varepsilon_{\rm eig}}+\delta)^2.
		\label{eq:flsp_actual_coverage}
	\end{align}
	For exact candidate SVDs, $\epsilon_{\rm loc}^2$ is the mean discarded
	squared singular-value energy. For an exact leading eigenspace of
	$A^{\rm cand}$, one can take $\varepsilon_{\rm eig}=0$.
\end{proposition}

\begin{proof}
	Since every $Q_i$ has rank $r$, $\operatorname{tr}(A^{\rm cand})=r$.
	For every rank-$r$ orthogonal projector $\Pi$,
	\begin{align*}
		\frac1N\sum_i\flspnf{(I-\Pi)Q_i}^2
		&=\frac1N\sum_i\operatorname{tr}(Q_i(I-\Pi)Q_i)\\
		&=r-\operatorname{tr}(\Pi A^{\rm cand}).
	\end{align*}
	The eigenvalue variational principle therefore gives
	\[
	d_r=\min_{\Pi}\frac1N\sum_i\flspnf{(I-\Pi)Q_i}^2\ge0,
	\]
	and~\eqref{eq:flsp_eigensolver_accuracy} bounds this average for
	$\widehat\Pi$ by $d_r+\varepsilon_{\rm eig}$.
	Decompose $H_i=Q_iH_i+(I-Q_i)H_i$.
	The triangle inequality and Jensen give
	\begin{align*}
		\flspnf{(I-\widehat\Pi)\bar H}
		&\le\frac1N\sum_i\flspnf{(I-\widehat\Pi)Q_iH_i}
		+\frac1N\sum_i\flspnf{(I-Q_i)H_i}\\
		&\le M\left(\frac1N\sum_i\flspnf{(I-\widehat\Pi)Q_i}^2\right)^{1/2}
		+\epsilon_{\rm loc}\\
		&\le M\sqrt{d_r+\varepsilon_{\rm eig}}+\epsilon_{\rm loc}.
	\end{align*}
	Here $\flspnf{XY}\le\flspnf X\flspno Y$ controls each product.
	Contraction and $G=\bar H+(G-\bar H)$ prove
	\eqref{eq:flsp_actual_residual}. Pythagoras proves
	\eqref{eq:flsp_actual_coverage}.
	No independence between the $H_i$, their projectors, and $G$ is needed.
\end{proof}

\paragraph{Application to FedLore.}
Under full participation and for $t\ge1$, let
$H_i^t=g_i^{t-1,K-1}$ be the last \emph{positive} local gradient estimate
from the preceding round, and let $Q_i^t=\widehat P_i^{t-1}(\widehat P_i^{t-1})^\top$ be its
candidate projector, where $\widehat P_i^{t-1}$ denotes the uploaded
candidate basis (called $P_i^{t-1}$ in the main algorithm), not the basis
used for local optimization in the preceding round. A negative gradient
produces the same singular subspace.
Then $A_t^{\rm cand}$ and its leading projector $\Pi^t$ are
$\mathcal F_t$-measurable. Apply
Proposition~\ref{prop:flsp_actual_certificate} to these matrices with
$G=G_t$. Denote its three-term right-hand side by
\begin{equation}
	R_t^{\rm sel}=\epsilon_{{\rm loc},t}
	+M_t\sqrt{d_{r,t}+\varepsilon_{{\rm eig},t}}+\delta_t,
	\qquad \delta_t=\flspnf{N^{-1}\sum_iH_i^t-G_t}.
	\label{eq:flsp_selection_certificate}
\end{equation}
The quantities for $t=0$ are unnecessary: its contribution can be bounded
directly by $\flspnf{(I-\Pi^0)G_0}^2$.
The certificate separates local truncation, disagreement of candidate
subspaces (and eigensolver error), and the mismatch between the candidate
mean gradient and the current global gradient. It provides a sufficient
condition, not a necessary one; the bound may be loose under cancellation.

\begin{corollary}[Full-gradient bound for the implemented selection rule]
	\label{cor:flsp_actual_rule}
	Under the shared projected-SGD dynamics of
	Theorem~\ref{thm:flsp_shared_residual}, full participation, and the preceding
	selection rule,
	\begin{align}
		\frac1{\Lambda_T}\sum_{t<T}h_t\flspE\flspnf{G_t}^2
		&\le\mathcal B_T+\frac{h_0}{\Lambda_T}\flspE\rho_0^2
		+\frac1{\Lambda_T}\sum_{t=1}^{T-1}h_t\flspE(R_t^{\rm sel})^2.
		\label{eq:flsp_actual_rule_bound}
	\end{align}
	The initial projector may be deterministic or random. With
	$\gamma_t=c/(LK\sqrt T)$, $0<c\le1/8$, fixed $K,N$, and finite
	$\flspE\rho_0^2$, the right-hand side is $O(T^{-1/2})$ whenever
	$T^{-1}\sum_{t=1}^{T-1}\flspE(R_t^{\rm sel})^2=O(T^{-1/2})$.
	Alternatively, if $\flspE(R_t^{\rm sel})^2\le
	(1-\mu)\flspE\flspnf{G_t}^2$ for $t\ge1$ and the initialization satisfies
	expected coverage with the same $\mu>0$, the bound $\mathcal B_T/\mu$ holds.
\end{corollary}
\begin{proof}
	The certificate gives $\rho_t^2\le(R_t^{\rm sel})^2$ for $t\ge1$.
	Substitute into~\eqref{eq:flsp_residual_stationarity}.
	The rate follows from~\eqref{eq:flsp_shared_remainder}; the initialization
	term is $O(T^{-1})$ for constant-within-run stepsizes.
	The last assertion follows from the orthogonal decomposition and
	\eqref{eq:flsp_expected_coverage}.
\end{proof}

\begin{proposition}[Mismatch of reused last-step gradients]
	\label{prop:flsp_reused_gradient}
	Under full participation and the oracle conditions, set
	$X_i^t=W_i^{t-1,K-1}$ and $H_i^t=g_i^{t-1,K-1}$ for $t\ge1$.
	Assume additionally that the candidate noises are conditionally uncorrelated
	in the ambient space at that local step: for $i\ne j$,
	$\flspE[\flspip{\varepsilon_i^{t-1,K-1}}{\varepsilon_j^{t-1,K-1}}
	\mid\mathcal H_{t-1,K-1}]=0$.
	Then the mismatch in~\eqref{eq:flsp_selection_certificate} satisfies
	\begin{equation}
		\flspE\delta_t^2
		\le\frac{2\sigma^2}{N}
		+\frac{2L^2}{N}\sum_i\flspE\flspnf{X_i^t-W^t}^2.
		\label{eq:flsp_reused_gradient_mismatch}
	\end{equation}
\end{proposition}
\begin{proof}
	Write $\bar\varepsilon=N^{-1}\sum_i
	[H_i^t-\nabla f_i(X_i^t)]$. The last-step oracle conditions imply
	$\flspE\flspnf{\bar\varepsilon}^2\le\sigma^2/N$ by the additional
	ambient-space uncorrelatedness assumption.
	Since
	\[
	\frac1N\sum_iH_i^t-G_t
	=\bar\varepsilon+\frac1N\sum_i
	[\nabla f_i(X_i^t)-\nabla f_i(W^t)],
	\]
	the two-term squared-norm inequality, Jensen, and smoothness give the bound.
	The candidate noise can be correlated with $W^t$; no cross-term cancellation
	between these two summands is used.
\end{proof}

\paragraph{Scope of this construction bound.}
For Proposition~\ref{prop:flsp_reused_gradient}, ambient-space conditional
uncorrelatedness of the candidate noises is an additional hypothesis.
With only the original projected-noise condition, Jensen gives the valid
replacement $2\sigma^2$ for $2\sigma^2/N$.
Fixed candidate minibatch noise need not vanish with $T$, and the candidate
truncation and disagreement terms need not vanish either. Thus
\eqref{eq:flsp_actual_rule_bound} does not establish unconditional
$O(T^{-1/2})$ full stationarity of the implemented rule. It states explicitly
which residuals must be controlled. For partial participation, constructing
candidates from only the previous active subset requires an additional
sampling-error analysis; the full-participation certificate is not silently
used as a population-gradient estimator bound in that setting.

\begin{proposition}[Agreement alone does not imply coverage]
	\label{prop:flsp_no_automatic_coverage}
	There exist two $1$-smooth strongly convex quadratic objectives with bounded
	heterogeneity for which exact local leading singular projectors agree
	uniquely, but their average selects a shared projector with zero coverage
	at a nonstationary point.
\end{proposition}
\begin{proof}
	Take $f_i(W)=\tfrac12\flspnf W^2+\flspip{C_i}{W}$, where
	$C_1=\operatorname{diag}(2,1)$ and $C_2=\operatorname{diag}(-2,1)$.
	At $W=0$, both unique leading rank-one left projectors are
	$Q_1=Q_2=\operatorname{diag}(1,0)$, whereas
	$G=\operatorname{diag}(0,1)$. Consequently $\widehat\Pi=Q_1$,
	$b=0$, and $\widehat\Pi G=0$ although $\flspnf G^2=1$.
	The heterogeneity is $N^{-1}\sum_i\flspnf{C_i-G}^2=4$ at every $W$.
	With $K=1$, full participation, exact gradients, and $\Pi^0=Q_1$,
	the two local projected updates cancel. The server stays at $W=0$,
	and the last-step candidate rule selects the same projector in every round.
	Here $d_r=0$, $\delta=0$, and $\epsilon_{\rm loc}=1$, so
	\eqref{eq:flsp_actual_residual} is attained with equality.
\end{proof}

\subsection{Coverage of an approximate shared subspace}
\label{sec:flsp_selection}

This subsection treats a different construction: the leading singular
subspace of a global-gradient estimate. Its results are benchmarks and
sufficient conditions for that construction, not guarantees for the
candidate-projector average. The latter is covered separately by
Proposition~\ref{prop:flsp_actual_certificate}.

\begin{proposition}[Shared subspace selection]
	\label{prop:flsp_global_svd}
	Fix a point $W$ with global gradient $G$. Under full participation,
	exact gradients, and $K=1$, any shared rank-$r$ orthogonal projector yields
	\begin{equation}
		f(W-\gamma\Pi G)
		\le f(W)-\gamma(1-L\gamma/2)\flspnf{\Pi G}^2.
		\label{eq:flsp_shared_descent}
	\end{equation}
	For $0<\gamma<2/L$, a projector onto the leading $r$ left singular vectors
	of $G$ maximizes the certified decrease on the right-hand side.
	For $G\ne0$, its retained fraction is
	\begin{equation}
		\mu_\star(G)=\frac{\sum_{j=1}^r\sigma_j(G)^2}{\flspnf G^2}
		\ge\frac{\min\{r,q\}}q,\qquad q=\operatorname{rank}(G),
		\label{eq:flsp_optimal_coverage}
	\end{equation}
	where singular values are padded by zeros to length $m$. In particular,
	$\mu_\star(G)\ge r/m$ for $1\le r\le m$.
\end{proposition}
\begin{proof}
	Smoothness, $\flspip G{\Pi G}=\flspnf{\Pi G}^2$, and the shared aggregation
	identity prove \eqref{eq:flsp_shared_descent}.
	To maximize captured energy, write $GG^\top=\sum_{j=1}^m\lambda_j u_ju_j^\top$
	with $\lambda_1\ge\cdots\ge\lambda_m\ge0$. For an orthogonal rank-$r$
	projector, $a_j=u_j^\top\Pi u_j$ satisfies $0\le a_j\le1$ and
	$\sum_j a_j=r$. Therefore,
	\[
	\flspnf{\Pi G}^2=\flsptr(\Pi GG^\top)=\sum_j\lambda_j a_j
	\le\sum_{j=1}^r\lambda_j.
	\]
	Equality is attained by the leading eigenspace. Exactly $q$ eigenvalues
	are positive. If $r<q$, the mean of the largest $r$ is at least the mean
	of those $q$ values; if $r\ge q$, all energy is retained. This proves
	\eqref{eq:flsp_optimal_coverage}.
	The optimum is for the smoothness bound; it need not minimize the actual
	post-update loss of a nonquadratic objective.
\end{proof}

\begin{proposition}[Coverage from an approximate global-gradient estimate]
	\label{prop:flsp_approximate}
	Let $G\ne0$, let $H$ satisfy $\flspnf{H-G}\le\epsilon\flspnf G$, and let
	$\widehat\Pi$ be a leading rank-$r$ left singular projector of $H$.
	If $\Pi_\star$ is a leading rank-$r$ projector of $G$, then
	\begin{equation}
		\frac{\flspnf{\widehat\Pi G}^2}{\flspnf G^2}
		\ge\left[\max\{\sqrt{\mu_\star(G)}-2\epsilon,0\}\right]^2.
		\label{eq:flsp_approximate_coverage}
	\end{equation}
	In particular, uniform bounds $\mu_\star(G_t)\ge\underline\mu_\star>0$
	and $\epsilon_t\le\bar\epsilon<\sqrt{\underline\mu_\star}/2$ imply
	Assumption~\ref{ass:flsp_coverage} for the shared method with
	$\mu=(\sqrt{\underline\mu_\star}-2\bar\epsilon)^2$.
\end{proposition}
\begin{proof}
	Triangle inequalities, contraction, and optimality on $H$ give
	\begin{align*}
		\flspnf{\widehat\Pi G}
		&\ge\flspnf{\widehat\Pi H}-\flspnf{H-G}\\
		&\ge\flspnf{\Pi_\star H}-\flspnf{H-G}\\
		&\ge\flspnf{\Pi_\star G}-2\flspnf{H-G}\\
		&\ge(\sqrt{\mu_\star(G)}-2\epsilon)\flspnf G.
	\end{align*}
	The norm is nonnegative, so take the positive part before squaring.
	At $G=0$, the coverage inequality is automatically true.
\end{proof}

\begin{proposition}[Additive coverage error]
	\label{prop:flsp_additive}
	Replace Assumption~\ref{ass:flsp_coverage} by
	\begin{equation}
		\flspip{G_t}{A_tG_t}\ge\mu\flspnf{G_t}^2-\chi_t,
		\qquad \chi_t\ge0,
		\label{eq:flsp_additive_coverage}
	\end{equation}
	where $\chi_t$ is $\mathcal F_t$-measurable and integrable. All other
	assumptions and the stepsize restriction remain unchanged. Then
	\begin{equation}
		\frac1T\sum_{t<T}\flspE\flspnf{G_t}^2
		\le\mathcal R_T+\frac8{\mu^2T}\sum_{t<T}\flspE\flspnf{b_t}^2
		+\frac4{\mu T}\sum_{t<T}\flspE\chi_t.
		\label{eq:flsp_additive_bound}
	\end{equation}
	In particular, let $H_t$ be $\mathcal F_t$-measurable with
	$\flspE\flspnf{H_t-G_t}^2<\infty$, and let the shared
	projector $\widehat\Pi_t$ maximize $\flspnf{\Pi H_t}^2$ over rank-$r$
	orthogonal projectors. If
	$\mu_\star(G_t)\ge\underline\mu_\star>0$ whenever $G_t\ne0$, then
	\eqref{eq:flsp_additive_bound} holds with
	\begin{equation}
		\mu=\frac{\underline\mu_\star}{2},\qquad
		\chi_t=4\flspnf{H_t-G_t}^2,\qquad b_t=0.
		\label{eq:flsp_additive_svd}
	\end{equation}
\end{proposition}
\begin{proof}
	Under \eqref{eq:flsp_additive_coverage}, the right-hand side of
	\eqref{eq:flsp_first_order} acquires only the additive term $\chi_t$.
	Consequently, the one-round inequality \eqref{eq:flsp_one_round} acquires
	$h\chi_t$. Summing and dividing by $\mu hT/4$ proves
	\eqref{eq:flsp_additive_bound}.
	
	Let $\Pi_{\star,t}$ be a leading rank-$r$ projector of $G_t$. The triangle
	inequality and optimality on $H_t$ imply
	\[
	\flspnf{\Pi_{\star,t}G_t}
	\le\flspnf{\Pi_{\star,t}H_t}+\flspnf{H_t-G_t}
	\le\flspnf{\widehat\Pi_tG_t}+2\flspnf{H_t-G_t}.
	\]
	Squaring and using $(a+b)^2\le2a^2+2b^2$ gives
	\[
	\flspnf{\widehat\Pi_tG_t}^2
	\ge\tfrac12\flspnf{\Pi_{\star,t}G_t}^2-4\flspnf{H_t-G_t}^2
	\ge\tfrac12\underline\mu_\star\flspnf{G_t}^2-4\flspnf{H_t-G_t}^2.
	\]
	This is \eqref{eq:flsp_additive_coverage} with the stated constants.
	At $G_t=0$ the inequality is automatic.
\end{proof}

Unlike a uniform relative-error condition, the additive formulation does
not require the estimator error to shrink in proportion to a vanishing
true gradient. With $\gamma=c/(LK\sqrt T)$ and $0<c\le\mu/16$,
an average squared estimation error of order $T^{-1/2}$ suffices to retain
the same stationarity rate. A fixed error level instead contributes its
explicit average penalty in \eqref{eq:flsp_additive_bound}.

\begin{corollary}[Stale global-gradient estimates]
	\label{cor:flsp_stale}
	Under Proposition~\ref{prop:flsp_additive}, suppose
	\[
	H_t=\nabla f(W^{t-\tau_t})+E_t^{\rm est},\qquad 0\le\tau_t\le t.
	\]
	The delay and both terms are $\mathcal F_t$-measurable. Provided the following
	right-hand side is integrable, one may take
	\begin{equation}
		\chi_t=8\flspnf{E_t^{\rm est}}^2
		+8L^2\tau_t\sum_{j=t-\tau_t}^{t-1}\flspnf{W^{j+1}-W^j}^2.
		\label{eq:flsp_stale}
	\end{equation}
	The sum is zero when $\tau_t=0$.
\end{corollary}
\begin{proof}
	By smoothness of $f$,
	\begin{align*}
		\flspnf{H_t-G_t}^2
		&\le2\flspnf{E_t^{\rm est}}^2
		+2L^2\flspnf{W^{t-\tau_t}-W^t}^2\\
		&\le2\flspnf{E_t^{\rm est}}^2
		+2L^2\tau_t\sum_{j=t-\tau_t}^{t-1}\flspnf{W^{j+1}-W^j}^2.
	\end{align*}
	Multiply by four and use \eqref{eq:flsp_additive_svd}.
\end{proof}

\subsection{Exact synchronization and optimizer scope}
\label{sec:flsp_sync}

\begin{proposition}[Exact coordinate aggregation with common decay]
	\label{prop:flsp_exact_sync}
	Fix a shared basis $P$ during a round. Suppose all clients start at the same
	$W$ and use updates
	\[
	W_{i,k+1}=a_kW_{i,k}+PD_{i,k},\qquad W_{i,0}=W,
	\]
	where $a_k$ are common known scalars and $D_{i,k}\in\flspR^{r\times n}$
	may be arbitrary, including coordinatewise adaptive updates. Define
	\[
	a=\prod_{k=0}^{K-1}a_k,\qquad
	Z_i=\sum_{k=0}^{K-1}\left(\prod_{\ell=k+1}^{K-1}a_\ell\right)D_{i,k}.
	\]
	An empty product is one. For any participating set $\mathcal S$,
	\begin{equation}
		W_{i,K}=aW+PZ_i,\qquad
		\frac1S\sum_{i\in\mathcal S}W_{i,K}
		=aW+P\left(\frac1S\sum_{i\in\mathcal S}Z_i\right).
		\label{eq:flsp_exact_sync_decay}
	\end{equation}
\end{proposition}
\begin{proof}
	Induction in $k$ unrolls the common scalar recursion and gives the endpoint
	factorization. Linearity of the mean gives the second identity.
\end{proof}

For zero weight decay, $a_k=1$. With stepsize $\eta_k$ and scale $\alpha$,
the Adam implementation uses $D_{i,k}=\eta_k\alpha N_{i,k+1}$ and must
accumulate these \emph{weighted} coordinates if the learning rate varies
within a round. Constant learning rate reduces to the factorization in the
main algorithm. With common decoupled weight decay $\lambda$,
$a_k=1-\eta_k\lambda$. The increment then contains the generally full-rank
term $(a-1)W$, but a client holding the synchronized cache $W$ reconstructs
it from $a$ without receiving a dense matrix.

Under full participation and synchronized caches, each projected matrix
requires $rn$ coordinate scalars and $mr$ candidate-basis scalars uploaded
per client, and $rn$ aggregated-coordinate scalars and $mr$ next-basis
scalars downloaded per client per round, excluding initialization and
unprojected parameters. This is $r(m+n)$ scalars in each direction.
The two coordinate moment tensors occupy $2rn$ scalars; this is not a
claim about total peak training memory. Under intermittent participation,
a returning client additionally needs the missed increments and their
bases, or a fresh model. The convergence sampling model does not itself
guarantee this full-participation communication cost.

\paragraph{Adaptive updates.}
Proposition~\ref{prop:flsp_exact_sync} holds for the Adam implementation
because it uses only a fixed basis and linear aggregation of endpoint
coordinates. The bias identity concerns linear projected gradients, and
the stationarity theorems concern projected SGD without weight decay.
In general, averaging independently normalized Adam directions does not
equal normalizing an averaged gradient. Sharing therefore does not imply
that all adaptive-optimizer bias or local drift vanishes. Projector
averaging is invariant to the representation of each candidate subspace;
coordinatewise Adam need not be invariant under an arbitrary rotation of
the shared basis. These are distinct properties.

\paragraph{Diagnostics.}
At selected checkpoints, measure the retained-energy ratio
$\flspnf{\Pi^tG_t}^2/\flspnf{G_t}^2$ when $G_t\ne0$, the residual
$\rho_t^2$, and, for client-specific projections, $\flspnf{b_t}^2$ and
$\langle G_t,A_tG_t+b_t\rangle_F$. The direction of the bias matters, so
projector disagreement alone is insufficient evidence of harmful coupling.
For the implemented selection rule, also report local discarded energy,
$d_{r,t}$, and candidate-gradient staleness where feasible. Diagnostic
global-gradient estimation has its own cost and need not be part of the
training protocol. The geometric certificate is valid for candidate
gradients from either implementation. Its stationarity consequence requires
the projected-SGD updates specified in Corollary~\ref{cor:flsp_actual_rule}.

\subsection{C4 Partitioning Protocol}
\label{app:c4_partition}

\paragraph{Domain-based grouping.}
Since C4 does not provide classification labels, we construct
heterogeneous client datasets using the source information
associated with each document.
For every training document, we extract the hostname from
its URL, convert it to lowercase, and remove a leading
\texttt{www.} prefix.
Documents with the same normalized hostname form a source
group.
Different subdomains remain separate groups, and documents
with missing or invalid hostnames are assigned to a
dedicated unknown-source group.
These groups serve only as partitioning metadata;
they are not prediction targets and are not added to
the model inputs.
Source identity provides a proxy for differences in
document content and writing style, rather than an
explicit semantic topic label.

\paragraph{C4 Dirichlet allocation.}
Let $\mathcal{D}_{\mathrm{train}}$ denote the selected
C4 training corpus before client allocation, and let
$\{\mathcal{D}_c\}_{c=1}^{C}$ be its source groups.
We partition these documents among $N=20$ clients.
For each group $c$, we independently sample
\begin{equation}
	\boldsymbol{p}_c
	\sim
	\operatorname{Dirichlet}
	\left(
	\alpha_{\mathrm{C4}}\boldsymbol{1}_N
	\right),
	\qquad
	\alpha_{\mathrm{C4}}=0.5,
	\label{eq:c4_dirichlet}
\end{equation}
where $p_{c,i}$ specifies the target fraction of group
$c$ assigned to client $i$.
Smaller concentration parameters produce more uneven
allocations of each source group across clients.

To convert these proportions into integer document counts,
let $M_c=|\mathcal{D}_c|$ and initialize
\begin{equation}
	n_{c,i}=\left\lfloor M_c p_{c,i}\right\rfloor.
\end{equation}
The remaining
$M_c-\sum_{i=1}^{N}n_{c,i}$ documents are assigned,
one per client, to those with the largest fractional
remainders $M_cp_{c,i}-n_{c,i}$.
Ties are resolved by client index.
We randomly shuffle the documents within each group and
allocate consecutive blocks of sizes
$n_{c,1},\ldots,n_{c,N}$ to the clients.
The resulting client datasets satisfy
\begin{equation}
	\mathcal{D}_{\mathrm{train}}
	=
	\bigcup_{i=1}^{N}\mathcal{D}_i,
	\qquad
	\mathcal{D}_i\cap\mathcal{D}_j=\varnothing
	\quad \text{for } i\neq j.
\end{equation}
Thus, every training document belongs to exactly one client.

\paragraph{Client sizes and local sampling.}
This procedure allows both source composition and dataset
size to vary across clients.
We do not rebalance client datasets after allocation,
because doing so would alter the sampled distribution.
If any client receives no documents, we redraw the allocation
using the next values from the same seeded random-number
generator.
All clients participate in each communication round and
perform $K=50$ local steps with batch size $16$.
Each client samples from its own shuffled data stream,
reshuffling and restarting the stream when it is exhausted.
Consequently, clients perform the same number of local
updates even when their dataset sizes differ.

\paragraph{Tokenization and sequence construction.}
Client assignment is performed at the document level,
before tokenization or sequence packing.
Within each client, documents are tokenized using the same
tokenizer, separated by an end-of-sequence token, and packed
into sequences of length $256$.
A final incomplete block is padded, with padding positions
excluded from the language-modeling loss.
Packing is performed separately for each client, so no
training sequence contains documents from different clients.
The source metadata used for partitioning are excluded
from the training text.

\paragraph{Evaluation and reproducibility.}
The validation corpus is kept separate from the training
partition and is identical across methods.
We evaluate all global models on this common validation
corpus, computing perplexity as the exponential of the
mean negative log-likelihood over non-padding target tokens.
For each seed in $\{42,43,44,45,46\}$, we generate one client
partition and reuse it for every compared method.
The selected training documents, tokenizer, document
assignments, and validation corpus are therefore matched
across methods within each run.
We use separate random-number streams for partitioning
and model training so that differences in training code
do not change the client allocation.

\clearpage
\section{Baseline Algorithm}
\label{app:fedlore_baseline}

\begin{algorithm}[H]
	\small
	\caption{Local GaLore in Federated Training}
	\label{alg:local_galore}
	\begin{algorithmic}[1]
		\REQUIRE Global weights $W^0$, learning rate $\eta$, scaling factor $\alpha$, rank $r$,
		decay rates $\beta_1,\beta_2$, numerical constant $\epsilon$,
		communication rounds $T$, local steps $K$, number of clients $N$, weight decay $\lambda=0$
		\ENSURE Trained global model $W^T$
		
		\FOR{$t=0$ to $T-1$}
		\STATE Server broadcasts current global model $W^t$ to all clients
		
		\FOR{each client $i \in \{1,\dots,N\}$ in parallel}
		
		\STATE Compute a round-start gradient using a separate construction mini-batch
		(independent of the subsequent local optimization draws):
		$G_i^t \gets -\nabla_W f_i(W^t;\xi_i^t)
		\in \mathbb{R}^{m\times n}$
		
		\STATE Compute truncated SVD:
		$U,S,V \gets \mathrm{SVD}(G_i^t)$
		
		\STATE Construct local projector:
		$P_i^t \gets U[:,1\!:\!r]
		\in \mathbb{R}^{m\times r}$
		
		\STATE Initialize
		$M_i^{t,0}\gets 0$,
		$V_i^{t,0}\gets 0$,
		$W_i^{t,0}\gets W^t$
		
		\FOR{$k=0$ to $K-1$}
		
		\STATE Sample a fresh mini-batch $\xi_i^{t,k}$ and compute gradient:
		$G_i^{t,k}\gets-\nabla_W f_i(W_i^{t,k};\xi_i^{t,k})
		\in\mathbb{R}^{m\times n}$
		
		\STATE Project gradient to the local low-rank subspace:
		$R_i^{t,k}\gets(P_i^t)^\top G_i^{t,k}
		\in\mathbb{R}^{r\times n}$
		
		\STATE \textit{Adam update in the low-rank space:}
		
		\STATE $M_i^{t,k+1}
		\gets
		\beta_1M_i^{t,k}
		+
		(1-\beta_1)R_i^{t,k}$
		
		\STATE $V_i^{t,k+1}
		\gets
		\beta_2V_i^{t,k}
		+
		(1-\beta_2)(R_i^{t,k}\odot R_i^{t,k})$
		
		\STATE $\widehat{M}_i^{t,k+1}
		\gets
		M_i^{t,k+1}/(1-\beta_1^{k+1})$
		
		\STATE $\widehat{V}_i^{t,k+1}
		\gets
		V_i^{t,k+1}/(1-\beta_2^{k+1})$
		
		\STATE $N_i^{t,k+1}
		\gets
		\widehat{M}_i^{t,k+1}/
		(\sqrt{\widehat{V}_i^{t,k+1}}+\epsilon)$
		
		\STATE Reconstruct the ambient-space low-rank direction:
		$\widehat{G}_i^{t,k}
		\gets
		\alpha P_i^tN_i^{t,k+1}
		\in\mathbb{R}^{m\times n}$
		
		\STATE Update local weights:
		$W_i^{t,k+1}
		\gets
		W_i^{t,k}+\eta\widehat{G}_i^{t,k}-\lambda \eta W_i^{t,k}$
		
		\ENDFOR
		
		\STATE Accumulate  updates:
		$\Delta W_i^t
		\gets
		W_i^{t,K}-
		W_i^{t,0}$
		
		\STATE Client $i$ uploads $(\Delta W_i^t)$ to the server
		
		\ENDFOR
		
		\STATE Aggregate updates in the ambient space:
		
		\STATE
		$W^{t+1}
		\gets
		W^t+
		\frac{1}{N}
		\sum_{i=1}^{N}
		\Delta W_i^t$
		
		\ENDFOR
		
	\end{algorithmic}
\end{algorithm}

\begin{algorithm}[t]
	\small
	\caption{\texttt{FedLore} Training Procedure}
	\label{alg:fedlore2}
	\begin{algorithmic}[1]
		\REQUIRE Global weights $W^0$; learning rate $\eta$;
		scale $\alpha$; rank $r$; Adam coefficients
		$\beta_1,\beta_2$; $\epsilon>0$;
		communication rounds $T$; local steps $K$;
		client set $\mathcal{N}$ with $|\mathcal{N}|=N$;
		weight decay $\lambda=0$.
		\ENSURE Trained global model $W^T$.
		
		\STATE Initialize a shared orthonormal basis
		$P^0\in\mathbb{R}^{m\times r}$
		\STATE Server broadcasts $(W^0,P^0)$ to all clients
		\STATE Each client initializes 
		global-model cache: $\widetilde W_i\gets W^0$
		
		\FOR{$t=0$ to $T-1$}
		\FOR{each client $i\in\mathcal{N}$
			\textbf{in parallel}}
		
		\STATE \texttt{//Low-Rank Model Synchronization}
		\IF{$t>0$}
		\STATE Receive $(P^t,\Delta^{t-1})$
		from the server
		\STATE $\widetilde W_i
		\gets \widetilde W_i
		+\eta\alpha P^{t-1}\Delta^{t-1}$
		\ENDIF
		\STATE $W_i^{t,0}\gets\widetilde W_i$, $M_i^{t,0}\gets 0$,
		$V_i^{t,0}\gets 0$, $\Delta Z_i^t\gets 0$
		
		\STATE \texttt{//Local Projected Adam Updates}
		\FOR{$k=0$ to $K-1$}
		%\STATE Sample a fresh mini-batch
		%	$\xi_i^{t,k}$
		\STATE $G_i^{t,k}\gets
		-\nabla_W f_i(W_i^{t,k};\xi_i^{t,k})
		\in\mathbb{R}^{m\times n}$;\quad $R_i^{t,k}\gets
		(P^t)^\top G_i^{t,k}
		\in\mathbb{R}^{r\times n}$
		
		\STATE $M_i^{t,k+1}\gets
		\beta_1M_i^{t,k}
		+(1-\beta_1)R_i^{t,k}$;\quad $V_i^{t,k+1}\gets
		\beta_2V_i^{t,k}
		+(1-\beta_2)
		(R_i^{t,k}\odot R_i^{t,k})$
		
		\STATE $\widehat M_i^{t,k+1}\gets
		M_i^{t,k+1}/(1-\beta_1^{k+1})$;\quad $\widehat V_i^{t,k+1}\gets
		V_i^{t,k+1}/(1-\beta_2^{k+1})$
		\STATE $N_i^{t,k+1}\gets
		\widehat M_i^{t,k+1}/
		(\sqrt{\widehat V_i^{t,k+1}}+\epsilon)$;\quad $W_i^{t,k+1}\gets
		W_i^{t,k}+\eta\alpha P^tN_i^{t,k+1}$
		\STATE $\Delta Z_i^t\gets
		\Delta Z_i^t+N_i^{t,k+1}$
		\ENDFOR
		
		\STATE \texttt{//Local Candidate Subspace}
		\STATE Reuse the last local gradient
		$G_i^{t,K-1}$;\quad $P_i^t\gets
		\operatorname{TopLeftSVD}_r(G_i^{t,K-1})
		\in\mathbb{R}^{m\times r}$
		\STATE Upload $(\Delta Z_i^t,P_i^t)$
		to the server;\quad Retain $P^t$ for next-round
		model synchronization
		\ENDFOR
		
		\STATE \texttt{//Server-Side Update Aggregation}
		\STATE $\Delta^t\gets
		\frac{1}{N}\sum_{i\in\mathcal{N}}\Delta Z_i^t$;\quad $W^{t+1}\gets
		W^t+\eta\alpha P^t\Delta^t$
		
		\STATE \texttt{//Shared-Subspace Construction}
		\STATE $B^t\gets
		\frac{1}{\sqrt N}
		[P_1^t,\ldots,P_N^t]
		\in\mathbb{R}^{m\times Nr}$;\quad $P^{t+1}\gets
		\operatorname{TopLeftSVD}_r(B^t)
		\in\mathbb{R}^{m\times r}$
		\ENDFOR
	\end{algorithmic}
\end{algorithm}

% SharedBasis(s,m,r):
% Generate A \in R^{m x r} using a fixed pseudorandom generator
% with seed s and i.i.d. standard Gaussian entries.
% Compute a deterministic thin QR factorization A = QR,
% with the diagonal entries of R made positive, and return Q.
% All participants use the same generator and QR convention.

\begin{algorithm}[t]
	\small
	\caption{\texttt{FedLore-Random} Training Procedure}
	\label{alg:fedlore_random}
	\begin{algorithmic}[1]
		\REQUIRE Global weights $W^0\in\mathbb{R}^{m\times n}$;
		learning rate $\eta$; scale $\alpha$;
		rank $1\le r\le m$; Adam coefficients
		$\beta_1,\beta_2$; $\epsilon>0$;
		communication rounds $T$; local steps $K$;
		client set $\mathcal{N}$ with $|\mathcal{N}|=N$;
		weight decay $\lambda=0$.
		\ENSURE Trained global model $W^T$.
		
		\STATE Server samples a random seed $s^0$ and sets
		$P^0\gets\operatorname{SharedBasis}(s^0,m,r)$
		\STATE Server broadcasts $(W^0,s^0)$ to all clients
		\STATE Each client initializes
		$\widetilde W_i\gets W^0$ and reconstructs
		$P^0\gets\operatorname{SharedBasis}(s^0,m,r)$
		
		\FOR{$t=0$ to $T-1$}
		\IF{$t>0$}
		\STATE Server samples a fresh random seed $s^t$,
		independently of the training randomness
		\STATE Server sets
		$P^t\gets\operatorname{SharedBasis}(s^t,m,r)$
		and broadcasts $(s^t,\Delta^{t-1})$
		\ENDIF
		
		\FOR{each client $i\in\mathcal{N}$
			\textbf{in parallel}}
		
		\STATE \texttt{// Low-Rank Model Synchronization}
		\IF{$t>0$}
		\STATE $\widetilde W_i
		\gets\widetilde W_i
		+\eta\alpha P^{t-1}\Delta^{t-1}$
		\STATE Reconstruct the shared random basis:
		$P^t\gets
		\operatorname{SharedBasis}(s^t,m,r)$
		\ENDIF
		\STATE $W_i^{t,0}\gets\widetilde W_i$,
		$M_i^{t,0}\gets0$,
		$V_i^{t,0}\gets0$,
		$\Delta Z_i^t\gets0$
		
		\STATE \texttt{// Local Projected Adam Updates}
		\FOR{$k=0$ to $K-1$}
		\STATE Sample a fresh mini-batch $\xi_i^{t,k}$
		\STATE $G_i^{t,k}\gets
		-\nabla_W f_i(W_i^{t,k};\xi_i^{t,k})$;
		\quad
		$R_i^{t,k}\gets(P^t)^\top G_i^{t,k}$
		
		\STATE $M_i^{t,k+1}\gets
		\beta_1M_i^{t,k}
		+(1-\beta_1)R_i^{t,k}$;
		\quad
		$V_i^{t,k+1}\gets
		\beta_2V_i^{t,k}
		+(1-\beta_2)
		(R_i^{t,k}\odot R_i^{t,k})$
		
		\STATE $\widehat M_i^{t,k+1}\gets
		M_i^{t,k+1}/(1-\beta_1^{k+1})$;
		\quad
		$\widehat V_i^{t,k+1}\gets
		V_i^{t,k+1}/(1-\beta_2^{k+1})$
		
		\STATE $N_i^{t,k+1}\gets
		\widehat M_i^{t,k+1}/
		(\sqrt{\widehat V_i^{t,k+1}}+\epsilon)$;
		\quad
		$W_i^{t,k+1}\gets
		W_i^{t,k}+\eta\alpha P^tN_i^{t,k+1}$
		
		\STATE $\Delta Z_i^t\gets
		\Delta Z_i^t+N_i^{t,k+1}$
		\ENDFOR
		
		\STATE Upload $\Delta Z_i^t$ to the server;
		retain $P^t$ for next-round model synchronization
		\ENDFOR
		
		\STATE \texttt{// Server-Side Update Aggregation}
		\STATE $\Delta^t\gets
		\frac{1}{N}\sum_{i\in\mathcal{N}}\Delta Z_i^t$;
		\quad
		$W^{t+1}\gets W^t+\eta\alpha P^t\Delta^t$
		\ENDFOR
	\end{algorithmic}
\end{algorithm}

\section{Related Work}
\label{sec:related_work}

\paragraph{Communication-efficient federated optimization.}
Communication-efficient federated learning commonly allows clients to perform
multiple local updates between aggregation rounds. FedBCGD reduces the amount
of information transmitted in each round by updating and communicating
parameter blocks, and further incorporates drift control and variance
reduction in its accelerated variant~\citep{liu2024fedbcgd}. This approach
addresses the size of each message. In contrast, \texttt{CTP-FL} communicates
one model-sized vector per participating client and changes where local
gradients are evaluated: clients query the same server-defined predictive
trajectory, so their updates can be averaged without first combining models
that have followed different local trajectories.

\paragraph{Alignment under heterogeneous data.}
Several recent methods study different forms of local--global misalignment.
FedSWA and FedMoSWA use stochastic weight averaging and momentum-based control
to improve generalization under highly heterogeneous data~\citep{liu2025fedswa}.
FedNSAM examines the mismatch between local and global flatness and uses a
global Nesterov direction to improve their consistency~\citep{liu2025fednsam}.
These methods primarily target the properties of the resulting solution,
including flatness and generalization. Our focus is the geometry of gradient
evaluation during a communication round: when all clients evaluate at common
points, averaging their gradients estimates the gradient of the global
objective at those points, irrespective of how different the individual
client gradients are.

\paragraph{Federated adaptive and structured optimizers.}
FedAdamW combines local correction, decoupled weight decay, and aggregation
of second-moment estimates for federated large-model
training~\citep{liu2026fedadamw}. FedMuon exploits matrix orthogonalization
and local--global alignment to improve federated optimization of
matrix-structured parameters~\citep{liu2025fedmuon}. FedPAC identifies
preconditioner drift as a source of instability when local second-order
optimizers induce incompatible client geometries, and proposes
preconditioner alignment and update correction~\citep{liu2026fedpac}.
Unlike these optimizer-specific mechanisms, \texttt{CTP-FL} applies to
stochastic-gradient evaluations without transmitting moments or
preconditioners. Its common trajectory aligns the \emph{locations} of
gradient evaluation rather than optimizer states.

\paragraph{Global flatness and privacy.}
DP-FedPGN encourages globally flat solutions in client-level differentially
private federated learning through a global gradient-norm
penalty~\citep{liu2025dpfedpgn}. Its objective and privacy accounting are
different from ours. We cite it because it likewise illustrates that a
quantity defined by the global objective need not be faithfully represented
by independently optimized local objectives.

\end{document}